\documentclass{article}

\usepackage[preprint,nonatbib]{neurips_2019}

\usepackage[utf8]{inputenc} % allow utf-8 input
\usepackage[T1]{fontenc}    % use 8-bit T1 fonts
\usepackage{hyperref}       % hyperlinks
\usepackage{url}            % simple URL typesetting
\usepackage{booktabs}       % professional-quality tables
\usepackage{amsfonts}       % blackboard math symbols
\usepackage{nicefrac}       % compact symbols for 1/2, etc.
\usepackage{microtype}      % microtypography

\usepackage{algorithm}
\usepackage{algorithmic}

\usepackage{url}            % simple URL typesetting

\usepackage{graphicx}
\usepackage{caption}
\usepackage{subfigure} 

\usepackage{mathtools}
\usepackage{amsmath,amssymb,amsthm}

\usepackage{enumitem}

\usepackage{multirow}

\def\D{{\mathcal D}}

\def\R{{\mathbb R}}
\def\CE{\mathbf{E}}
\def\mG{\mathcal{G}}

\newcommand{\notimplies}{%
  \mathrel{{\ooalign{\hidewidth$\not\phantom{=}$\hidewidth\cr$\implies$}}}}
  
\makeatletter
\newcommand{\setlabel}[2]{%
  \phantomsection
  #1\def\@currentlabel{\unexpanded{#1}}\label{#2}%
}
\makeatother

\newtheorem{theorem}{Theorem}
\newtheorem{corollary}{Corollary}
\newtheorem{proposition}{Proposition}
\newtheorem{lemma}{Lemma}
\newtheorem{assumption}{Assumption}
\newtheorem{remark}{Remark}
\newtheorem{definition}{Definition}

\graphicspath{{images/}} % Directory in which figures are stored

\title{Blockwise Adaptivity: 
Faster Training and \\Better 
Generalization
in Deep Learning}

\author{
  Shuai Zheng, James T. Kwok\\
  \texttt{szhengac@cse.ust.hk, jamesk@cse.ust.hk} \\
  Department of Computer Science\\
  Hong Kong University of Science and Technology
}

\begin{document} 

\maketitle

\begin{abstract} 
Stochastic methods with coordinate-wise adaptive stepsize (such as
RMSprop and 
Adam)
have been widely used in training deep neural networks. 
Despite their fast convergence, 
they can generalize worse than stochastic gradient descent. 
In this paper, by revisiting the design of Adagrad, we propose
to split the network parameters into blocks, and
use a blockwise
adaptive stepsize.
Intuitively, 
blockwise adaptivity is less aggressive than
adaptivity to individual coordinates, 
and can have a better balance between adaptivity and generalization. 
We show 
theoretically 
that the proposed blockwise adaptive gradient
descent has comparable convergence rate as its counterpart with
coordinate-wise adaptive stepsize, but is faster
up to some constant. 
We also study its uniform stability 
and show that blockwise adaptivity can lead to lower generalization error than coordinate-wise adaptivity. 
Experimental results show that blockwise adaptive gradient
descent
converges
faster 
and improves generalization performance over
Nesterov's accelerated gradient and Adam.
 
\end{abstract} 

%%%%%%%%%%%%%%%%%%%%%%%%%%%%%%%%%%%%%%%%%%%%%%%%%%%%%%%%%%%%%%%%%%%%%%%%%%%%%%%%

\section{Introduction}

Deep networks
have achieved excellent performance in a variety of domains such as computer vision \cite{he2016deep}, 
language modeling \cite{zaremba2014recurrent}, and speech recognition \cite{graves2013speech}.
%and machine translation \cite{vaswani2017attention}. 
%Typically, these problems are formulated as the minimization of a nonconvex objective on a set of training samples. 
The most popular optimizer
is stochastic gradient decent (SGD) \cite{robbins1951stochastic},
which is simple and has low per-iteration complexity.
Its convergence rate is also well-established 
%for both convex and nonconvex cases 
\cite{ghadimi2013stochastic,bottou2018optimization}. 
However, vanilla SGD is sensitive to the choice of stepsize, and requires careful tuning. 
To improve the efficiency and robustness of SGD, many variants have been proposed, 
such as momentum acceleration
\cite{polyak1964some,nesterov1983method,sutskever2013importance} and adaptive
stepsizes \cite{duchi2011adaptive,hinton2012rmsprop,zeiler2012adadelta,kingma2014adam}. 

Though variants with coordinate-wise adaptive stepsize (such as Adam
\cite{duchi2011adaptive})
have shown to be effective in accelerating convergence,
%in the nonconvex setting,
their
generalization performance 
is often worse 
than SGD \cite{wilson2017marginal}.
To improve generalization performance, attempts have been made to use  a
layer-wise stepsize \cite{singh2015layer,yang2017lars,adam2017normalized,zhou2018adashift}, which
assign different stepsizes to different layers or normalize the layer-wise gradient.
However, there has been no theoretical analysis for its empirical success. 
More generally, the whole network parameter can also be partitioned into blocks instead of simply
into layers.

%one can achieve faster convergence and favourable accuracy over SGD and its variants with coordinate-wise adaptive learning rate. 

Recently, it is shown that coordinate-wise adaptive gradient descent is closely related to sign-based gradient descent \cite{pmlr-v80-balles18a,pmlr-v80-bernstein18a}. 
Theoretical arguments and empirical evidence 
%are provided to 
suggest that the gradient sign would impede generalization
\cite{pmlr-v80-balles18a}.
To contract the generalization gap, a partial adaptive parameter for the
second-order momentum is proposed \cite{chen2018closing}.
By using a smaller partial adaptive parameter, the adaptive gradient algorithm behaves less like sign descent and more like SGD. 

Moreover,
in 
%adaptive gradient 
methods
with coordinate-wise adaptive stepsize, a small $\epsilon$ ($=10^{-8}$) parameter is
typically used 
to avoid numerical
problems in practical
implementation. 
It is discussed in \cite{zaheer2018adaptive} that this $\epsilon$ parameter
controls adaptivity of the algorithm, and using a larger value (say,
$\epsilon = 10^{-3}$) can reduce adaptivity and empirically helps Adam to match its generalization performance with SGD. 
This implies that coordinate-wise adaptivity may be too strong for good generalization performance.

In this paper, 
by revisiting the derivation of Adagrad, we consider partitioning the model parameters into blocks 
as in \cite{singh2015layer,yang2017lars,adam2017normalized,zhou2018adashift},
and propose the use of a
blockwise stepsize. 
By allowing this blockwise stepsize to depend on the corresponding gradient block, we have the notion of blockwise adaptivity. 
Intuitively, 
it is less aggressive 
to adapt to parameter blocks instead of to individual coordinates,
and this reduced adaptivity can have a better balance between adaptivity and generalization. 
Moreover,  as
blockwise adaptivity is not 
coordinate-wise adaptivity, it does not suffer from the performance deterioration
as for sign-based gradient descent.

We will focus on 
the expected risk
minimization problem
\cite{pmlr-v80-bernstein18a,ghadimi2013stochastic,ward2018adagrad,zaheer2018adaptive,zou2018convergence,zou2018sufficient}:
\begin{eqnarray} \label{eq:expected_loss}
\min_{\theta} F(\theta) = \CE_z[f(\theta; z)],
\end{eqnarray}
where $f$ is some 
possibly nonconvex
loss function, and $z$ is a random sample.
%and the expectation is taken w.r.t. the distribution of $z$. 
The expected risk 
measures the generalization performance
on unseen data \cite{bottou2018optimization}, and
reduces to the empirical risk when a finite training set is considered.  
We show theoretically 
that the proposed blockwise adaptive gradient descent can be faster than its counterpart with 
coordinate-wise adaptive stepsize.
%under certain conditions. 
Using tools on uniform stability \cite{bousquet2002stability,hardt2016train},
we also show that blockwise adaptivity has potentially lower generalization error than coordinate-wise adaptivity. 
Empirically, 
blockwise adaptive gradient descent converges faster and obtains better generalization performance than
its coordinate-wise counterpart (Adam) and Nesterov's accelerated gradient (NAG)
\cite{sutskever2013importance}. 

{\bf Notations}. 
For an integer $n$, $[n] =\{1, 2, \dots, n\}$. 
For a vector $x$,
%\in \R^d$, 
$x^T$ denotes its transpose,
$\text{Diag}(x)$ is a diagonal matrix with $x$ on its diagonal, $\sqrt{x}$ is the
element-wise square root of $x$, $x^2$ is the coordinate-wise square of $x$, $\|x\|_2 = \sqrt{x^Tx}$, $\|x\|_{\infty} = \max_i |x_i|$,   
$\|x\|_Q^2 = x^TQx$,  where $Q$ is a positive semidefinite (psd) matrix, and $x \geq 0$ means $x_i \geq 0$ for all $i$. 
For two vectors $x$ and $y$, 
$x/y$, and $\langle x, y\rangle$ denote the element-wise 
division and dot product, respectively. For a square matrix $X$, $X^{-1}$ is its
inverse, and $X \succeq 0$ means that $X$ is psd. Moreover,
$1_d = [1, 1, \dots, 1]^T \in \R^d$. 

%%%%%%%%%%%%%%%%%%%%%%%%%%%%%%%%%%%%%%%%%%%%%%%%%%%%%%%%%%%%%%%%%%%%%%%%%%%%%%%%

\section{Related Work}

Adagrad 
\cite{duchi2011adaptive}
is the first adaptive gradient method 
in online convex learning
with coordinate-wise stepsize.
It is particularly useful for sparse learning, as parameters for
the rare features 
can take large steps. Its
stepsize schedule is competitive with the
best coordinate-wise stepsize in hindsight \cite{mcmahan2010adaptive}. 
Recently, its convergence rate with a global adaptive stepsize 
in nonconvex optimization 
is established 
\cite{ward2018adagrad}. 
%\citeauthor{zou2018sufficient} (\citeyear{zou2018sufficient}) proposed a more general adaptive gradient learning framework with exponential moving average momentum for minimizing nonconvex objectives. 
It is shown that Adagrad converges to a stationary point at the optimal $\mathcal{O}(1/\sqrt{T})$ rate 
(up to a 
factor
$\log(T)$), where $T$ is the total number of iterations. 

Recall that 
the SGD iterate is the solution to the problem: 
%\begin{eqnarray} \label{eq:sgd}
$\theta_{t+1} =
\arg\min_\theta \langle g_t, \theta \rangle + \frac{1}{2\eta}\|\theta - \theta_{t}\|_2^2$, 
%\end{eqnarray}
where $g_t$ 
is the gradient of the loss function $f_t$
at iteration $t$,
and 
$\theta_t \in \R^d$ 
is the parameter vector.
To incorporate information about the curvature of sequence $\{f_t\}$, 
the $\ell_2$-norm 
in the SGD update can be replaced 
by the Mahalanobis norm,
leading to
\cite{duchi2011adaptive}:
\begin{eqnarray} \label{eq:gd}
\theta_{t+1} =
\arg\min_\theta \langle g_t, \theta \rangle + \frac{1}{2\eta}\|\theta - \theta_{t}\|_{\text{Diag}(s_t)^{-1}}^2,
\end{eqnarray}
where $s_t \geq 0$. This is an instance of mirror descent \cite{nemirovsky1983problem}.  
Its regret bound has a 
gradient-related
term 
$\sum_{t=1}^T\|g_t\|_{\text{Diag}(s_t)^{-1}}^2$. 
Adagrad's
stepsize
can be obtained by examining a similar objective
\cite{duchi2011adaptive}:
\begin{equation} \label{eq:adagrad_problem}
\min_{s \in \mathcal{S}}
\; \sum_{t=1}^T\|g_t\|_{\text{Diag}(s)^{-1}}^2,
\end{equation} 
where 
$\mathcal{S} = \{s: s \geq 0,  \langle s, 1 \rangle \leq c\}$, and
$c$ is some constant. 
At optimality, 
$s_{*,i} = c\|g_{1:T, i}\|_2/\sum_{j=1}^d\|g_{1:T, j}\|_2$, 
where $g_{1:T, i} = [g_{1, i}^T, \dots, g_{T, i}^T]^T$. As $s_t$ cannot depend on
$g_j$'s with $j > t$, this suggests $s_{t, i} \propto \|g_{1:t, i}\|_2$. 
Theoretically, this choice of $s_t$ leads to a regret bound that is 
competitive with the best post-hoc optimal bound \cite{mcmahan2010adaptive}. 

To solve the expected risk minimization problem in (\ref{eq:expected_loss}),
an Adagrad variant called weighted AdaEMA
is 
recently
proposed in  
 \cite{zou2018sufficient}.
It employs weighted averaging of $g_{t, i}^2$'s for stepsize and momentum
acceleration. This is a general coordinate-wise adaptive method and includes many
Adagrad variants as special cases, including Adam and RMSprop.

%%%%%%%%%%%%%%%%%%%%%%%%%%%%%%%%%%%%%%%%%%%%%%%%%%%%%%%%%%%%%%%%%%%%%%%%%%%%%%%%

\section{Blockwise Adaptive Descent}

%In this section, we first introduce a novel blockwise adaptive stepsize schedule by analyzing a similar problem as (\ref{eq:adagrad_problem}), and then introduce a stochastic momentum method for solving the expected loss minimization problem.  We show theoretically that this can lead to faster convergence and lower generalization error than its coordinate-wise counterpart.

%under certain conditions. 
%Empirically, the proposed method shows fast convergence without deteriorating generalization. 

%%%%%%%%%%%%%%%%%%%%%%%%%%%%%%%%%%%%%%

\subsection{Blockwise 
vs
Coordinate-wise 
Adaptivity}

Let $n$ 
be the sample size, 
$d$  be the input dimensionality, and 
$m$ be the output dimensionality.
Consider 
a $L$-layer neural network,
with output
$\phi_{L-1}(\cdots\phi_2(\phi_1(XW_1)W_2 )\cdots W_{L-1})W_L$,
where $X \in \R^{n\times d}$ is the input matrix and $\{W_l \in \R^{d_{l-1} \times d_l}\}_{l=1}^L$ are the weight matrices with
$d_0 = d$ and $d_L = m$. 
The activation functions $\{\phi_l\}_{l=1}^{L-1}$ are assumed to be bijective (e.g., tanh and leaky ReLU). 
For simplicity, assume that $d_l = d = m>n$ for all $l$.
Training this neural network with the square loss corresponds to solving
the nonlinear optimization problem:
%\begin{eqnarray} \label{eq:nonlinear_ls}
$\min_{\{W_l\}_{l=1}^L} \|\phi_{L-1}(\cdots\phi_2(\phi_1(XW_1)W_2 )\cdots
W_{L-1})W_L - Y\|^2_2$,
%\end{eqnarray}
where $Y \in \R^{n\times m}$ is the label matrix.
Consider training  the network
layer-by-layer,
starting from the bottom one.
For layer $l$,
%the following update rule is used at time $t$:
%\begin{eqnarray} \label{eq:blockwise-method}
$W_{t+1, l} = W_{t, l} - \eta_{t, l}g_{t, l}$,
%\end{eqnarray}
where $g_{t, l}$ is a stochastic gradient evaluated at $W_{t, l}$
at time $t$,
and $\eta_{t, l}$ is the stepsize
%SGD is an instance of (\ref{eq:blockwise-method}). 
which may 
be adaptive in that it
depends on $g_{t, l}$.
This 
layer-wise training
is analogous to block coordinate descent, with each layer being a block.  
The optimization subproblem for the $l$th layer can be rewritten as 
\begin{eqnarray} \label{eq:nonlinear_sub_ls}
\min_{W_l} \|\Phi_l(H_{l-1}W_l) - Y\|^2_2,
\end{eqnarray}
where $\Phi_l(\cdot) = \phi_{L-1}(\cdots\phi_{l+1}(\phi_l(\cdot)W_{l+1})\cdots
W_{L-1})W_L$, $H_{l-1} = \phi_{l-1}(\cdots\phi_1(XW_1)\cdots W_{l-1})$ is
the input hidden representation
of $X$ at the $l$th layer, and $H_0 = X$. 
%$H_{l-1}$ can be seen as preprocessing for data $X$, while the upper layers are fixed nonlinear transformations.
%making subproblem (\ref{eq:nonlinear_sub_ls}) a nonlinear least-squares problem.  

\begin{proposition} \label{theorem:nonlinear_large_margin}
Assume that 
$W_{l'}$'s (with $l' > l$)
are invertible.
If 
$W_{l}$ is initialized to zero,
and
$H_{l-1}$ has full row rank, 
then
the critical point that 
%(\ref{eq:blockwise-method}) 
it converges to is also the minimum
$\ell_2$-norm solution of
(\ref{eq:nonlinear_sub_ls})
in expectation.
\end{proposition} 
As stepsize $\eta_{t, l}$ can depend on $g_{t,l}$, 
Proposition~\ref{theorem:nonlinear_large_margin}
shows that blockwise
adaptivity can find the minimum $\ell_2$-norm solution 
of (\ref{eq:nonlinear_sub_ls}).
%as for SGD.\footnote{there exists work that shows sgd converge the global optimal solution of (\ref{eq:nonlinear_ls}) with identity activation fn.  for the scenarios we consider, i am not sure. \#*** prev work only considered sgd in the linear least sqr problem. but here u hv a nonlinear problem. so hv ppl shown  that sgd can find the minimum $\ell_2$-norm solution?}
In contrast, coordinate-wise adaptivity fails to find the minimum
$\ell_2$-norm solution even for the underdetermined linear least squares
problem \cite{wilson2017marginal}. 
Another benefit of using a blockwise stepsize is that 
the
optimizer's
extra memory cost
can be reduced. Using a coordinate-wise stepsize
requires an additional $\mathcal{O}(d)$ memory for storing estimates of the
second moment, while the blockwise stepsize only needs an extra
$\mathcal{O}(B)$ memory, where $B$ is the number of blocks. A deep network
generally has millions of parameters but only tens of layers. If we set $B$
to be the number of layers, memory reduction 
can
be significant.

%Exploiting blockwise adaptivity may also be beneficial for jointly training the nonlinear problem (\ref{eq:nonlinear_ls}), potentially leading to a nearly large-margin solution in each layer.  \footnote{i don't think one can empirically show the method finds the large margin solution in each layer, as no one knows how such solution looks like. \#*** including this sentence is both good and bad. note that reviewers may ask u for expt evidence (on margin and other things)}
%Thus, what remains is to find a provably convergent blockwise adaptive method for optimizing nonconvex objectives. 

%For instance, consider large word embeddings.  With a blockwise stepsize, we can assign a single adaptive stepsize to each word embedding vector. This reduces the memory cost from $\mathcal{O}(Vd_{e})$ to $\mathcal{O}(V)$, where $V$ is the vocabulary size and $d_{e}$ is the size of the embedding vector. 

There have been some recent attempts on the
use of layer-wise stepsize
in deep networks,
either by assigning a specific adaptive stepsize
to each layer or normalizing the layer-wise gradient
%, and have shown improved performance
%\footnote{they can both be considered as using layer-wise learning rate, as normalization can be incorporated into learning rate.  so they both falls into (\ref{eq:blockwise-method}). \#*** u mention 2 approaches here.  better split the refs to 2 groups}
\cite{singh2015layer,yang2017lars,adam2017normalized,zhou2018adashift}. 
However, 
justifications
and 
convergence analysis 
%\footnote{for deep learning. theoretical learning rate is generally different from empirical stepsize schedule.  some highly tuned stepsize or fancy stepsize schedule are usually employed in practice.  this happens in all the paper. \#*** indeed, though u provide some settings of $\eta$, they're not used in the expts.  and so ur settings also do not hv justifications.} for the stepsize
are still lacking.

%%%%%%%%%%%%%%%%%%%%%%%%%%%%%%%%%%%%%%

%\subsection{Nonconvex Optimization with Momentum}
%\label{sec:adabcm}

\subsection{Blockwise Adaptive Learning Rate
with Momentum}
\label{sec:adabc}

Let the gradient $g_t \in \R^d$ 
be partitioned 
to $\{g_{t, \mathcal{G}_b} \in \R^{d_b}: b = 1, \dots, B\}$, 
where $\mathcal{G}_b$ is the set of indices in block $b$, and $g_{t, \mathcal{G}_b}$ 
is the subvector 
of $g_t$ 
belonging to block $b$. 
Inspired by 
problem~(\ref{eq:adagrad_problem}) 
in the derivation of
Adagrad,
we consider the following variant
which imposes a block structure
on $s$:\footnote{We assume the indices in $\mathcal{G}_b$ are consecutive; otherwise, we can simply reorder the elements of the gradient.  Note that reordering does not change the result, as the objective is invariant to ordering of the coordinates. }
\begin{eqnarray}
\min_{s \in \mathcal{S}'} \sum_{t=1}^T\|g_t\|_{\text{Diag}(s)^{-1}}^2 \label{eq:opt2},
\end{eqnarray}
where $\mathcal{S}' = \{s: s = [q_1 1_{d_1}^T, \dots, q_B 1_{d_B}^T]^T \geq 0, \langle s, 1 \rangle \leq c\}$ for some
 $q_i\in \R$.
It can be easily shown that
at optimality, $q_b=c\|g_{1:T, \mG_b}\|_2/(\sqrt{d_b}\sum_{i=1}^B\sqrt{d_i}\|g_{1:T, \mG_i}\|_2)$,
where $g_{1:T, \mG_b} = [g_{1, \mG_b}^T, \dots, g_{T, \mG_b}^T]^T$. 
The optimal $q_b$ is 
thus
proportional to $\|g_{1:T, \mG_b}\|_2/\sqrt{d_b}$.  
When $s_t$ in (\ref{eq:gd}) is partitioned by the same block structure, the optimal $q_b$  
suggests
to incorporate
$\|g_{1:t, \mG_b}\|_2/\sqrt{d_b}$ into $s_t$ for block $b$ at time
$t$. Thus, we consider the following update rule with
blockwise adaptive stepsize:
\begin{eqnarray} 
v_{t, b} & = & v_{t-1, b} + \|g_{t, \mG_b}\|_2^2/d_b, \label{eq:var}\\
\theta_{t+1, \mG_b} & = & \theta_{t, \mG_b} - \eta_tg_{t, \mG_b}/(\sqrt{v_{t,b}} + \epsilon), 
\label{eq:blk_adagrad}
\end{eqnarray}
where $\epsilon$ is a hyperparameter that prevents numerical issues. 
When $B=d$, 
this update rule reduces to Adagrad.
In Appendix~\ref{app:online}, we show that it can outperform Adagrad in online convex learning.

As $v_t$ in (\ref{eq:var}) is increasing w.r.t. $t$, the update in ({\ref{eq:blk_adagrad}}) suffers from vanishing stepsize, making slow progress on 
nonconvex problems such as
deep network
training.
To alleviate this problem, 
weighted moving average momentum 
has been used in many Adagrad variants such as RMSprop, Adam and weighted AdaEMA
 \cite{zou2018sufficient}.
In this paper, 
we adopt
weighted AdaEMA 
with the use of a blockwise adaptive stepsize.
The proposed procedure, which
will be called blockwise adaptive gradient with momentum (BAGM), is shown in 
Algorithm~\ref{alg:block-mom-adagrad}.
When $B=d$ and $\epsilon=0$, BAGM reduces to weighted AdaEMA.
As weighted AdaEMA 
includes many Adagrad variants,
the proposed 
BAGM
also
covers the corresponding blockwise variants.
In Algorithm~\ref{alg:block-mom-adagrad},
$m_t$ serves as an exponential
moving averaged momentum, and $\{\beta_t\}$ is a sequence of momentum parameters. 
The $a_t$'s 
assign different weights to the past gradients in the accumulation of variance,
as:
\begin{equation} \label{eq:a}
\hat{v}_{t,b} = \sum_{i=1}^t \frac{a_i}{A_t}\frac{ \|g_{i, \mG_b}\|^2_2 }{d_b}= \frac{1}{\sum_{j=1}^ta_j}\sum_{i=1}^t a_i\frac{ \|g_{i, \mG_b}\|^2_2 }{d_b}.
\end{equation} 
%Some examples of sequence $\{a_t\}$ will be given in Section~\ref{sec:adabcm_converge}.
In this paper, we consider the three weight sequences $\{a_t\}$ introduced in \cite{zou2018convergence}.
%As will be seen in section~\ref{sec:adabcm_converge}, \ref{seq:1} and \ref{seq:2} lead to a convergence rate of $\mathcal{O}(\log(T)/\sqrt{T})$, while \ref{seq:3} leads to a $\mathcal{O}(1)$ bound.
%\begin{enumerate}[label=\textbf{S.\arabic*}]
%\item \label{seq:1}
%\footnote{i have added a custom cmd such that we can label word in sentence. \#*** i removed the labels to save space. so pls add them back manually to the rest of the paper}
\setlabel{{\bf S.1}}{seq:1}: $a_t = a$ for some $a > 0$;
%\item \label{seq:2}
\setlabel{{\bf S.2}}{seq:2}: $a_t = t^{\tau}$ for some $\tau > 0$:
The fraction $a_t/A_t$ in  
(\ref{eq:a}) then decreases as
$\mathcal{O}(1/t)$.
%\item \label{seq:3}
\setlabel{{\bf S.3}}{seq:3}: $a_t = \alpha^{-t}$ for some $0 < \alpha < 1$: 
%From (\ref{eq:a}), we have $\hat{v}_{t, b} = \frac{1 - \alpha}{1 - \alpha^t}\sum_{i=1}^t\alpha^{t-i}\frac{\|g_{i, \mG_b}\|^2_2}{d_b}$.  As $\frac{1 - \alpha}{1 - \alpha^t}\sum_{i=1}^t\alpha^{t-i} = 1$, 
It can be shown that this is equivalent to using the exponential moving average estimate:
$v_{t,b}  =  \alpha v_{t-1, b} + (1 - \alpha)\frac{\|g_{t,
\mathcal{G}_b}\|_2^2}{d_b}$,
%hspace{.1in}
and $\hat{v}_{t,b}  =  \frac{v_{t,b}}{1 - \alpha^t}$.
%\end{eqnarray*}
When $B = d$, 
$\beta_t = \beta$, and $\eta_t = \eta/(\sqrt{t}(1 - \beta^t ))$, the proposed algorithm reduces to Adam. 
% \end{enumerate}

%As weighted AdaEMA incorporates $\epsilon$ into $v_0$ (i.e., $v_0 = \epsilon$), the effect of $\epsilon$ decreases when training proceeds. 
%As will be seen in Section~\ref{sec:gen_error_as}, a larger $\epsilon$ reduces the generalization error. Empirically, this has been observed in \cite{zaheer2018adaptive} that a large $\epsilon$ (=$10^{-3}$) improves generalization performance. 

%%%%%%%%%%%%%%%%%%%%%%%%%%%%%%%%%%%%%%

%\subsubsection{Proposed Algorithm}

\begin{algorithm}[t]
\caption{BAGM: Blockwise adaptive gradient with momentum for stochastic nonconvex optimization.}
   \label{alg:block-mom-adagrad}
\begin{algorithmic}[1]
   \STATE {\bfseries Input:} $\{\eta_t\} $; 
    $\{a_t\}$;
    $\{\beta_t\}$;
    $\epsilon > 0$.
   \STATE {\bfseries initialize} 
	$\theta_1 \in \R^d$;
	$v_0 \leftarrow 0$;
	$m_0 \leftarrow 0$
	 $A_0  \leftarrow 0$;
   \FOR{$t=1, 2, \dots, T$}
     \STATE{Sample an unbiased stochastic gradient $g_t$ }
     \STATE{$A_t = A_{t-1} + a_t$}
     \FOR{$b=1, 2, \dots, B$}
      \STATE{$v_{t, b} = v_{t-1, b} + a_t \|g_{t, \mG_b}\|^2_2/d_b$}
     \STATE{$\hat{v}_{t, b} = v_{t, b} / A_t$}
      \STATE{$m_{t, \mG_b} = \beta_t m_{t - 1,\mG_b} + (1 - \beta_t) g_{t, \mG_b}$}
     \STATE{$\theta_{t+1, \mG_b} = \theta_{t, \mG_b} - \eta_tm_{t, \mG_b}/(\sqrt{\hat{v}_{t, b}} + \epsilon)$}
     \ENDFOR
   \ENDFOR
\end{algorithmic}
\end{algorithm}

\subsection{Convergence Analysis}
\label{sec:adabcm_converge}

%\subsubsection{Assumptions}

We make the following assumptions. 

\begin{assumption} \label{assumption:smooth}
$F$ 
in (\ref{eq:expected_loss})
is lower-bounded (i.e., $F(\theta_*) = \inf_\theta F(\theta) >
-\infty$)  and 
$L$-smooth.
\end{assumption}

\begin{assumption} \label{assumption:variance}
Each block of stochastic gradient has bounded second moment,
i.e.,
$\CE_t[\|g_{t, \mathcal{G}_b}\|^2_2]/d_b \leq \sigma_b^2, \forall b \in [B],
\forall t$,
where the expectation is taken w.r.t. the random 
$f_t$. 
\end{assumption}
Assumption~\ref{assumption:variance}
implies that the variance of each block of stochastic gradient is upper-bounded by
$d_b\sigma_b^2$ (i.e., $\CE_t[\|g_{t, \mathcal{G}_b} -
\nabla_{\mG_b}F(\theta_t)\|^2_2] = \CE_t[\|g_{t, \mathcal{G}_b}\|^2_2] -
\|\nabla_{\mG_b}F(\theta_t)\|_2^2 \leq d_b\sigma_b^2$). 

We make the following assumption on sequence $\{\beta_t\}$.
This implies that 
we can
use,
for example, 
a constant $\beta_t = \beta$, or an increasing sequence $\beta_t = \beta (1 -
1/t^\tau)$.

\begin{assumption}  \label{assumption:beta}
$0 \leq \beta_t \leq \beta$ for some $0 \leq \beta < 1$.
\end{assumption}

%Sequences \ref{seq:1}, \ref{seq:2} and \ref{seq:3} satisfy the following assumptions:

\begin{assumption} \label{assumption:weight}
(i) $\{a_t\}$ is non-decreasing;
(ii) 
$a_t$ grows slowly
such that $\{A_{t-1}/A_t\}$ is non-decreasing and $A_t/(A_{t-1} + a_1) \leq \omega$ for some $\omega \geq 0$; 
(iii) $p \equiv \lim_{t\rightarrow \infty}A_{t-1}/A_t > \beta^2$.
\end{assumption}
This is satisfied by sequences \ref{seq:1}
(with $\omega = 1$ and $p = 1$), \ref{seq:2}
(with $\omega = 
(1 + 2^{\tau})/2$ and $p = 1$),
and \ref{seq:3}
(with $\omega = 
(1 + 1/\alpha)/2$ and $p = \alpha > \beta^2$). 

\begin{assumption} \cite{zou2018sufficient} \label{assumption:stepsize}
The stepsize $\eta_t$ is chosen such that $w_t = \eta_t/\sqrt{a_t/A_t}$ is ``almost" non-increasing, i.e.,
there exists a non-increasing sequence $\{z_t\}$ and positive constants $C_1$ and $C_2$ 
such that $C_1z_t \leq w_t \leq C_2z_t$ for all $t$. 
\end{assumption}
Assumption~\ref{assumption:stepsize}
is satisfied by 
the example sequences \ref{seq:1}, \ref{seq:2}, \ref{seq:3} when
$\eta_t =\eta/\sqrt{t}$ for some $\eta > 0$. 
Interested readers are
referred 
to \cite{zou2018sufficient}
for details.

%\subsubsection{Convergence Analysis}

As in weighted AdaEMA, we define a sequence of virtual estimates of the second
moment:
%\begin{eqnarray*}
$\tilde{v}_{t,b} = (v_{t-1, b} + a_t\CE_t[\|g_{t, \mathcal{G}_b}\|_2^2/d_b])/A_t$.
%\end{eqnarray*}
Let $\bar{v}_{T, B} \equiv \max_{1 \leq t \leq T}\CE[\max_{b} \tilde{v}_{t,b}]$
be its maximum over all blocks and training iterations,
where the expectation 
$\CE$
is taken over all 
random $f_t$'s.
Let $\hat{A}_{t, i} = \prod_{j = i + 1}^tA_{j - 1}/A_j$ for $1 \leq i < t$ and $\hat{A}_{t, t} = 1$. For a constant $\tilde{p}$
such that $\beta^2 <  \tilde{p} < p$, define $C_a = \prod_{j = 2}^NA_{j -
1}/(A_j\tilde{p})$, where $N$ is the largest index
for which $A_{j - 1}/A_j < \tilde{p}$. When $A_1/A_2
\geq \tilde{p}$, we set $C_a = 1$. 

The following Theorem provides a bound related to the gradients.

\begin{theorem} \label{theorem:nonconvex_momentum}
Suppose that 
Assumptions~\ref{assumption:smooth}-\ref{assumption:stepsize} hold. 
Let $\rho = \beta^2/\tilde{p}$. We have
%\begin{eqnarray*}
$\min_{1 \leq t \leq T}(\CE[\|\nabla F(\theta_t)\|^{4/3}_2])^{3/2} \leq
C(T)$,
%\end{eqnarray*} 
where\footnote{When $T = 1$, the second term in $C(T)$ (involving summation from $t=2$ to $T$) disappears.}
$C(T) = \frac{\sqrt{2\left(\bar{v}_{T, B} + \epsilon^2\right)}}{\eta_TT}\left[\frac{2C_2}{(1 - \beta)C_1}C_0  + C_4\left[\frac{\beta}{\sqrt{C_a(1 - \rho)}}\sum_{b=1}^B\sigma_bd_b\sum_{t=2}^Tw_t\left(\sqrt{\frac{A_t}{A_{t-1}}} - 1\right) \right.\right.$ 
$ \left.\left. + \sum_{b=1}^BC_b'\left[w_1\log\left(\frac{\sigma_b^2}{\epsilon^2} + 1\right) + \omega\sum_{t=1}^{T}\eta_t\sqrt{\frac{a_t}{A_t}}\right]\right]\right]$, 
$C_0 = F(\theta_1) - F(\theta_*)$, 
$C_4 = \frac{2C_2^2}{C_1^2\sqrt{C_a}(1 - \sqrt{\rho})(1 - \beta)}$, $C_b'  =  \frac{LC_2^3w_1d_b}{C_1^3C_a(1 - \sqrt{\rho})^2} + \frac{2C_3^2C_2\sigma_bd_b}{C_1}$, 
and $C_3 =   \frac{\beta/(1 - \beta)}{\sqrt{C_aA_{1}/A_2(1 - \rho)}}  + 1$. 
\end{theorem}
%Recall that when $B=d$, Algorithm~\ref{alg:block-mom-adagrad} reduces to weighted AdaEMA.
When $B=d$,
the bound 
here is tighter than that
in \cite{zou2018sufficient}, as 
we exploit heterogeneous second-order upper bound
(Assumption~\ref{assumption:variance}).
The following Corollary shows the bound with 
high
probability.
\begin{corollary} \label{corollary:nonconvex_momentum_highpro}
With probability at least $1 - \delta^{2/3}$, we have $\min_{1 \leq t \leq T}\|\nabla F(\theta_t)\|_2^2 \leq  C(T)/\delta$.
\end{corollary}

\begin{corollary} \label{corollary:nonconvex_momentum_seq}
Let $\tilde{\beta}_t = \prod_{i=1}^t\beta_i$, and
$\eta_t = \eta/(\sqrt{t}(1 - \tilde{\beta}_t ))$ for some positive constant $\eta$.
When $\sum_{t=1}^Ta_t = \mathcal{O}(T^\gamma)$ for some $\gamma > 0$
(which holds for sequences \ref{seq:1} and \ref{seq:2}), $C(T) = \mathcal{O}(\log(T)/\sqrt{T})$. 
When $a_t = \alpha^{-t}$ for some $0 < \alpha < 1$
(sequence \ref{seq:3}),
$C(T) = \mathcal{O}(1)$.
\end{corollary}
When $B=d$, we obtain the same non-asymptotic convergence rates as in \cite{zou2018sufficient}.  
Note that
SGD is analogous to BAGM with $B=1$, as they both use a single stepsize for all
coordinates and the convergence rates depend on the same second-order moment upper
bound in Assumption~\ref{assumption:variance}. 
With a decreasing
stepsize, 
SGD also has a convergence rate of $O(\log(T)/T)$,
%\footnote{there is no $a_t = \alpha^{-t}$ in sgd. \#*** does SGD still has a convergence rate of $O(\log(T)/T)$ in this case when $a_t = \alpha^{-t}$ for some $0 < \alpha < 1$?  *** no, it is only $O(1)$. actually, it is expected as shown in amsgrad paper. but this one works in practice. \#*** u dont hv $O(\log(T)/T)$ convergence when $a_t = \alpha^{-t}$ for some $0 < \alpha < 1$?}
which can be seen by 
setting their
stepsize
$\gamma_k$ to $\eta/\sqrt{k}$ in (2.4) of \cite{ghadimi2013stochastic}. 
Thus, our rate is as good as SGD. 

%In the sequel, we will show that BAGM has faster convergence than SGD due to heterogeneous second-order moment. 

Next, we compare the effect of $B$ on convergence. 
As $\bar{v}_{T, B}$ in $C(T)$ depends on the sequence $\{\theta_t\}$, a direct
comparison is difficult.
Instead, we study an 
upper bound 
looser 
than
$C(T)$. First, we introduce
the following assumption, which is stronger than Assumption~\ref{assumption:variance}
(that only bounds the expectation).
%to bound $\bar{v}_{T, B}$. 
\begin{assumption} \label{assumption:strict}
$\|g_{t, \mG_b}\|_2^2/d_b \leq G_b^2, \forall b\in [B]$ and $\forall t$. 
\end{assumption} 

With Assumption~\ref{assumption:strict}, 
it can be easily shown
that $\bar{v}_{T, B} \leq \max_bG_b^2$.
We can then
define a looser upper bound $\tilde{C}(T)$ by replacing
$\bar{v}_{T, B}$ in $C(T)$ 
%(in Theorem~\ref{theorem:nonconvex_momentum}) 
with $\max_bG_b^2$.  
%Then, we have $C(T)\leq \tilde{C}(T)$, and 
%we can compare $\tilde{C}(T)$ for different $B$'s. 
We proceed to compare the convergence using 
coordinate-wise stepsize (with $B=d$)
and 
blockwise stepsize
(with $B=\tilde{B}$ for some $\tilde{B}$).
Note that 
when $B=d$, Assumption~\ref{assumption:strict} becomes $g_{t, i}^2 \leq G_i^2$ for some $G_i$,
and Assumption~\ref{assumption:variance} becomes $\CE_t[g_{t, i}^2] \leq \sigma_i^2$ for some $ \sigma_i$.  
When $B=\tilde{B}$,
we assume that Assumption~\ref{assumption:variance} is tight in the sense that
$\sigma_b^2 \leq \sum_{i \in \tilde{\mG}_b}\sigma_i^2/d_b$,\footnote{Note that
$\frac{1}{d_b}\CE_t[\|g_{t, \tilde{\mG}_b}\|_2^2] = \frac{1}{d_b}\sum_{i\in\tilde{\mG}_b}\CE_t[g_{t,i}^2] \leq \frac{1}{d_b}\sum_{i\in\tilde{\mG}_b}\sigma_i^2$. On the other hand, 
$\CE_t[\|g_{t, \tilde{\mG}_b}\|_2^2]/d_b
\leq \sigma_b^2$.
Thus, this bound is tight in the sense that $\sigma_b^2\leq
\frac{1}{d_b}\sum_{i\in\tilde{\mG}_b}\sigma_i^2$.}
where $\tilde{\mG}_b$ is the set of indices in 
block 
$b$.
%of the gradient.
The following Corollary shows that 
blockwise stepsize
can have faster convergence than coordinate-wise stepsize.

\begin{corollary}  \label{corollary:nonconvex_momentum_comparison}
Assume that 
Assumption~\ref{assumption:strict} holds. 
%we have
%$\bar{v}_{T, B} \leq \max_bG_b^2$ and 
%$C(T) \leq \tilde{C}(T)$, 
Let $\tilde{C}_d(T)$ and $\tilde{C}_{\tilde{B}}(T)$ be the values of $\tilde{C}(T)$ for $B=d$ and $B=\tilde{B}$, respectively. 
Define
$r_1 := \frac{\sum_{b=1}^{\tilde{B}}\sum_{i \in \tilde{\mG}_b}\log\left(\sigma_i^2/\epsilon^2 + 1\right)}{\sum_{b=1}^{\tilde{B}}d_b\log\left(\sigma_b^2/\epsilon^2 + 1\right)}$, 
$r_2 := \frac{\sum_{b=1}^{\tilde{B}}\sum_{i \in \tilde{\mG}_b}\sigma_i}{\sum_{b=1}^{\tilde{B}}\sigma_bd_b}$ and $r_3 := \frac{\sum_{b=1}^{\tilde{B}}\sum_{i \in \tilde{\mG}_b}\sigma_i\log\left(\sigma_i^2/\epsilon^2 + 1\right)}{\sum_{b=1}^{\tilde{B}}\sigma_bd_b\log\left(\sigma_b^2/\epsilon^2 + 1\right)}$. Let $r_{\min} = \min(r_1, r_2, r_3)$. 
Then, 
$\frac{\tilde{C}_d(T)}{\tilde{C}_{\tilde{B}}(T)} \geq \min(1, r_{\min})\sqrt{(\max_b\max_{i \in \tilde{\mG}_b}G_i^2 + \epsilon^2)/(\max_bG_b^2 +
\epsilon^2)}$.
\end{corollary}
Note that $r_{\min}$ can be larger than $1$ as 
$\sigma_b^2 \leq \frac{1}{d_b}\sum_{i\in \tilde{\mG}_b}\sigma_i^2$. 
Corollary~\ref{corollary:nonconvex_momentum_comparison}
then indicates that blockwise adaptive stepsize will
lead to improvement if
%$\sqrt{\frac{\max_b\max_{i \in \tilde{\mG}_b}G_i^2 + \epsilon^2}{\max_bG_b^2 +
%\epsilon^2}}  > \frac{1}{r_{\min}}$. 
$\sqrt{(\max_b\max_{i \in \tilde{\mG}_b}G_i^2 + \epsilon^2)/(\max_bG_b^2 +
\epsilon^2)}  > \frac{1}{r_{\min}}$. 
Assume that the upper bound $G_b$ is tight so that $G_b^2 \leq \frac{1}{d_b}\sum_{i \in \tilde{\mG}_b}G_i^2$.  
Thus, $\max_b\max_{i \in \tilde{\mG}_b}G_i^2 \geq \max_bG_b^2$, and the above condition is likely to hold when $r_{\min}$ is close to $1$.  
From the definitions of $r_1$, $r_2$ and $r_3$, we can see that they get close to
$1$ when $\{\sigma_i^2\}_{i \in \tilde{\mG}_b}$ are close to $\sigma_b^2$
(i.e., $\{\sigma_i^2\}_{i \in \tilde{\mG}_b}$ has low variability). 
In particular, $r_{\min} = 1$ when $\sigma_i = \sigma_b$ for all $i \in \tilde{\mG}_b$ (note that $\sigma_i = \sigma_b \notimplies G_i = G_b$).   
This is empirically verified 
in 
Appendix~\ref{sec:verify_corollary}.
%Appendix~C.2.1.

%On the other hand, SGD is analogous to BAGM with $B=1$
%and they both depend
%on the same variance bound\footnote{*** they both depend on the same var bnd does not mean that their properties are similar. better
%remove this 2nd half of the sentence *** a common assumption for sgd is the bounded second moment, i.e., $\CE_t[\|g_t\|_2^2] \leq \sigma^2$. here in order to ensure consistency, we use $\sigma^2$ as the averaged value, i.e., $\CE_t[\|g_t\|_2^2]/d \leq \sigma^2$ \#*** how to see that? ref?}, i.e., $\CE_t[\|g_t\|^2_2]/d \leq \sigma^2$.   
%When $B=\tilde{B}=1$, $r_{\min}$ can be very small as the 
%gradient magnitudes can vary a lot across blocks. In this case, BAGM with $B=d$ will be faster. 
%\footnote{so i also delete the "In the sequel, we will show that BAGM has faster convergence than SGD due to heterogeneous second-order moment. " in the prev para. \#*** since u dont hv expt results, better remove it to avoid being asked by reviewers *** as i connect sgd and B=1 in the previous sentence, so here i compare B=d with B=1 as a surrogoate for the comparison between B=d and sgd. experimentally, it is similar to dist paper that it is difficult to perform experiment with B=1\#*** expt evidence?} 

\subsection{Uniform Stability and
Generalization Error}
\label{sec:gen_error_as}

Given a sample $S = \{z_i\}_{i=1}^n$ of $n$ examples drawn
i.i.d. from an underlying unknown data distribution $\D$, one often learns the model by minimizing the empirical risk:
%\begin{eqnarray*} \label{eq:erm_loss}
$\min_{\theta} \Phi_S(\theta) \equiv \frac{1}{n}\sum_{i=1}^nf(\theta; z_i)$,
%\end{eqnarray*}
where $\theta = M(S)$ is the output of a possibly randomized algorithm $M$ 
(e.g., SGD) 
%(e.g., SGD and BAGM) 
running on data $S$.  

\begin{definition}
\cite{hardt2016train} Let $S$ and $S'$ be two samples of size $n$ that differ in only one example. 
Algorithm $M$ is
{\em $\epsilon_{u}$-uniformly stable} if
$\epsilon_{stab}\equiv \sup_{S, S'}\sup_{z\in\D}\CE_M[f(M(S); z) - f(M(S'); z)] \leq \epsilon_{u}$.
\end{definition}
The generalization error \cite{hardt2016train}
is defined as 
%the expected difference 
$\epsilon_{gen}\equiv\CE_{S, M}[\Phi_S(M(S)) - F(M(S))]$, where the expectation is taken w.r.t. the sample $S$ and randomness of $M$. 
It is shown 
that the generalization error 
is bounded by the uniform stability of $M$,
i.e., $|\epsilon_{gen}| \leq \epsilon_{stab}$
\cite{hardt2016train}.
In other words, the more uniformly stable an algorithm is, the lower is its generalization error.
%Moreover, one can use this to convert an uniform stability bound (as obtained below) to  a generalization error bound.

Let $\Delta_t = \|\theta_t - \theta_t'\|_2$, and
$\tilde{\Delta}_{t}(z)=
|f(\theta_{t}; z) - f(\theta_{t}'; z)|$,
where $\theta_t, \theta_t'$ are the $t$th iterates of BAGM 
on $S$ and $S'$, respectively. 
The following shows how $\CE[\tilde{\Delta}_{t}(z)]$ (uniform stability) grows with $t$.

%\footnote{the proposition only serves to show how the uniform stability grows such that we can use to analyze the effect of using blocks.  a more throughout analysis w.r.t $T$ and $n$ requires to further bound $W_t$ explicitly. but i currently fail to do so. a loose bound would be $\CE[|f(\theta_{t+1}; z) - f(\theta_{t+1}'; z)|] \leq \frac{2\tilde{\gamma}C_2}{C_1}\sqrt{\left[w_1^2\sum_{b=1}^Bd_b\log\left(\frac{\sigma_b^2}{\epsilon^2} + 1\right) + d\omega\sum_{k=1}^{t}\eta_k^2\right]t} $, in which factor $1/n$ is missing and we cannot observe the benefit of using blocks. \#*** u said u want to know how it varies with $t$. so what's the observation from the prop below?}

%how $\sup_z\CE[|f(\theta_{t+1}; z) \!-\! f(\theta_{t+1}'; z)|]$ grows that implies the growth of the generalization error. 
\begin{proposition} \label{prop:gen_error}
Assume that $f$ is $\tilde{\gamma}$-Lipschitz\footnote{In other words, $|f(\theta;z) - f(\theta';z)| \leq \tilde{\gamma}\|\theta - \theta'\|_2$ for any $z$.}. 
Suppose that 
Assumptions~\ref{assumption:smooth}-\ref{assumption:stepsize} hold,
$\beta_t=0$,
and $\theta_1 = \theta_1'$.
For any $z\in \D$, we
have
$
\CE[\tilde{\Delta}_{t+1}(z)] \leq  \frac{2\tilde{\gamma}C_2}{nC_1}\sqrt{\left[w_1^2\sum_{b=1}^Bd_b\log\left(\sigma_b^2/\epsilon^2 + 1\right) 
+ d\omega\sum_{k=1}^{t}\eta_k^2\right]t} + \left(1 - \frac{1}{n}\right)\tilde{\gamma}W_t,
$
where $W_t = \tilde{\gamma}\sum_{k=1}^t\eta_k\CE\left[\max_b\left|1/(\sqrt{\hat{v}_{k, b}} + \epsilon) - 1/(\sqrt{\hat{v}_{k, b}'} + \epsilon)\right|\right] + L\sum_{k=1}^t\eta_k\CE\left[\Delta_k/(\sqrt{\min_b\hat{v}_{k, b}} + \epsilon)\right]$.
\end{proposition}

Using Proposition~\ref{prop:gen_error},
we can study how $B$ affects the growth of $\CE[\tilde{\Delta}_{t+1}(z)]$.
%that implies the growth of generalization error. 
Consider the first term on the RHS of the bound.
Recall that $\sigma_b^2 \leq \frac{1}{d_b}\sum_{i\in \tilde{\mG}_b}\sigma_i^2$.
If $\sigma_b^2 = \frac{1}{d_b}\sum_{i\in \tilde{\mG}_b}\sigma_i^2$, 
this term is smallest when $B=d$; otherwise,
%If $\sigma_b^2 < \frac{1}{d_b}\sum_{i\in \tilde{\mG}_b}\sigma_i^2$, 
some $B < d$ will make this term smallest.  
For the $W_t$ term,
as $\frac{1}{\tilde{B}}\sum_{b=1}^{\tilde{B}}\hat{v}_{k, b} (\text{on which } B=1 \text{ depends}) \geq \min_b\hat{v}_{k, b} \geq \min_b\min_{i \in \tilde{G}_b}\hat{v}_{k, i} (\text{on which } B=d \text{ depends})$,
the $\min_b\hat{v}_{k, b}$ term inside is
typically the smallest when $B=d$, and is largest when $B=1$.
Thus, the first term of the bound is small when $B$ is close to $d$, while $W_t$ is small when $B$ approaches $1$.
As a result, 
for $B$ equals to some $1<\tilde{B} <d$,
$\CE[\tilde{\Delta}_{t+1}(z)]$, and thus
the generalization error,
grows slower than those of $B=d$ and $B=1$.  
%This in turn shows that of BAGM grows slower for some $1<\tilde{B} <d$.

%of $\theta_{t+1}$ is bounded by $\sup_{z\in\D}\CE[\tilde{\Delta}_{t+1}(z)]$, the use of blocks can then allow the generalization error of the model returned by

%The expected risk can be decomposed into three terms: 
%\begin{eqnarray*}
%\CE[F(\theta_t)] = \CE[\min_{\theta}\Phi_S(\theta)]+ \CE[F(\theta_t) - \Phi_S(\theta_t)] + \CE[\Phi_S(\theta_t) - \min_{\theta}\Phi_S(\theta)],
%\end{eqnarray*}
%where the second term measures the generalization error and the last term estimates the optimization error. 
%The convergence rate on gradient norm cannot translate to the one on optimization error due to the non-convexity of the problem. However, we may expect that faster convergence on gradient norm may indicate faster convergence on function value in practice.   
%Corollary~\ref{corollary:nonconvex_momentum_comparison} suggests that $B=\tilde{B}$ can converge faster than $B=d$ and $B=1$, while Proposition~\ref{prop:gen_error} implies that $B=\tilde{B}$ can potentially have lower generalization error than $B=d$ and $B=1$.  Thus, $B=\tilde{B}$ may yield a solution with lower expected risk.  

%%%%%%%%%%%%%%%%%%%%%%%%%%%%%%%%%%%%%%%%%%%%%%%%%%%%%%%%%%%%%%%%%%%%%%%%%%%%%%%%

\section{Experiments}

In this section, we perform experiments on 
CIFAR-10
(Section~\ref{sec:resnet}), 
ImageNet
(Section~\ref{sec:imagenet}), and WikiText-2 (Section~\ref{sec:lm}). 
All the experiments are run on a AWS p3.16 instance with 8 NVIDIA V100 GPUs. 
We introduce four block construction strategies: 
%\begin{enumerate}[label=\textbf{B.\arabic*}]
%\item 
\setlabel{{\bf B.1}}{bs:1}: Use a single adaptive stepsize
for each parameter tensor/matrix/vector. A parameter tensor can be the kernel tensor in a convolution layer, a parameter matrix can be the weight matrix in a fully-connected layer, and a parameter vector can be a bias vector; 
%\item 
\setlabel{{\bf B.2}}{bs:2}: Use an adaptive stepsize for each output dimension 
of the parameter matrix/vector in a fully connected layer, and 
an adaptive stepsize
for each output channel in the convolution layer;  
%\item 
\setlabel{{\bf B.3}}{bs:3}: Use an adaptive stepsize for each output dimension 
of the parameter matrix/vector in a fully connected layer, and 
an adaptive stepsize
for each kernel
%\footnote{$C_{in}$ kernels produce an output channel. \#*** each kernel yields an output channel, right? if so, it's the same as b2}
in the convolution layer; 
%\item 
\setlabel{{\bf B.4}}{bs:4}: Use an adaptive stepsize for each input dimension 
of the parameter tensor/matrix, and an adaptive stepsize for each parameter vector. 
%\end{enumerate}

We compare 
the proposed BAGM 
(with block construction approaches \ref{bs:1}, \ref{bs:2},
\ref{bs:3}, \ref{bs:4})
with the following baselines:
(i)
Nesterov's accelerated gradient (NAG)
\cite{sutskever2013importance}; and (ii) Adam
\cite{kingma2014adam}. These two algorithms are widely applied in deep
networks \cite{zaremba2014recurrent,he2016deep,vaswani2017attention}. 
NAG provides a strong baseline with good generalization performance, while Adam
serves as a fast counterpart with coordinate-wise adaptive stepsize. 

As grid search for all hyper-parameters is very computationally expensive, we 
only tune
the most important ones
using a validation set and 
fix the rest.
We use a constant $\beta_t = \beta$ (momentum parameter) and exponential increasing
sequence \ref{seq:3} with $\alpha =  0.999$ for BAGM. For Adam, we also fix its
second moment parameter to $0.999$ and tune its momentum parameter. 
Note that with such configurations, Adam is a special case of BAGM with $B=d$ (i.e., weighted AdaEMA). 
For all the
adaptive methods, 
we use
$\epsilon=10^{-3}$
as suggested in \cite{zaheer2018adaptive}.

%%%%%%%%%%%%%%%%%%%%%%%%%%%%%%%%%%%%%%

\subsection{ResNet on CIFAR-10}
\label{sec:resnet}

We train a deep residual network
from the MXNet Gluon CV model zoo\footnote{\url{https://github.com/dmlc/gluon-cv/blob/master/gluoncv/model_zoo/model_zoo.py}}
on the CIFAR-10 data set. 
We use the 56-layer and 110-layer networks as in \cite{he2016deep}. 
10\% of the training data are carved out as validation set. We perform grid search
using the validation set
for the initial stepsize
$\eta$ and momentum parameter $\beta$ on ResNet56. The obtained hyperparameters are
then also used on ResNet110. 
We follow the similar setup as in \cite{he2016deep}. 
Details are in 
Appendix~\ref{sec:exp_setup_cifar10}.
%Appendix~C.2.
  
Table~\ref{tab:val_resnet} shows
the testing errors of the various
methods. 
With a large $\epsilon=10^{-3}$, the testing performance of Adam matches that of NAG. 
This agrees with \cite{zaheer2018adaptive} that a larger
$\epsilon$ reduces adaptivity and improves generalization performance. It also
agrees with Proposition~\ref{prop:gen_error} that the bound is smaller when $\epsilon$ is larger. 
Specifically, Adam has lower testing error than NAG on ResNet56 but higher on ResNet110.  
For both models, BAGM reduces the testing error over Adam for all block construction strategies used.  
In particular, except \ref{bs:3}, BAGM with all other schemes outperform
NAG.

 %&  \multirow{2}{*}{$\eta$} &  \multirow{2}{*}{$\beta$} &   \multicolumn{2}{c}{test error (\%)}  \\ \cline{4-5}

\begin{table}[ht]
\begin{center}
\begin{tabular}{c|c|c|c|c}
 \hline 
  &   \multicolumn{2}{c}{CIFAR-10} &  \multicolumn{2}{|c}{ImageNet}   \\  
  %\cline{2-3} \cline{4-5}
 & ResNet56 & ResNet110 & \multicolumn{2}{c}{ResNet50}  \\\hline 
& \multicolumn{2}{c|}{test error (\%)}  & top-1 validation error (\%) & top-5 validation error (\%) \\\hline 
NAG & $6.91$ & $6.28$ & $20.94$ & $5.51$ \\ 
Adam  & $6.64$  & $6.35$ & $21.04$  & $5.47$ \\\hline 
BAGM-\ref{bs:1}  & $\mathbf{6.26}$ & $\mathbf{5.94}$ & $\mathbf{20.79}$ & $5.43$ \\
BAGM-\ref{bs:2} & $6.51$  & $6.27$ & $20.90$ & $\mathbf{5.39}$ \\
BAGM-\ref{bs:3} & $6.52$  & $6.31$ & $20.88$ & $5.52$  \\
BAGM-\ref{bs:4} & $6.38$  & $6.02$ & $20.82$ & $5.48$ \\
\hline
\end{tabular}
\end{center}
\caption{Testing errors (\%) on CIFAR-10 and validation set errors (\%) on ImageNet. 
The best results
are bolded.
}
\label{tab:val_resnet}
\vskip -0.1in
\end{table}

\begin{figure}[ht]
\begin{center}
\subfigure[Training error.]{\includegraphics[width=0.32\columnwidth]{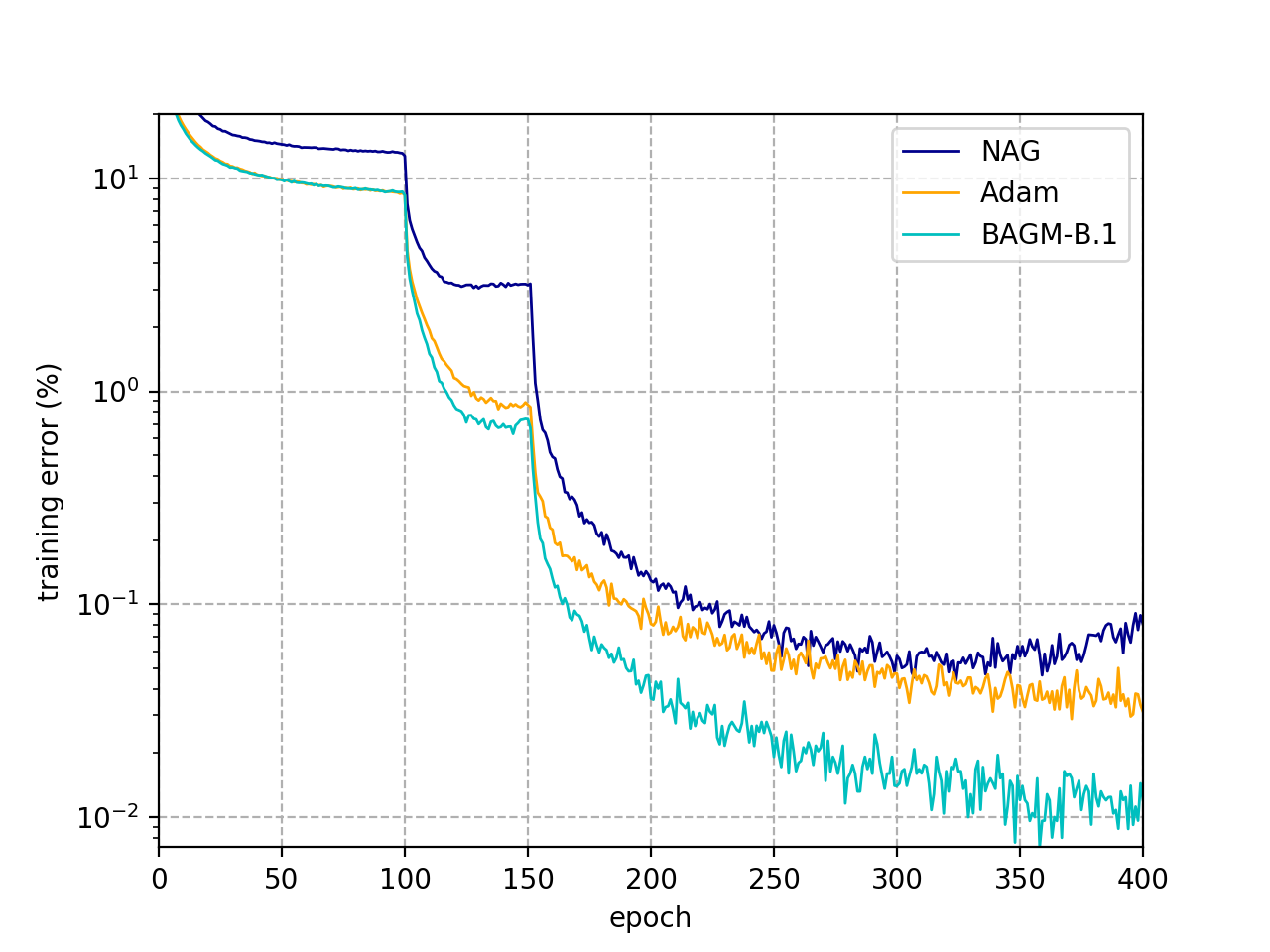}}
\subfigure[Testing error.]{\includegraphics[width=0.32\columnwidth]{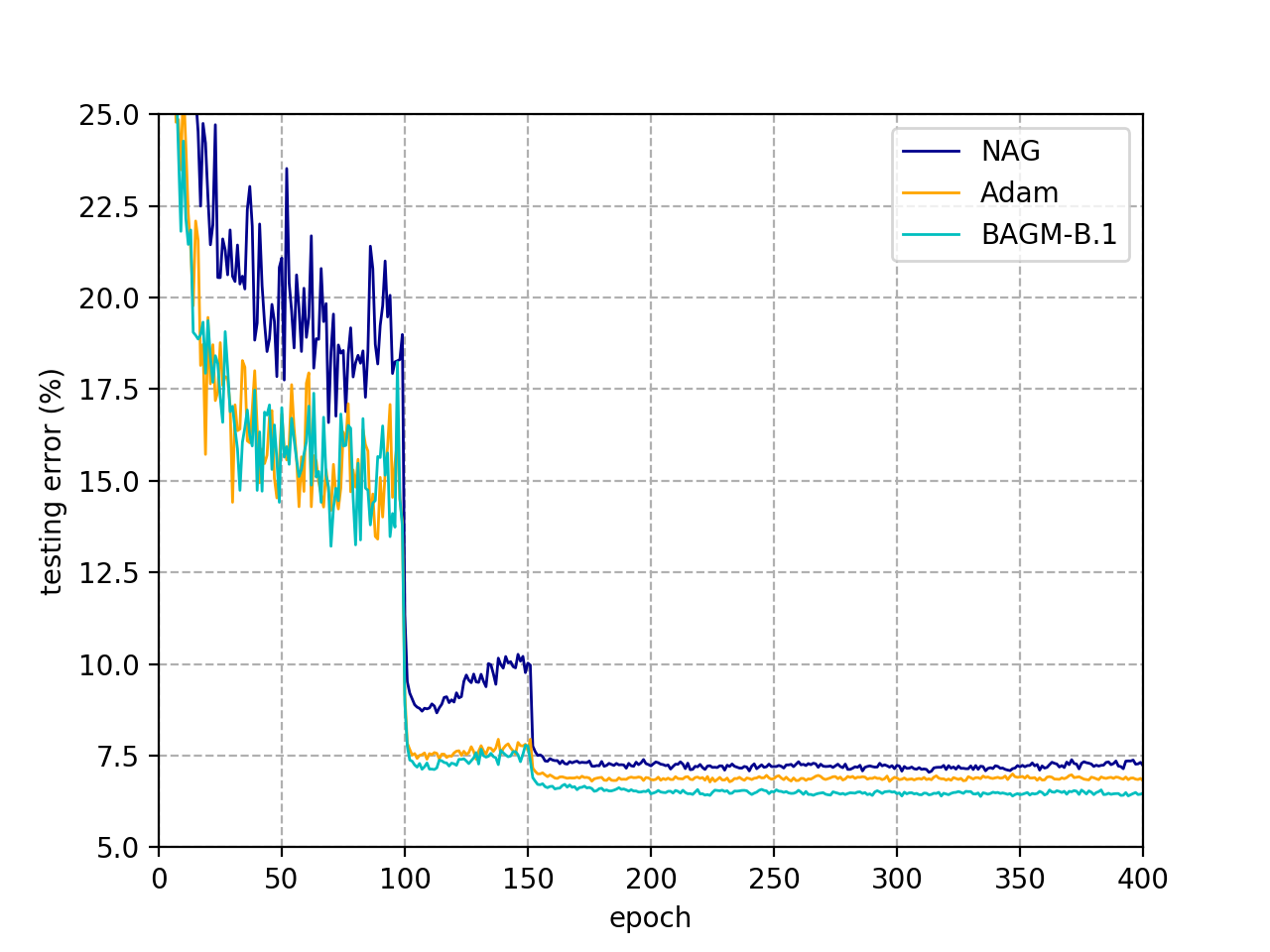}} 
\subfigure[Generalization error.]{\includegraphics[width=0.32\columnwidth]{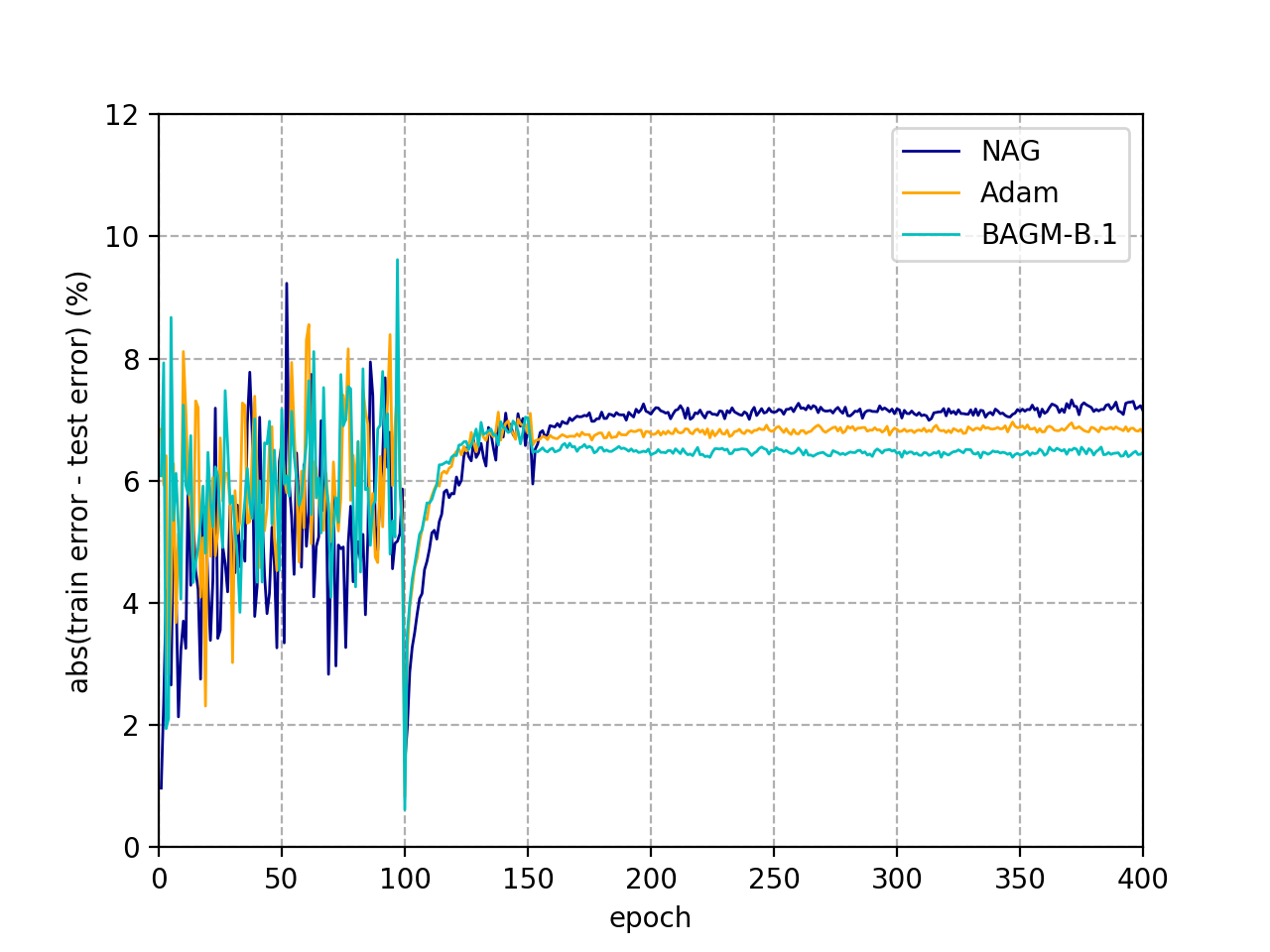}} \\
\vskip -0.15in
\subfigure[Training error.]{\includegraphics[width=0.32\columnwidth]{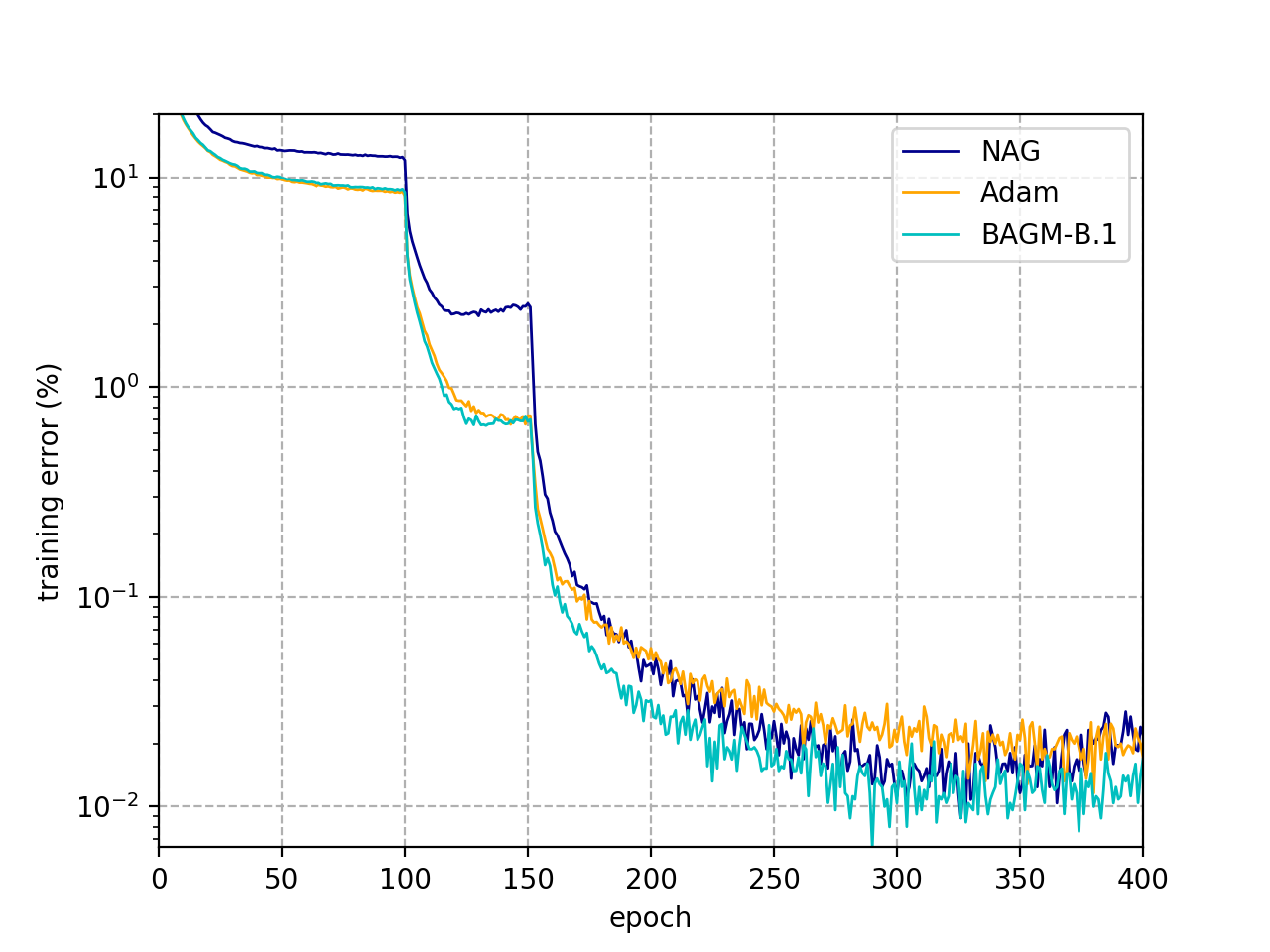}}
\subfigure[Testing error.]{\includegraphics[width=0.32\columnwidth]{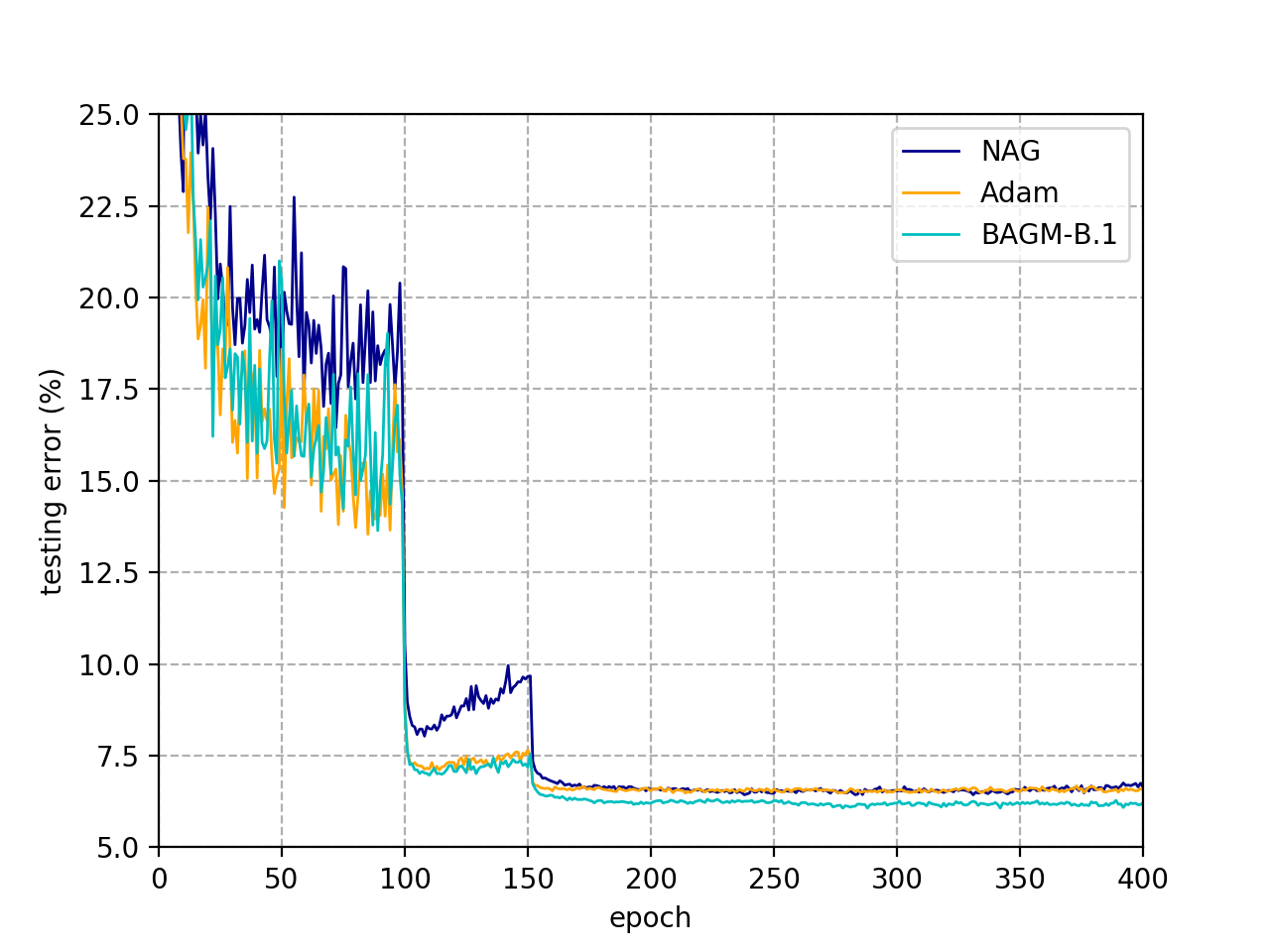}}
\subfigure[Generalization error.]{\includegraphics[width=0.32\columnwidth]{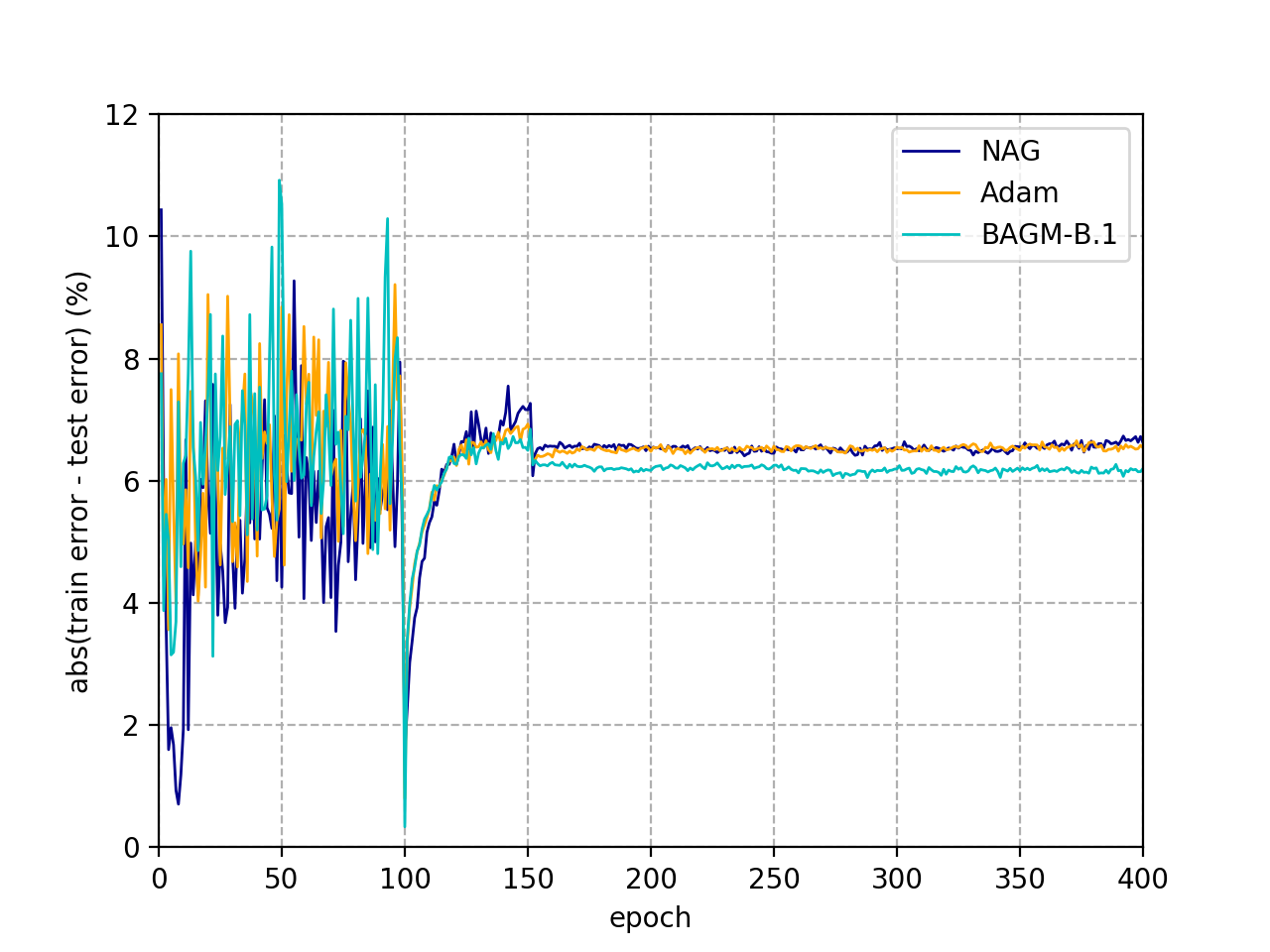}} \\
\caption{Results on residual networks. Top:
ResNet56; Bottom:
ResNet110.
The curves are obtained by running the algorithms with the best hyper-parameters obtained by grid search. 
The training error (\%) is plotted in log scale. To reduce statistical variance, results are averaged over $5$ repetitions.}
\label{fig:resnet}
\end{center}
\vskip -0.1in
\end{figure}

Convergence of the training, testing, and generalization errors (absolute difference between training error and testing error) are shown
in Figure~\ref{fig:resnet}.\footnote{To reduce clutterness, we only show results of the block construction scheme
BAGM-\ref{bs:1}, which gives the lowest testing error among the proposed block
schemes. Figure with the full results is shown in 
Appendix~\ref{sec:exp_setup_cifar10}.}
%Appendix~C.2..}
As can be seen, on both models,  
BAGM-\ref{bs:1} converges to a lower training error rate than Adam. 
This agrees with
Corollary~\ref{corollary:nonconvex_momentum_comparison} 
that blockwise adaptive methods can have faster convergence than their counterparts with element-wise
adaptivity. Moreover, the generalization error of BAGM-\ref{bs:1} is smaller than
Adam, which agrees with Proposition~\ref{prop:gen_error} that blockwise adaptivity
can have a slower growth of generalization error. On both models, BAGM-\ref{bs:1}
gives the smallest generalization error, while NAG has the highest generalization error on ResNet56.  
Overall, the proposed methods can accelerate convergence and improve generalization performance. 

%%%%%%%%%%%%%%%%%%%%%%%%%%%%%%%%%%%%%%

\subsection{ImageNet Classification}
\label{sec:imagenet}

In this experiment, we train a 50-layer ResNet model on ImageNet 
\cite{russakovsky2015imagenet}.
The data set has 1000 classes, 1.28M training samples, and 50,000 validation images.
As the data set does not come with labels for its test set, we evaluate its
generalization performance on the validation set.  We use the ResNet50\_v1d network from the MXNet Gluon CV model zoo. We train the FP16 (half precision) model on 8 GPUs, each of which processes 128 images in each iteration. 
More details are in 
Appendix~\ref{app:imagenet}.
%Appendix~C.3.

Performance on the validation set is shown in
Table~\ref{tab:val_resnet}.
%\footnote{we do not tune $\beta$, instead we reuse the obtained $\beta$ in cifar-10 experiments. \#*** opt $\beta$ is again 0}
As
can be seen, 
BAGM with all the block schemes (particularly
BAGM-\ref{bs:1})
achieve lower top-1 errors than Adam and NAG.
As for the top-5 error,
BAGM-\ref{bs:2} obtains the lowest, which is then followed by
BAGM-\ref{bs:1}. 
Overall, BAGM-\ref{bs:1} has the best performance on both CIFAR-10 and ImageNet. 

%%%%%%%%%%%%%%%%%%%%%%%%%%%%%%%%%%%%%%%%%%%%%%%%%%%%%%%%%%%%%%%%%%%%%%%%%%%%%%%%

\subsection{Word-Level Language Modeling}
\label{sec:lm}

In this section, we train the
AWD-LSTM word-level language model \cite{merity2017regularizing} on the WikiText-2 (WT2) 
data set \cite{merity2016pointer}. 
We use the publicly available implementation 
%of the AWD-LSTM language model 
in the Gluon NLP
toolkit\footnote{\url{https://gluon-nlp.mxnet.io/}.}. 
%the CIFAR-10 experiment 
We perform grid search on 
the initial learning rate and momentum parameter
as in Section~\ref{sec:resnet},  and
set the weight decay to $1.2\times10^{-6}$
as in \cite{merity2017regularizing}. 
More details on the setup are in 
Appendix~\ref{app:lm}.  
%Appendix~C.4.
As there is no convolutional layer, \ref{bs:2} and \ref{bs:3} are the same.
%so we only perform experiment on \ref{bs:2}. 
Table~\ref{tab:test_lm}
shows 
the testing perplexities,
the lower the better. As can be seen, 
all adaptive methods achieve lower test perplexities than NAG, and
%In this experiment, 
BAGM-\ref{bs:2} obtains the best result. 

\begin{table*}[ht]
\begin{center}
\begin{tabular}{c|cc|ccc}
 \hline 
 & NAG & Adam & BAGM-\ref{bs:1} & BAGM-\ref{bs:2} & BAGM-\ref{bs:4} \\\hline 
test perplexity & $65.75$ & $65.40$ & $65.42$ & $\mathbf{65.29}$ & $65.55$\\ 
\hline
\end{tabular}
\end{center}
\caption{Testing perplexities on WikiText-2. Results are averaged over $3$
repetitions. }
\label{tab:test_lm}
\end{table*}

%%%%%%%%%%%%%%%%%%%%%%%%%%%%%%%%%%%%%%%%%%%%%%%%%%%%%%%%%%%%%%%%%%%%%%%%%%%%%%%%

\section{Conclusion}

In this paper, we proposed 
%partitioning the network parameters into blocks and introduce 
adapting the stepsize 
for each parameter block, instead of for each individual parameter as in Adam and
RMSprop.
%It is expected that the blockwise adaptivity can balance adaptivity and generalization. 
Convergence and uniform stability analysis shows that it can have faster convergence and lower
generalization error
than its counterpart with coordinate-wise adaptive stepsize. 
Experiments on image classification and language modeling
confirm
these theoretical results.
%CIFAR-10, ImageNet and WikiText-2 
%demonstrate that the proposed method exhibits faster convergence and better generalization.

\bibliography{block}
\bibliographystyle{plain}

%\end{document}

\newpage
\appendix

%%%%%%%%%%%%%%%%%%%%%%%%%%%%%%%%%%%%%%%%%%%%%%%%%%%%%%%%%%%%%%%%%%%%%%%%%%%%%%%%

\section{Online Convex Learning}
\label{app:online}

In online learning, the learner picks a prediction $\theta_t \in \R^d$ at round $t$, and then suffers a loss $f_t(\theta_t)$. The goal of the learner is to choose $\theta_t$ 
and achieve a low regret w.r.t.
an optimal predictor $\theta_* = \arg\min_{\theta}\sum_{t=1}^Tf_t(\theta)$ in
hindsight.
The regret (over $T$ rounds) is defined as 
\begin{eqnarray} \label{eq:regret}
R(T) \equiv \sum_{t=1}^Tf_t(\theta_t) - \inf_{\theta}\sum_{t=1}^Tf_t(\theta).
\end{eqnarray}

%%%%%%%%%%%%%%%%%%%%%%%%%%%%%%%%%%%%%%

\subsection{Proposed Algorithm}

The proposed procedure, which will be called
blockwise adaptive gradient
(BAG),
is shown in Algorithm~\ref{alg:block-adagrad}.
Compared to Adagrad, 
each block, instead of each coordinate, has its own learning rate.

\begin{remark}
When $B=d$ (i.e., each block has only one coordinate),
Algorithm~\ref{alg:block-adagrad}
reduces to Adagrad.
When $B=1$ (i.e., all coordinates are grouped together),
Algorithm~\ref{alg:block-adagrad}
produces the update:
$\theta_{t+1} = \theta_t - \eta\frac{g_t}{\|g_{1:t}\|_2/\sqrt{d} + \epsilon}$
with a global
adaptive learning rate. 
\end{remark}

\begin{algorithm}[t]
\caption{BAG: Blockwise adaptive gradient
for online convex learning.}
   \label{alg:block-adagrad}
\begin{algorithmic}[1]
   \STATE {\bfseries Input:} $\eta > 0$; 
    $\epsilon > 0$.
   \STATE {\bfseries initialize} 
	$\theta_1 \in \R^d$;
	$v_0 \leftarrow 0$
   \FOR{$t=1, 2, \dots, T$}
     \STATE{Receive subgradient $g_t \in \partial f_t(\theta_{t})$ of $f_t$ at $\theta_{t}$}
     \FOR{$b=1, 2, \dots, B$}
     \STATE{$v_{t, b} = v_{t-1, b} + \|g_{t, \mG_b}\|^2_2/d_b$}
     \STATE{$\theta_{t + 1, \mG_b} = \theta_{t, \mG_b} - \eta g_{t, \mG_b}/(\sqrt{v_{t, b}} + \epsilon)$}
     \ENDFOR
   \ENDFOR
\end{algorithmic}
\end{algorithm}

%%%%%%%%%%%%%%%%%%%%%%%%%%%%%%%%%%%%%%

\subsection{Regret Analysis}
\label{sec:bc_regret_comparison}

First, we make the following assumptions.

\begin{assumption} \label{assumption:convex}
Each
$f_t$ in (\ref{eq:regret}) is convex  but possibly nonsmooth. There exists a subgradient
$g \in \partial f_t(\theta)$ such that
$f_t(\theta') \geq f_t(\theta) + \langle g, \theta' - \theta \rangle$ 
for all $\theta, \theta'$.
\end{assumption}

\begin{assumption} \label{assumption:bounded_dist}
Each parameter block is in a ball of the corresponding optimal block
throughout the iterations.
In other words, 
$\max_{t}\|\theta_{t, \mG_b} - \theta_{*, \mG_b}\|_2 \leq D_{b}$
for all $b \in [B]$,
where $\theta_{*, \mG_b}$ is the subvector of $\theta_*$
for block $b$.  
\end{assumption}

\begin{theorem} \label{theorem:convex_regret}
Suppose that Assumptions~\ref{assumption:convex} and \ref{assumption:bounded_dist}
hold. Then,
%\footnote{large block contributes more to the regret, but it does not mean that it makes regret smaller. based on the discussion in the later part, the regret is smaller when the set $\mathcal{G}_b = \{i| \CE[g_{t, i}^2] \leq \sigma_b^2\}$ for some smallest $\sigma^2_b$ is larger. \#*** curious, how will $d_b$ affact the soln? which block size will hv the largest contribution to regret? any intuition?} 
\begin{eqnarray} \label{eq:block_regret}
R(T) \leq  \sum_{b=1}^B\left[\frac{1}{2\eta\sqrt{d_b}}D_b^2 + \eta\sqrt{d_b}\right]\|g_{1:T, \mG_b}\|_2.
\end{eqnarray}
\end{theorem}

\begin{remark}
When $B=d$,
by setting $D_b = D_{\infty}$ for all $b \in [B]$, where $D_\infty$ is some constant such that $\max_{t}\|\theta_t - \theta_*\|_{\infty} \leq D_\infty$,
the regret bound reduces to that 
of Adagrad
in Theorem~5 of
\cite{duchi2011adaptive}.
\end{remark}

By Jensen's inequality, the last term of (\ref{eq:block_regret}) is minimized when $B=d$. However, the comparison with Adagrad is indeterminate in the first term due to the constant $D_b$. 

In the following, we provide an example showing that 
when gradient 
magnitudes
for elements in the same block have the same upper bound,
blockwise adaptive learning
rate can lead to lower regret than coordinate-wise adaptive learning rate (in Adagrad).
This then indicates that blockwise adaptive method can potentially be beneficial in
training deep networks, 
as its architecture can be naturally divided into blocks and parameters in the same block are likely to have gradients with similar magnitudes.  

Let $f_t$ be the hinge loss for a linear model:
\begin{eqnarray} \label{eq:hinge}
f_t(\theta_t) = \max(0, 1 - y_t\langle\theta_t, x_t\rangle),
\end{eqnarray}
where $y_t \in \{-1, 1\}$ is the label and $x_t \in \R^d$ is the feature vector.
Assume that input $x_t$ is partitioned into $\tilde{B}$ blocks. 
For each $i$ in input block
$b$,
with probability $p_b$,
%\footnote{this is more general so that it cover many cases. \#*** why not just make $p_b=1$? the argument here is not related to sparsity of $x_t$}
$x_{t, i} \sim \mathcal{N}(c_b y_t, \gamma_b^2)$ 
for some given $c_b, \gamma_b$,
and $x_{t, i} = 0$ otherwise.
Then, $\CE[g_{t, i}^2] \leq p_b(c_b^2 + \gamma_b^2)$, and the expected gradient
magnitudes
for elements in the same input block have the same upper bound. 
Taking expectation of the gradient terms in (\ref{eq:block_regret}), we have,
for all $b$'s,
\begin{eqnarray*}
\CE[\|g_{1:T,\mG_b}\|_2]  \leq \sqrt{\sum_{i \in \mathcal{G}_b}\sum_{t=1}^T\CE[g_{t, i}^2]} \leq \tau_b\sqrt{d_bp_bT}, 
\end{eqnarray*}
where $\tau_b^2 = c_b^2 + \gamma_b^2$. 
Thus, 
with $B=\tilde{B}$ and the gradient partitioned in the same way as the input
features,
(\ref{eq:block_regret}) 
reduces to
\begin{eqnarray} \label{eq:regret_abstract_block}
\CE[R(T)]  \leq  \sum_{b=1}^{\tilde{B}}\tau_b\left[\frac{1}{2\eta}D_b^2+ \eta d_b\right]\sqrt{p_bT}.
\end{eqnarray}
On the other hand, for Adagrad,
$B=d$, and
Assumption~\ref{assumption:bounded_dist} becomes
%\footnote{cannot use blockwise assumption on coordinate-wise case, as the bound would be very loose.\#*** what if u dont chg the assumption (ie, still use the blkwise assumption), can u say sth about the bound in (\ref{eq:block_regret})?}
$\max_t (\theta_{t, i} - \theta_{*, i})
\leq D_i$ for some 
$D_i$.
The bound in (\ref{eq:block_regret}) reduces to
\begin{eqnarray} \label{eq:regret_abstract_coordinate}
\CE[R(T)]  \leq  \sum_{b=1}^{\tilde{B}}\tau_b\left[\frac{1}{2\eta}\sum_{i \in \tilde{\mathcal{B}}_b}D_i^2 + \eta d_b\right]\sqrt{p_bT}, 
\end{eqnarray}
where $\tilde{\mathcal{B}}_b$ is the set of indices in the $b$th input block. 
We assume that Assumption~\ref{assumption:bounded_dist} is tight. 
Then 
$D_b^2 \leq \sum_{i \in
\tilde{\mathcal{B}}_b}D_i^2$,
and the bound in 
(\ref{eq:regret_abstract_block}) is smaller than that in
(\ref{eq:regret_abstract_coordinate}). 

Figure~\ref{fig:synthetic_cvx} compares
BAG with $B=1, 2, 3, 4$,  and $d$ 
($= 100$)
on a
synthetic data set.
At round $t$, we randomly sample class label $y_t \in \{-1, 1\}$ with equal
probabilities. 
The first $50$ features are sampled independently from $\mathcal{N}(
10y_t, 100)$ with probability $0.5$, and zero otherwise.
The last $50$ features are sampled independently from $\mathcal{N}(-5y_t, 25)$ with probability $0.4$, and zero otherwise. For
$B=2$, we partition the elements of gradient $g_t$ into two blocks, one for the first 
$50$ coordinates, and the other for the rest (and thus exactly the same as the
input block structure). For $B=3$, we form the first block using the first $35$ coordinates, the second block with the next $30$ coordinates, 
and the third block with the remaining $35$ elements. The block 
structure is thus different from the input block structure. 
For $B=4$, the coordinates of gradient $g_t$ are divided into four blocks each of $25$
elements.  
We initialize $\theta_1$ to zero,
fix $\epsilon=10^{-8}$ and $\eta=0.01$.
The expected regret
is estimated 
by averaging over $100$ repetitions.
As can be seen
from Figure~\ref{fig:synthetic_cvx},
BAG with $B=2$ and $4$ achieve lower regrets than the others.
BAG with $B=3$ is a little worse but still performs better than 
$B=d$.  
For $B=1$, the mismatch in block structures is severe and the performance is 
worst.

\begin{figure}[ht]
\vskip -0.1in
\begin{center}
\includegraphics[width=0.5\columnwidth]{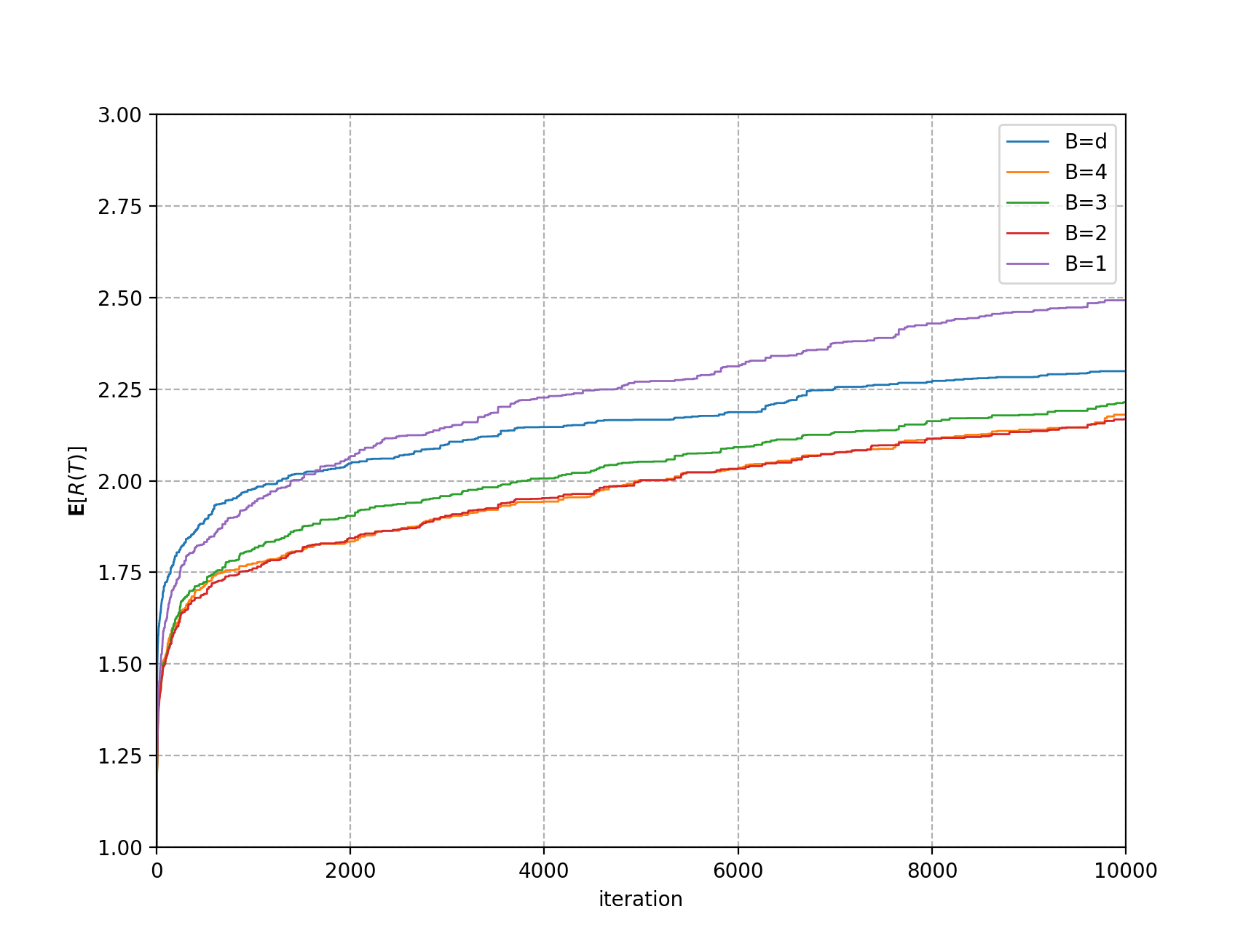}
\vskip -0.1in
\caption{Expected regret on a synthetic data set.}
\label{fig:synthetic_cvx}
\end{center}
\end{figure}

%%%%%%%%%%%%%%%%%%%%%%%%%%%%%%%%%%%%%%

\subsection{Excess Risk}

To measure the generalization ability, one is interested in minimizing the expected loss (\ref{eq:expected_loss}). 
Here,
the expectation is taken w.r.t. the distribution of random (loss) function $f$ (i.e., $f_t$'s are generated i.i.d.). 
When the distribution of $f$ corresponds to a finite training set,
(\ref{eq:expected_loss}) reduces to empirical risk minimization. 
The goal is to find parameter $\hat{\theta}$ with good 
generalization ability, i.e., small excess risk:
\begin{eqnarray*} \label{eq:excess_risk}
\mathcal{E}(\theta) \equiv F(\hat{\theta}) - \min_{\theta}F(\theta).
\end{eqnarray*}
Using the online-to-batch conversion \cite{cesa2004generalization}, one can
convert the regret bound (on past data) to excess risk (on unseen data)
bound. In particular, we have the following corollary.
\begin{corollary} \label{corollary:convex_convergence}
\cite{cesa2004generalization}
Assume that the loss is bounded in $[0, 1]$.  
If $f_t$'s are generated i.i.d., with probability greater than $1 - \delta$, we have 
$\mathcal{E}\left(\frac{1}{T}\sum_{t=1}^T\theta_t\right) \leq \frac{R(T)}{T} + 2\sqrt{\frac{2\log(1/\delta)}{T}}$.
\end{corollary}
Thus, achieving lower regret can be seen as obtaining better generalizarion performance. 

%%%%%%%%%%%%%%%%%%%%%%%%%%%%%%%%%%%%%%

\subsection{Least Squares Problem}

Consider the under-determined least squares problem:
\begin{eqnarray} \label{eq:ls}
\min_{\theta} \|X\theta - y\|^2_2,
\end{eqnarray}
where $X \in \R^{n\times d}$ with $n < d$,
and $y \in \R^n$. We assume that $XX^T$ is invertible.  
Any stochastic gradient descent method 
on problem (\ref{eq:ls})
with a global
stepsize outputs a trajectory with iterates lying in the span of the rows of $X$. 
One solution of (\ref{eq:ls}) is $X^T(XX^T)^{-1}y$,
which happens to be the solution with
minimum $\ell_2$-norm among all possible global minimizers. The minimum-norm
solution has the largest margin, and maximizing margin typically leads to lower generalization error \cite{boser1992training}.  
It is known that SGD converges to the minimum $\ell_2$-norm
solution of problem
(\ref{eq:ls}) \cite{zhang2017theory}, while adaptive methods (including Adagrad,
RMSprop, and Adam) converge to solutions with low $\ell_{\infty}$-norm
\cite{wilson2017marginal}. In particular, some examples show that solutions
obtained by adaptive methods can generalize arbitrarily poorly, while the SGD
solution makes no error. 

The following proposition studies the BAG solution. 
\begin{proposition} \label{prop:large_margin_solution}
Consider the underdetermined least squares problem in (\ref{eq:ls}).
If each submatrix $X_{:, \mG_b} \in \R^{n\times d_b}$ has full row rank, then
BAG
(with initial $\theta_1= 0$) converges to an optimal solution $\theta_*$
of (\ref{eq:ls})
in which each subvector $\theta_{*, \mG_b} = X_{:, \mG_b}^T(X_{:, \mG_b}X_{:, \mG_b}^T)^{-1}u_b$ for some $u_b \in \R^n$ and $\sum_{b\in[B]}u_b = y$.
\end{proposition}
Obviously, when $B=1$, BAG converges to the minimum $\ell_2$-norm solution of
(\ref{eq:ls}). By adapting the proof, it is easy to see that the same result 
%\footnote{the result can be obtained by a small adaptation of the proof.  as the change of proof is very small, so i mention that here. \#*** u do not hv a similar result in the next section} 
also holds for BAGM.

%%%%%%%%%%%%%%%%%%%%%%%%%%%%%%%%%%%%%%%%%%%%%%%%%%%%%%%%%%%%%%%%%%%%%%%%%%%%%%%%

\section{Synthetic Experiment on BAGM}

Figure~\ref{fig:synthetic_ncvx}
shows an example.
The objective
is based on the smoothed hinge loss (which satisfies Assumption~\ref{assumption:smooth}) 
\cite{rennie5loss}: 
\begin{equation*}
f_t(\theta) = \left\{ \begin{array}{ll}
\frac{1}{2} - y_t\langle \theta, x_t\rangle & y_t\langle \theta, x_t\rangle \leq 0, \\
\frac{1}{2}(1 - y_t\langle \theta, x_t\rangle)^2 & 0 < y_t\langle \theta, x_t\rangle < 1, \\
0 & y_t\langle \theta, x_t\rangle \geq 1.
\end{array} \right.
\end{equation*} 
The data generation and block construction are shown in Appendix~\ref{sec:bc_regret_comparison}. 
The initial $\theta_1$ is zero. We
set $a_t=1$, $\beta_t=0.9$, $\epsilon=10^{-8}$, $\eta_t = \eta/\sqrt{t}$ with $\eta=1$. 
The gradient $\nabla F(\theta_t)$ is estimated on 
$10,000$ 
randomly samples. 
Results are averaged over $10$ repetitions.
As can be seen, BAGM with $B=2$ and 4 have the fastest convergence. 
BAGM with 
$B=1$ 
and 
3 
have smaller $r_{\min}$ in corollary~\ref{corollary:nonconvex_momentum_comparison},
and thus are slower, but still faster than its counterpart with $B=d$ (which reduces to weighted AdaEMA).

\begin{figure}[ht]
\vskip -0.1in
\begin{center}
\includegraphics[width=0.5\columnwidth]{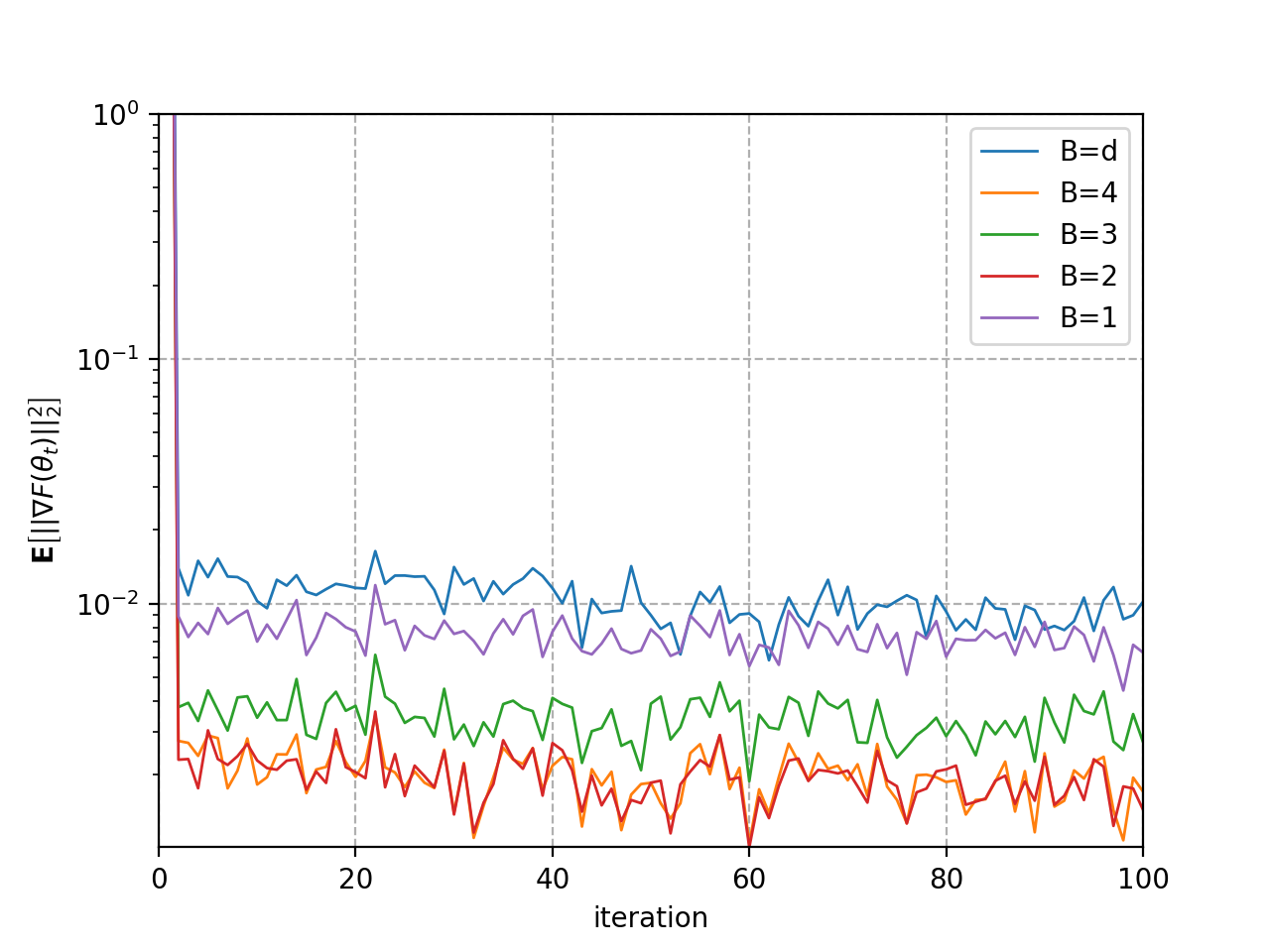}
\vskip -0.1in
\caption{Convergence of $\CE[\|\nabla F(\theta_t)\|^2_2]$ by BAGM on synthetic data.}
\label{fig:synthetic_ncvx}
\end{center}
\end{figure}

%%%%%%%%%%%%%%%%%%%%%%%%%%%%%%%%%%%%%%%%%%%%%%%%%%%%%%%%%%%%%%%%%%%%%%%%%%%%%%%%

\section{Experimental Setup}
\label{appendix_sec:gs}

%%%%%%%%%%%%%%%%%%%%%%%%%%%%%%%%%%%%%%

\subsection{Implementation}

As $\{a_t\}$ is non-decreasing, the accumulated sum $A_t$ can grow significantly, which may potentially cause some numerical issue. In practice, 
using steps 7 and 8 in Algorithms~\ref{alg:block-mom-adagrad}, we equivalently rewrite the update in (\ref{eq:a})
as the following exponentially moving update:
\begin{eqnarray*}
\hat{v}_{t, b} = \alpha_t \hat{v}_{t - 1, b} + (1 - \alpha_t)\frac{\|g_{t,\mG_b}\|_2^2}{d_b}, 
\end{eqnarray*}
where $\alpha_t = 1 - a_t/A_t$. 
If $a_t = \alpha^{-t}$, then $\alpha_t = \alpha(1 - \alpha^{t-1})/(1 - \alpha^t)$. Based on Corollary~\ref{corollary:nonconvex_momentum_seq}, this setting leads to an $\mathcal{O}(1)$ bound. On the other hand, if $a_t = t^{\tau}$, then we have $a_t/A_t = \mathcal{O}(1/t)$. This suggests that we can use polynomial-decay
averaging $\alpha_t = 1 - (c + 1)/(t + c)$ for some $c \geq 0$ \cite{shamir2013stochastic}, whereas $c > 0$ reduces the weight of earlier
iterates compared to later ones. The larger $c$ corresponds to the larger $\tau$. In this case, as $\sum_{t=1}^Ta_t = \mathcal{O}(T^\gamma)$ for some $\gamma > 0$, we have a convergence rate of $ \mathcal{O}(\log(T)/T)$. 

There are many possibilities of partitioning parameters in a deep network to blocks. In this paper, we propose the following. For a fully connected layer (i.e., $h_{l+1} = \phi(W_l^T h_{l} + b_l)$), 
we can assign an adaptive learning rate to either each column of $W_l$ (output
dimension) or each row of $W_l$ (input dimension) or the whole weight matrix
$W_l$.
Similarly, 
for the bias vector $b_l$,
either each of its element 
has its adaptive
learning rate or $b_l$ as a whole uses a single adaptive learning rate. For convolution
layers with weight tensor of shape $C_{out}\times C_{in} \times H \times W$, we
can use an adaptive learning rate for each kernel (leading to
$B_{conv} = C_{out}\times C_{in}$
blocks), each output channel ($B_{conv} =
C_{out}$), each input dimension ($B_{conv} = C_{in} \times H \times W$), or for the whole parameter tensor ($B_{conv} = 1$). 
For the bias vector, the construction is
similar to that for fully connected layers. 

%%%%%%%%%%%%%%%%%%%%%%%%%%%%%%%%%%%%%%

\subsection{ResNet on CIFAR-10}
\label{sec:exp_setup_cifar10}

\begin{figure}[t]
\begin{center}
\subfigure[{\em ResNet56: Training error}.]{\includegraphics[width=0.32\columnwidth]{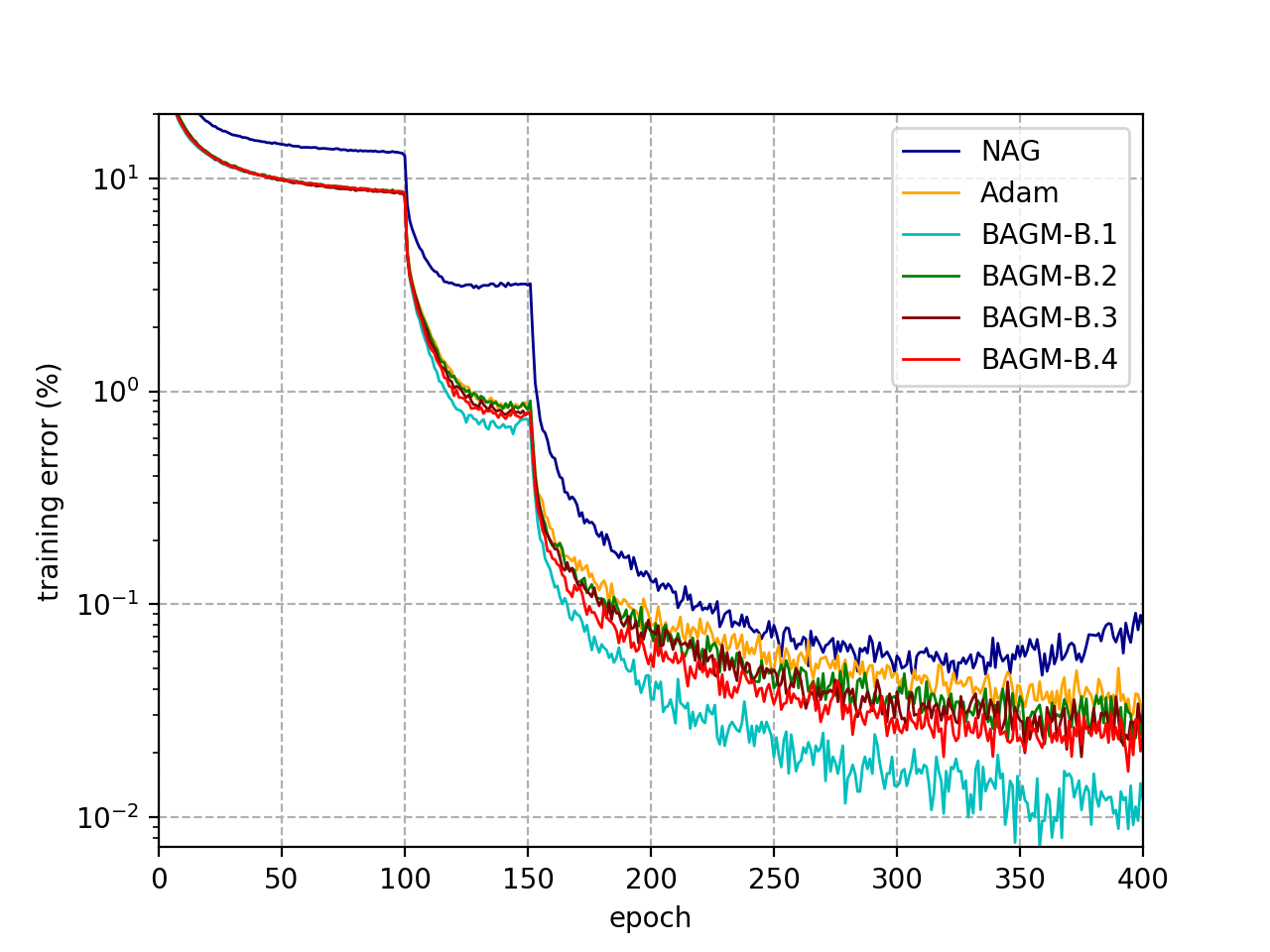}}
\subfigure[{\em ResNet56: Testing error}.]{\includegraphics[width=0.32\columnwidth]{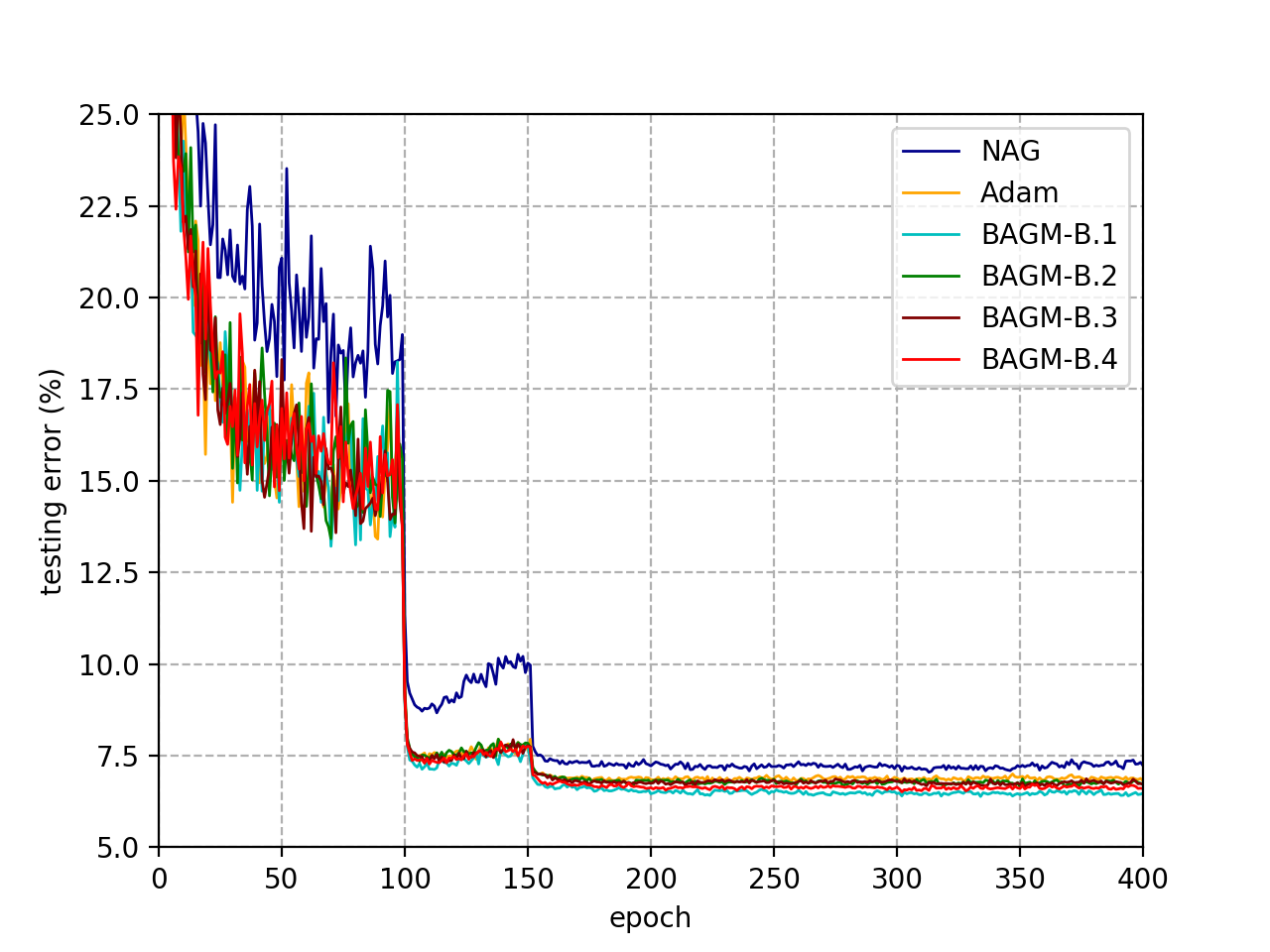}} 
\subfigure[{\em ResNet56: Generalization error}.]{\includegraphics[width=0.32\columnwidth]{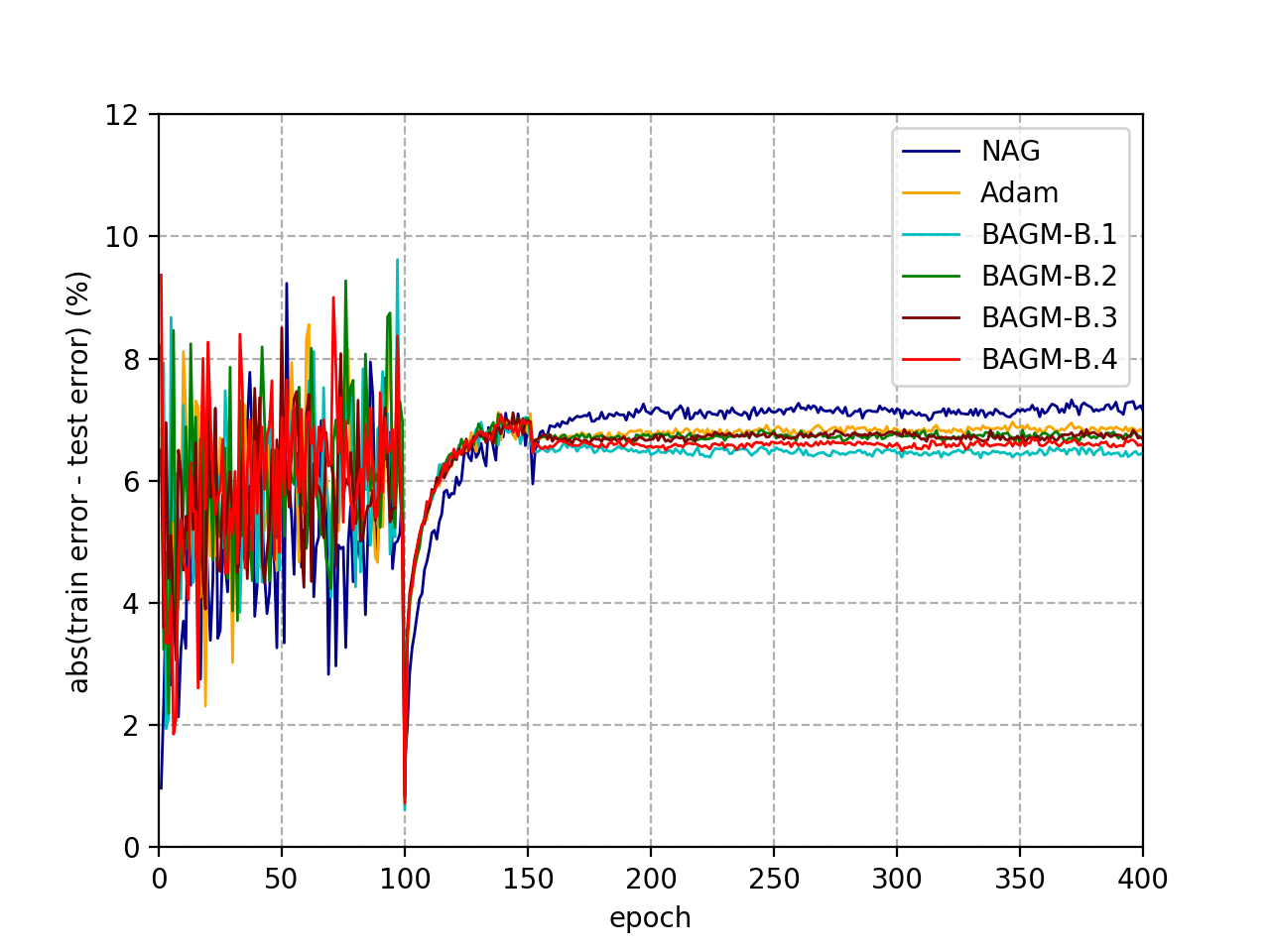}} \\
\vskip -0.1in
\subfigure[{\em ResNet110: Training error}.]{\includegraphics[width=0.32\columnwidth]{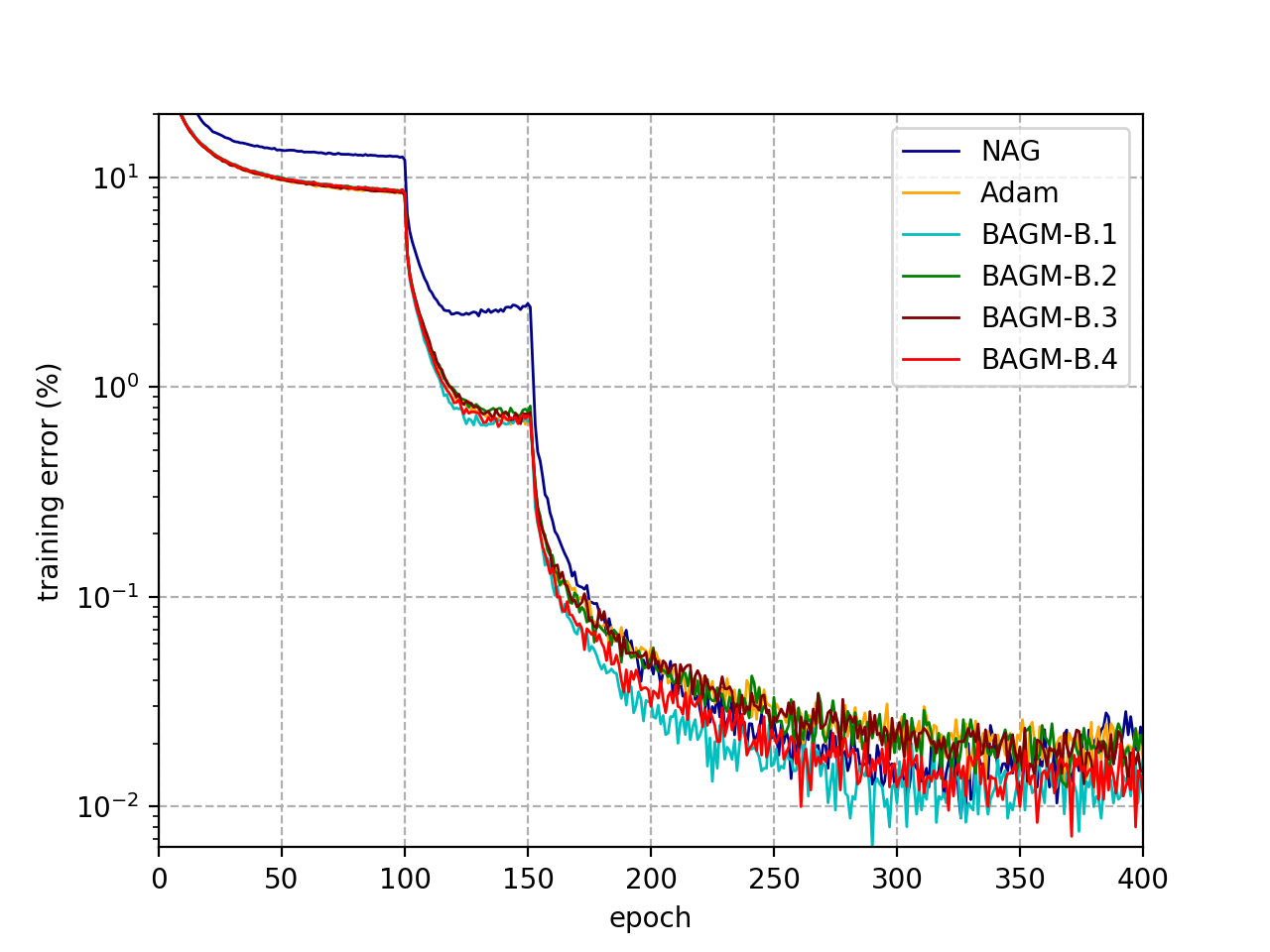}}
\subfigure[{\em ResNet110: Testing error}.]{\includegraphics[width=0.32\columnwidth]{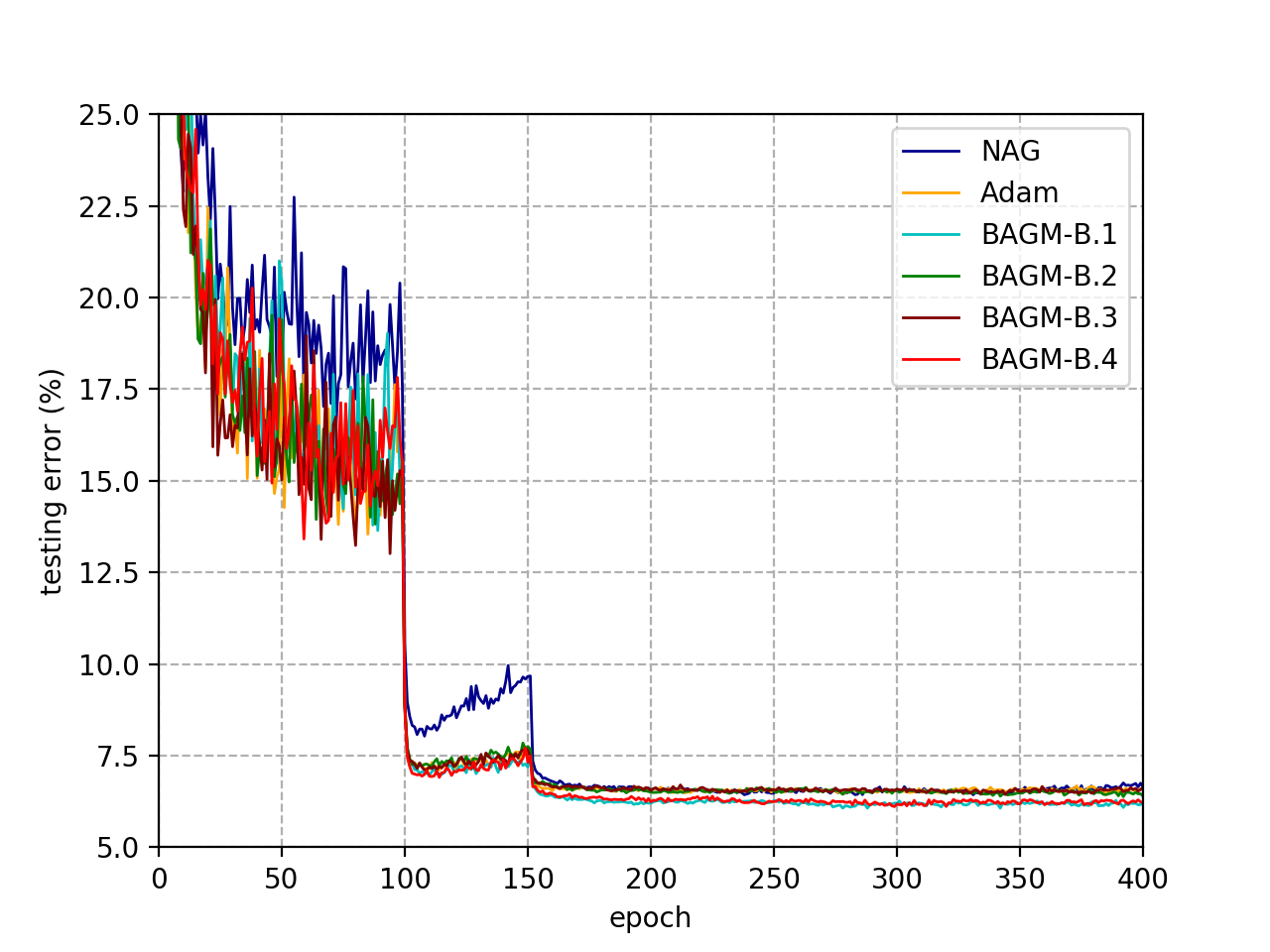}}
\subfigure[{\em ResNet110: Generalization error}.]{\includegraphics[width=0.32\columnwidth]{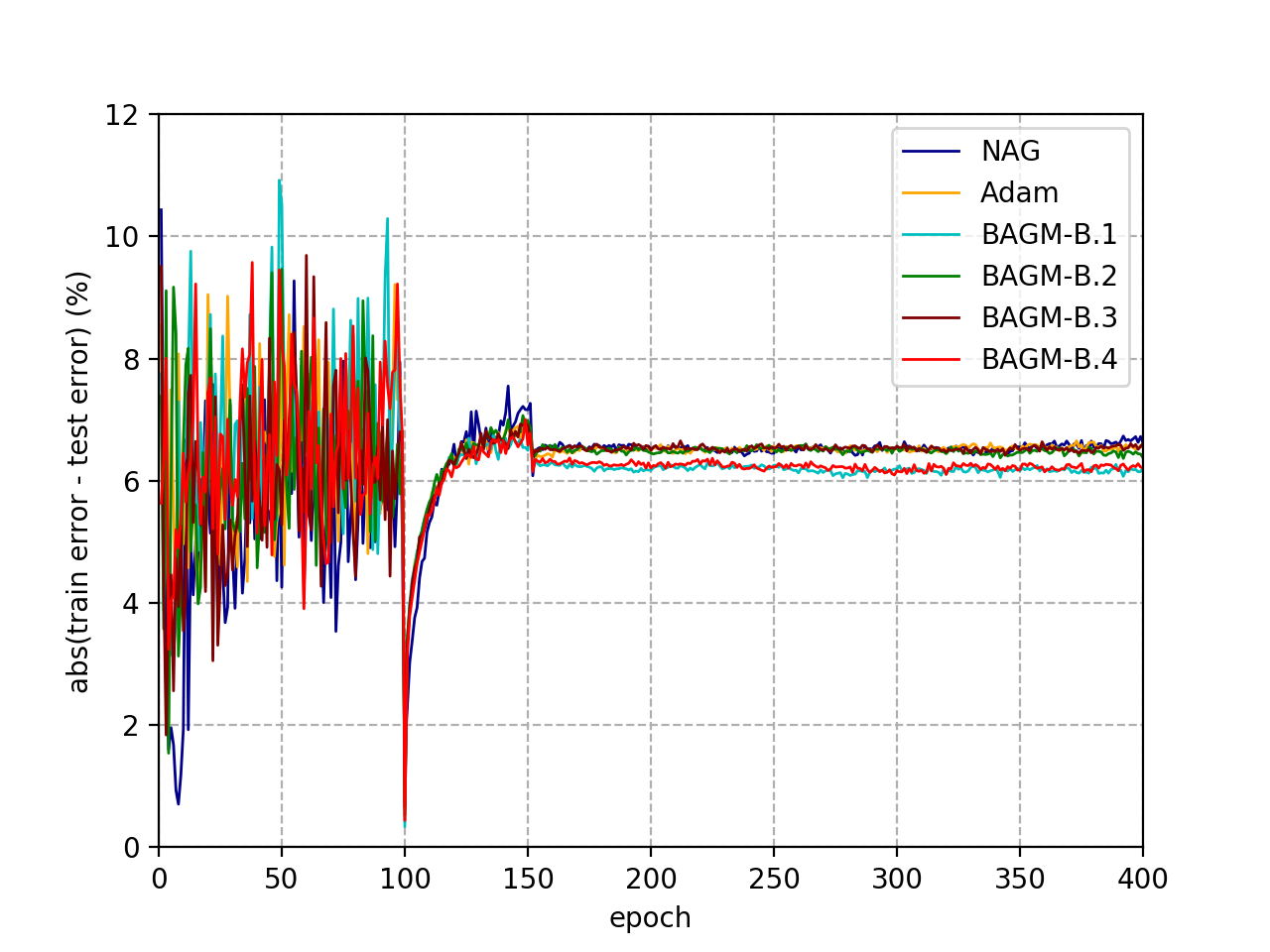}} \\
\caption{Results for all the compared methods on deep residual network. These curves are obtained by running the algorithms with the best hyper-parameters obtained by grid search. 
The training error (\%) is plotted on a logarithmic scale. To reduce statistical variance, results are averaged over $5$ repetitions.}
\label{fig:full_resnet}
\end{center}
\vskip -0.1in
\end{figure}

The CIFAR-10 data set has 50,000 training images and 10,000 testing images. 
As in \cite{he2016deep}, we employ data augmentation for training: 
1) pad the input picture by adding 4 pixels on each side of the image;
2) and then a 32x32 crop is randomly sampled from the padded
image with random horizontal flipping. In this experiment, 
a mini-batch size of $128$ is used. The stepsize is divided by $10$ at the 39k and 59k iterations.  
We use a weight decay of $0.0001$. 

For NAG, the initial learning rate $\eta$ is chosen from $\{0.01, 0.05, 0.1, 0.5, 1\}$, while for the adaptive methods, we have $\eta
\in \{0.0001, 0.0005, 0.001, 0.005, 0.01\}$. The
momentum parameter is searched over $\{0, 0.5, 0.9\}$. 
The learning rate is multiplied by $0.1$ at $100$ and $150$ epochs. 
We grid search the hyper-parameters by running each algorithm for $200$ epochs on ResNet56. 
The hyper-parameters that give the highest accuracy on the validation set are employed. 
The testing performance is obtained by running each algorithm with its best hyperparamters on full training set for $400$ epochs.
The same obtained hyperparameters are then used on training ResNet110. 
When NAG is applied to ResNet110, we use a smaller learning rate in the beginning
to warm up the training. Specifically, the obtained learning rate is divided by
$10$ in the first $4000$ iterations, and then go back to the original one and
continue training. The grid search results are shown in Table~\ref{tab:lr_cifar10}.

\begin{table}[ht]
\begin{center}
\begin{tabular}{c|c|c}
 \hline 
 & $\eta$ &  $\beta$   \\\hline 
NAG & 0.5 & 0.9 \\ 
Adam & 0.005 & 0 \\ \hline 
BAGM-\ref{bs:1} & 0.005 & 0  \\
BAGM-\ref{bs:2} & 0.005 & 0  \\
BAGM-\ref{bs:3} & 0.005 & 0  \\
BAGM-\ref{bs:4} & 0.005 & 0  \\
\hline
\end{tabular}
\end{center}
\caption{The best learning rate $\eta$ and momentum parameter $\beta$ obtained by grid search for each method. }
\label{tab:lr_cifar10}
\end{table}

%As can be seen, the best momentum parameters for Adam and BAGM are both $0$. This may be explained by the Theorem~\ref{theorem:nonconvex_momentum}, where we can see that the bound is smaller when $\beta$ is smaller. 
Figure~\ref{fig:full_resnet} shows that, on ResNet56, 
BAGM converges to a lower training error rate than Adam for all schemes used. 
For the deeper ResNet100 model, 
BAGM-\ref{bs:1} and \ref{bs:4} has faster convergence than Adam, while BAGM-\ref{bs:2} and \ref{bs:3} show the same convergence speed with Adam.  

\subsubsection{Verifying Corollary~\ref{corollary:nonconvex_momentum_comparison}}
\label{sec:verify_corollary}
 
In this experiment,
we use BAGM-\ref{bs:1}, as it shows fastest convergence. 
At the end of each epoch, we perform $10$ full data passes with random shuffle and
data augmentation mentioned in Appendix~\ref{sec:exp_setup_cifar10} 
to compute $\CE[g_{i}^2]$ and $\CE[\|g_{\tilde{\mG}_b}\|_2^2]/d_b$. 
Then, we approximate $\sigma_i^2$ and $\sigma_b^2$ by their empirical maxima over all epochs. 
Let $\bar{v}_{T,d} = \bar{v}_{T,B=d}$ and $\bar{v}_{T,\tilde{B}} = \bar{v}_{T,B=\tilde{B}}$. 
Empirically, we estimate $\bar{v}_{T,B}$ instead of $G_b$, as $C(T)$ is tighter than $\tilde{C}(T)$.   
We estimate $\bar{v}_{T, \tilde{B}}$ using $\max_{1 \leq t \leq T}\max_{b} \hat{v}_{t,b}$. 
We obtain $r_{\min} \approx 1.02$ and $\sqrt{(\bar{v}_{T,d} + \epsilon^2)/(\bar{v}_{T, \tilde{B}} + \epsilon^2)} \approx 3.70$ for ResNet56, 
and $r_{\min} \approx 1.01$ and $\sqrt{(\bar{v}_{T,d} + \epsilon^2)/(\bar{v}_{T, \tilde{B}} + \epsilon^2)} \approx 3.30$ for ResNet110.   
These statistics explain why the proposed blockwise adaptivity leads to faster convergence. 
Figure~\ref{fig:cv_cifar} shows
the coefficient of variation\footnote{The coefficient of variation is defined as the ratio of the standard deviation to the mean.} 
\cite{everitt2006cambridge} 
of $\{\sigma_i^2\}_{i \in \tilde{\mG}_b}$.
The results confirm our hypothesis that $\{\sigma_i^2\}_{i \in \tilde{\mG}_b}$ are under-dispersed. 

\begin{figure}[ht]
\begin{center}
\subfigure[ResNet56.]{\includegraphics[width=0.45\columnwidth]{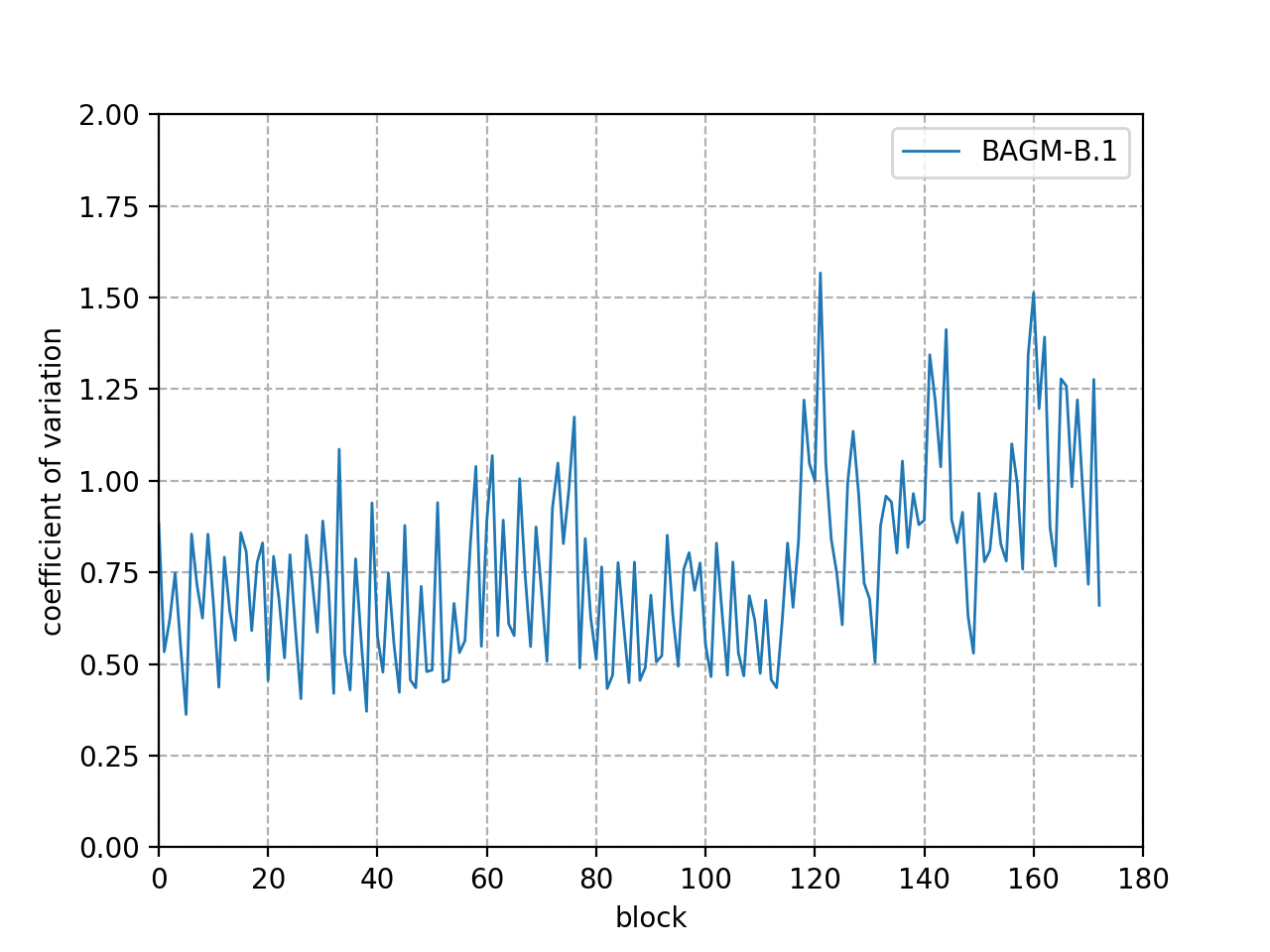}}
\subfigure[ResNet110.]{\includegraphics[width=0.45\columnwidth]{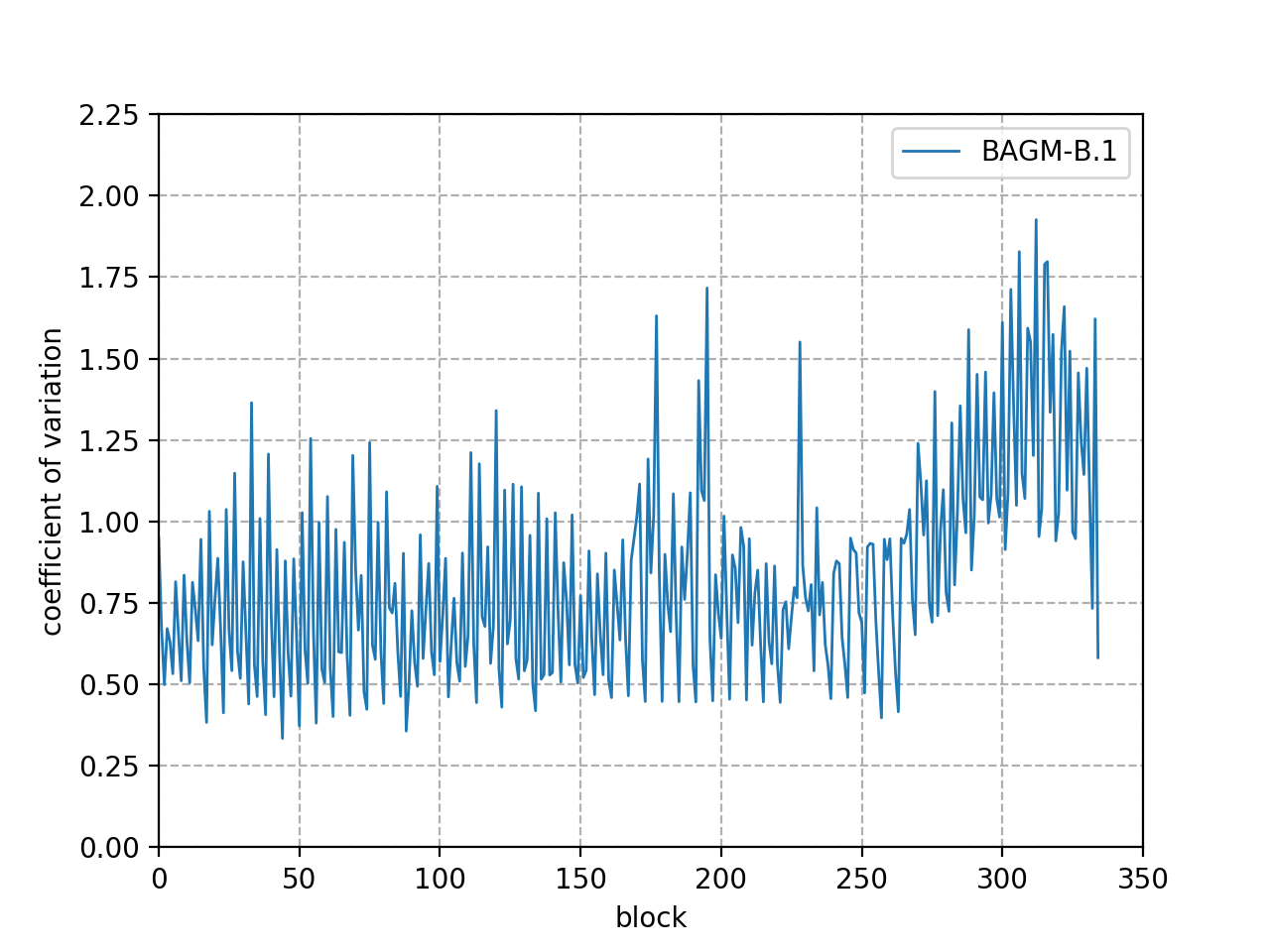}}
\vskip -0.1in
\caption{Coefficient of variation of $\{\sigma_i^2\}_{i \in \tilde{\mG}_b}$ for all the blocks with \ref{bs:1}. 
The blocks with higher indices in the abscissa belong to deeper layers. 
Notice that around $86\%$ 
(resp. $75\%$) of all blocks for ResNet56 (resp. ResNet110) have 
coefficient of variation 
smaller than $1$, indicating that $\{\sigma_i^2\}_{i \in \tilde{\mG}_b}$ have low variance and concentrate around the mean.}
\label{fig:cv_cifar}
\end{center}
\vskip -0.2in
\end{figure}

%%%%%%%%%%%%%%%%%%%%%%%%%%%%%%%%%%%%%%

\subsection{ImageNet Classification}
\label{app:imagenet}

In this experiment, we employ label smoothing and mixup \cite{zhang2018mixup}. The cosine schedule \cite{loshchilov2017sgdr} for learning rate is used. 
A warmup of $5$ epochs is applied. During validation, we use the center crop.  
The hyperparameter tunning is based on the obtained results in Section~\ref{sec:exp_setup_cifar10}. Specifically, 
for NAG, the initial learning rate is chosen from $\{0.4, 0.5\}$, and momentum parameter is fixed to 0.9. 
For Adam and BAGM, we have the initial learning rate $\eta \in \{0.004, 0.005\}$, and we use momentum parameter $\beta = 0$. 
A weight decay of 0.0001 is used (weight decay is not applied to bias vectors, and parameters for batch normalization layers) \footnote{The example script for running NAG with $\eta=0.4$ can be found in \url{https://raw.githubusercontent.com/dmlc/web-data/master/gluoncv/logs/classification/imagenet/resnet50_v1d-mixup.sh}. The details of the data augmentation can be found in \url{https://github.com/dmlc/gluon-cv/blob/master/scripts/classification/imagenet/train_imagenet.py}. }. 
The best learning rates for each method are presented in Table~\ref{tab:lr_imagenet}. 

\begin{table}[ht]
\begin{center}
\begin{tabular}{c|c}
 \hline 
 & $\eta$  \\\hline 
NAG & 0.4 \\ 
Adam & 0.004 \\ \hline 
BAGM-\ref{bs:1} & 0.004   \\
BAGM-\ref{bs:2} & 0.004  \\
BAGM-\ref{bs:3} & 0.004   \\
BAGM-\ref{bs:4} & 0.004  \\
\hline
\end{tabular}
\end{center}
\caption{The best learning rate $\eta$. }
\label{tab:lr_imagenet}
\end{table}

%%%%%%%%%%%%%%%%%%%%%%%%%%%%%%%%%%%%%%

\subsection{Word Language Modeling}
\label{app:lm}

In this experiment, we follow the same setting in \cite{merity2017regularizing}. 
A 3-layer AWD-LSTM is considered. The model is unrolled for 70 steps, and a mini-batch of size 80 is used. 
We clip the norm of the gradients at $0.25$. 
The details of the configuration used in this experiment can be found in \url{https://github.com/dmlc/gluon-nlp/blob/master/scripts/language_model/word_language_model.py}. 
For completeness, we show the model configuration in Table~\ref{tab:gs_lm_summary}. 

\begin{table}[ht]
\begin{center}
\begin{tabular}{c|c}
\hline
 & dimensionality/dropout rate \\
 \hline
Embedding size & 400 \\\hline
Hidden size & 1150 \\\hline
Dropout & 0.4 \\\hline
Dropout for RNN layers & 0.2 \\\hline
Dropout for input embedding layers & 0.65\\\hline
Dropout to remove words from embedding layer &  0.1\\\hline
Weight dropout &  0.5\\
\hline
\end{tabular}
\end{center}
\caption{Model configuration of AWD-LSTM model.}
\label{tab:gs_lm_summary}
\end{table}

As the WikiText-2 data set comes with a validation set, we perform the grid search by evaluating the performance on the validation set. 
For NAG, the initial stepsize is chosen from $\{1, 3, 10, 30\}$. For the adaptive methods, we select stepsize $\eta \in \{0.1, 0.03, 0.01, 0.003\}$. 
The momentum parameters varies in $\{0, 0.5, 0.9\}$. 
The learning rate is multiplied by $0.1$ when the validation performance does not improve for consecutive $30$ epochs. 
We tie the word embeddings and the softmax weights. 
For each algorithm, we employ the iterate averaging scheme proposed in \cite{merity2017regularizing}. 
The model is trained for $750$ epochs. 
The hyper-parameters obtained by the grid search is shown in Table~\ref{tab:lr_lm}. 
In general, \ref{bs:1} and \ref{bs:4} are not suitable for updating the word embedding matrix as word frequency varies a lot and thus the gradient is highly sparse. However, the gradient becomes dense when we use the weight tying. In modern toolkits such as Tensorflow, MXNet, and Pytorch, the weight matrices of the gates of the LSTM are concatenated to speed up the matrix-vector multiplication. We need to apply \ref{bs:1} and \ref{bs:4} to these weight matrices separately. 

\begin{table}[ht]
\begin{center}
\begin{tabular}{c|c|c}
 \hline 
 & $\eta$ & $\beta$ \\\hline 
NAG & 30 & $0$\\ 
Adam & 0.03 & $0.5$ \\ \hline 
BAGM-\ref{bs:1} & 0.03 & $0.9$  \\
BAGM-\ref{bs:2} & 0.03 & $0.5$  \\
BAGM-\ref{bs:4} & 0.03 & $0.5$ \\
\hline
\end{tabular}
\end{center}
\caption{The best learning rate $\eta$ and momentum parameter $\beta$ obtained by grid search for each method.}
\label{tab:lr_lm}
\end{table}

%\begin{table*}[ht]
%\begin{center}
%\begin{tabular}{c|c|c}
% \hline 
%  &  \multicolumn{2}{c}{test perplexity}   \\
%  & $\epsilon=10^{-3}$ & $\epsilon=10^{-6}$ \\\hline 
%NAG & \multicolumn{2}{c}{$65.75$}  \\ \hline 
%Adam & $65.40$ & ?   \\ \hline 
%BAGM-\ref{bs:1} & $65.42$ & $64.24$  \\
%BAGM-\ref{bs:2} & $\mathbf{65.29}$ & $64.10$   \\
%BAGM-\ref{bs:4} & $65.55$ & $64.25$  \\
%\hline
%\end{tabular}
%\end{center}
%\caption{Testing perplexities on WikiText-2 data set. Results are averaged over $3$
%repetitions. }
%%The best learning rate $\eta$ and momentum parameter $\beta$ obtained by grid search are included for each method. }
%\label{tab:val_lm}
%\end{table*}

%\begin{figure*}[ht]
%\begin{center}
%\subfigure[{\em training perplexity}.]{\includegraphics[width=0.65\columnwidth]{lm_train.png}}
%\subfigure[{\em validation perplexity}.]{\includegraphics[width=0.65\columnwidth]{lm_val.png}}
%\subfigure[{\em testing perplexity}.]{\includegraphics[width=0.65\columnwidth]{lm_test.png}}
%\caption{Results on word-level language model. These curves are obtained by running the algorithms with the best hyper-parameters obtained by grid search. We grid search the hyper-parameters by evaluating the perplexity on the validation set.}
%\label{fig:lm}
%\end{center}
%\end{figure*}

%%%%%%%%%%%%%%%%%%%%%%%%%%%%%%%%%%%%%%%%%%%%%%%%%%%%%%%%%%%%%%%%%%%%%%%%%%%%%%%%

\section{Proof of Proposition~\ref{theorem:nonlinear_large_margin}}
\begin{proof}
In this proof, we use denominator layout for matrix calculus. 
As all the activation functions are bijective and $\{W_k\}_{k=l+1}^L$ are invertible, $\Phi_l$ is bijective and has an inverse function $\Phi_l^{-1}$. 
Specifically, $\Phi_l^{-1}$ is given by
\begin{eqnarray*}
 \Phi_l^{-1}(Y) = \phi_{l}^{-1}(\cdots \phi_{L-2}^{-1}(\phi_{L-1}^{-1}(YW_L^{-1})W_{L-1}^{-1})\cdots W_{l+1}^{-1}).
\end{eqnarray*}
Then, with the assumption that $H_{l-1}$ has full row rank, the nonconvex objective (\ref{eq:nonlinear_sub_ls}) can be reformulated as the following convex problem:
\begin{eqnarray} \label{eq:nonlinear_sub_ls_rw}
\min_{W_l}\|H_{l-1}W_l - \Phi_l^{-1}(Y)\|_2^2.
\end{eqnarray}
It is obvious that its large margin solution is $H_{l-1}^T(H_{l-1}H_{l-1}^{T})^{-1}\Phi_l^{-1}(Y)$. 
In the sequel, we will see that every critical point of (\ref{eq:nonlinear_sub_ls}) is a global optimal solution.  
Let $h_{i, l - 1}$ denotes a column vector that is the $i$-th row of $H_{l-1}$ and $Z_{:, i}$ be the $i$-th column of matrix $Z$.  
The gradient of (\ref{eq:nonlinear_sub_ls}) is 
\begin{eqnarray*}
2H_{l-1}^T\sum_{k=1}^d\text{Diag}(\Phi_l(H_{l-1}W_l)_{:, k} - Y_{:, k})G_{k, l} = H_{l-1}^TE_l,
\end{eqnarray*}
where $G_{k, l} = [\nabla_{x = h_{1, l-1}^TW_{l}}\Phi_l(x)_k; \cdots; \nabla_{x = h_{n, l-1}^TW_{l}}\Phi_l(x)_k] \in \R^{n \times d}$ and $E_l = 2\sum_{k=1}^d\text{Diag}(\Phi_l(H_{l-1}W_l)_{:, k} - Y_{:, k})G_{k, l}$ to be the error matrix. As $H_{l-1}$ has full row rank, then clearly gradient is zero if only and if $E_l = 0$. 
By the definition of $G_{k, l}$, we can see that $E_l= 0$ if only and if $\Phi_l(H_{l-1}W_l) = Y$ when $\nabla_{x = h_{i, l-1}^TW_{l}}\Phi_l(x)$ has full row rank for all $i \in [n]$. Note that the gradient $\nabla_{x = h_{i, l-1}^TW_{l}}\Phi_l(x)$ is of the following form:
\begin{eqnarray*}
\nabla_{x = h_{i, l-1}^TW_{l}}\Phi_l(x) = (W_L\circ \phi_{L-1}'(h_{i, L-2}^TW_{L-1})^T1_d^T)^T\cdots(W_{l+1}\circ \phi_{l}'(h_{i, l-1}^TW_{l})^T1_d^T)^T,
\end{eqnarray*}
where $\circ$ is the Hadamard product. For all $k \in \{l, \dots, L-1\}$, as $W_{k+1}$ has full rank and $\phi'_k(z) \not= 0$ for any $z \in \R$, we have that 
$W_{k+1}\circ \phi_{k}'(h_{i, k-1}^TW_{k})^T1_d^T$ has full rank. Applying the fact that the multiplication of a number of invertible matrices preserves full rank,  
we obtain that $\nabla_{x = h_{i, l-1}^TW_{l}}\Phi_l(x)$ has full rank. 
Therefore, every critical point satisfies $\Phi_l(H_{l-1}W_l) = Y$ and every critical point is a global optimal solution.  

Let $i_t$ be the index chosen at iteration $t$ and $y_{i_t}$ be the $i_t$-th row of $Y$.  
Let us define $e_{t,l} = 2\sum_{k=1}^d(\Phi_l(h_{i_t, l-1}^TW_{t, l})_k - y_{i_t, k})\nabla_{x=h_{i_t, l-1}^TW_{t, l}}\Phi_l(x)_k$.
Now, we prove that if the following update rule applied on (\ref{eq:nonlinear_sub_ls}) finds a critical point, then the iterate converges to the largest margin solution.  
\begin{eqnarray} \label{eq:blockwise-stoc-method}
W_{t+1, l} = W_{t, l} - \eta_{t, l}h_{i_{t}, l-1}e_{t,l} = H_{l-1}^T\left(-\sum_{j=1}^t\eta_{j, l} \tilde{E}_{j, l}\right),
\end{eqnarray}
where we use $W_{l, 1} = 0$, $\eta_{t, l}$ is the stepsize for $l$-th layer at iteration $t$, and $\tilde{E}_{j, l}$ is a matrix in which its $i_k$-th row is $e_{j, l}$ and all the other rows are zeros. Then, the solution found by (\ref{eq:blockwise-stoc-method})  
lies in the span of rows of $H_{l-1}$. In other words, the solution has the following parametric form:
\begin{eqnarray*}
W_l = H_{l-1}^T\alpha_l
\end{eqnarray*}
for some $\alpha_l \in \R^{n}$. Thus, if (\ref{eq:blockwise-stoc-method}) is converging to a critical point in expectation, then we have  
$W_{t, l} \rightarrow W_{*, l}$ as $t \rightarrow \infty$, where $W_{*, l} = H_{l-1}^T\alpha_{*, l}$ for some optimal $\alpha_{*, l}$.  
Since every critical point is an optimal solution, then $W_{*, l}$ is also a solution to (\ref{eq:nonlinear_sub_ls_rw}), and we have 
\begin{eqnarray*}
\Phi_l^{-1}(Y) = H_{l-1}W_{*, l} = H_{l-1}H_{l-1}^T\alpha_{*, l}.
 \end{eqnarray*}
 We solve for $\alpha_{*, l}$ and obtain
 \begin{eqnarray*}
\alpha_{*, l} = (H_{l-1}H_{l-1}^T)^{-1}\Phi_l^{-1}(Y).
 \end{eqnarray*}
 Therefore, $W_{*, l} = H_{l-1}^T(H_{l-1}H_{l-1}^T)^{-1}\Phi_l^{-1}(Y)$. 
\end{proof}

%%%%%%%%%%%%%%%%%%%%%%%%%%%%%%%%%%%%%%%%%%%%%%%%%%%%%%%%%%%%%%%%%%%%%%%%%%%%%%%%

\section{Proof of Proposition~\ref{prop:large_margin_solution}}
\begin{proof}
Let $(x_{i_t}, y_{i_t})$ be the pair of sample selected at iteration $t$. The stochastic gradient of least square problem (\ref{eq:ls}) at the $t$-th iteration is 
\begin{eqnarray*}
2(x_{i_t}^T\theta_t - y_{i_t})x_{i_t} = X^Te_t,
\end{eqnarray*}
where we define $e_t$ to be the error vector with value $2(x_{i_t}^T\theta_t - y_{i_t})$ in the $i_t$-th coordinate and zeros elsewhere.  
For each block $b$, BAG with $\theta_1 = 0$ uses the following 
update rule:
\begin{eqnarray*}
\theta_{t + 1, \mG_b} = \theta_{t, \mG_b} - \eta_{t,b} X_{:, \mG_b}^Te_t = X_{:, \mG_b}^T\left(-\sum_{i=1}^t\eta_{i,b} e_{i}\right),
\end{eqnarray*}
where $\eta_{t,b} = \eta/(\sqrt{\sum_{i=1}^t\|X_{:, \mG_b}^Te_i\|_2^2/d_b} + \epsilon)$. Then, each subvector of the solution found by BAG 
lies in the span of rows of $X_{:, \mG_b}$. In other words, each subvector of the solution is of the following parametric form:
\begin{eqnarray*}
\theta_{\mG_b} = X_{:, \mG_b}^T\alpha_b
\end{eqnarray*}
for some $\alpha_b \in \R^{n}$. Combining with Corollary~\ref{corollary:convex_convergence}, BAG is converging in expectation 
$\frac{1}{t}\sum_{i=1}^t\theta_i \rightarrow \theta_*$ as $t \rightarrow \infty$, where $\theta_{*, \mG_b} = X_{:, \mG_b}^T\alpha_{*, b}$ for some optimal $\alpha_{*, b}$.  
Since $\theta_*$ is a solution to (\ref{eq:ls}), we have
\begin{eqnarray*}
y = X\theta_* = \sum_{b=1}^BX_{:, \mG_b}X_{:, \mG_b}^T\alpha_{*, b}.
 \end{eqnarray*}
Assume that each submatrix $X_{:, \mG_b}$ has full row rank, then $X_{:, \mG_b}X_{:, \mG_b}^T$ is invertible, we can solve for $\alpha_{*, b}$'s  and obtain
\begin{eqnarray*}
\alpha_{*, b} = (X_{:, \mG_b}X_{:, \mG_b}^T)^{-1}u_b
 \end{eqnarray*}
for some $u_b \in \R^n$ and $\sum_{b=1}^Bu_b = y$.
\end{proof}

%%%%%%%%%%%%%%%%%%%%%%%%%%%%%%%%%%%%%%%%%%%%%%%%%%%%%%%%%%%%%%%%%%%%%%%%%%%%%%%%

\section{Proof of Theorem~\ref{theorem:convex_regret}}

\begin{lemma} \label{lemma:partial_regret}
Let $\{\theta_t\}$ be the sequence generated by the Algorithm~\ref{alg:block-adagrad}. Define $s_t = [(\sqrt{v_{t,1}} + \epsilon)1_{d_1}^T, \dots, (\sqrt{v_{t,B}} + \epsilon)1_{d_B}^T]^T$. Let $H_t = \text{Diag}(s_t)$. Then, for any $\theta$, we have
\[f_t(\theta_t) - f_t(\theta) \leq \frac{1}{2\eta}\|\theta_{t} - \theta\|_{H_t}^2 - \frac{1}{2\eta}\|\theta_{t+1} - \theta\|_{H_t}^2  + \frac{\eta}{2}\|g_t\|_{H_t^{-1}}^2.\]
\end{lemma}
\begin{proof}
For any $\theta$, the convexity of $f_t$ indicates that
\begin{eqnarray*}
\lefteqn{f_t(\theta_t) - f_t(\theta)} \\
& \leq & \langle g_t, \theta_t - \theta \rangle \\
& = & \langle g_t, \theta_{t+1} - \theta \rangle + \langle g_t, \theta_{t} - \theta_{t+1} \rangle \\
& = & \frac{1}{\eta}\langle \theta_{t+1} - \theta, H_t(\theta_t - \theta_{t+1}) \rangle + \langle g_t, \theta_{t} - \theta_{t+1} \rangle \\
& =& \frac{1}{2\eta}\|\theta_{t} - \theta\|_{H_t}^2 - \frac{1}{2\eta}\|\theta_{t+1} - \theta\|_{H_t}^2 - \frac{1}{2\eta}\|\theta_{t+1} - \theta_t\|_{H_t}^2
+ \langle g_t, \theta_{t} - \theta_{t+1} \rangle \\
& \leq & \frac{1}{2\eta}\|\theta_{t} - \theta\|_{H_t}^2 - \frac{1}{2\eta}\|\theta_{t+1} - \theta\|_{H_t}^2 - \frac{1}{2\eta}\|\theta_{t+1} - \theta_t\|_{H_t}^2
+  \frac{1}{2\eta}\|\theta_{t+1} - \theta_t\|_{H_t}^2 + \frac{\eta}{2}\|g_t\|_{H_t^{-1}}^2 \\
& = & \frac{1}{2\eta}\|\theta_{t} - \theta\|_{H_t}^2 - \frac{1}{2\eta}\|\theta_{t+1} - \theta\|_{H_t}^2  + \frac{\eta}{2}\|g_t\|_{H_t^{-1}}^2,
\end{eqnarray*}
where the second to last inequality
follows from Fenchel's inequality applied to the conjugate functions $\frac{1}{2\eta}\|\cdot\|_{H_t}^2$ and $\frac{\eta}{2}\|\cdot\|_{H_t^{-1}}^2$.
\end{proof}

\begin{lemma} \label{lemma:sum_inequality}
Considering an arbitrary R-valued sequence $\{a_i\}$ and its vector representation $a_{1:t} = [a_1, \dots, a_t]$, we have
\begin{eqnarray*}
\sum_{t=1}^T\frac{a_t^2}{\|a_{1:t}\|_2} \leq 2\|a_{1:T}\|_2.
\end{eqnarray*}
\end{lemma}
\begin{proof}
The lemma can be proved by induction. The lemma trivially holds when $T = 1$. Assume the lemma holds for $T-1$, we get
\begin{eqnarray*}
\sum_{t=1}^T\frac{a_t^2}{\|a_{1:t}\|_2}  &\leq& 2\|a_{1:T-1}\|_2 + \frac{a_T^2}{\|a_{1:T}\|_2} \\
& = & 2\sqrt{Z - x} + \frac{x}{\sqrt{Z}},
\end{eqnarray*}
where we define $Z = \|a_{1:T}\|_2^2$ and $x = a_T^2$. As the RHS is non-increasing for $x \geq 0$. We can set $x=0$ to maximize the bound and obtain $2\sqrt{Z}$.
\end{proof}

\begin{lemma} \label{lemma:grad_sum_inequality}
Let $H_t$ be defined as in Lemma~\ref{lemma:partial_regret}. Denote $g_{1:t, \mG_b} = [g_{1, \mG_b}^T, \dots, g_{t, \mG_b}^T]^T$. We have
\[\sum_{t=1}^T\|g_t\|_{H_t^{-1}}^2 \leq 2\sum_{b=1}^B\sqrt{d_b}\|g_{1:T, \mG_b}\|_2.\]
\end{lemma}
\begin{proof}
\begin{eqnarray*}
\sum_{t=1}^T\|g_t\|_{H_t^{-1}}^2 
&\leq& \sum_{t=1}^T \langle g_t, \text{Diag}(s_t)^{-1}g_t\rangle \\
&=& \sum_{t=1}^T\sum_{b=1}^B \frac{\sqrt{d_b}\|g_{t, \mG_b}\|^2_2}{\|g_{1:t, \mG_b}\|_2} \\
&=& \sum_{b=1}^B\sqrt{d_b}\sum_{t=1}^T \frac{\|g_{t, \mG_b}\|^2_2}{\|g_{1:t, \mG_b}\|_2} \\
& \leq & 2\sum_{b=1}^B\sqrt{d_b}\|g_{1:T, \mG_b}\|_2.
\end{eqnarray*}
where the last inequality follows from the Lemma~\ref{lemma:sum_inequality} by setting $a_i = \|g_{i, \mG_b}\|_2^2$.
\end{proof}

%%%%%%%%%%%%%%%%%%%%%%%%%%%%%%%%%%%%%%

\subsection{Proof of Theorem~\ref{theorem:convex_regret}}
\begin{proof}
By summing up the equation in Lemma~\ref{lemma:partial_regret} with $\theta = \theta_*$, we obtain
\begin{eqnarray*}
\sum_{t=1}^Tf_t(\theta_t) - f_t(\theta_*) \leq \frac{1}{2\eta}\|\theta_{1} - \theta_*\|_{H_1}^2 + \frac{1}{2\eta}\sum_{t=1}^{T-1}\left[\|\theta_{t+1} - \theta_*\|_{H_{t+1}}^2 - \|\theta_{t+1} - \theta_*\|_{H_t}^2\right]  + \frac{\eta}{2}\sum_{t=1}^T\|g_t\|_{H_t^{-1}}^2.
\end{eqnarray*}
By the construction of $H_t$, we have that $H_{t+1} \succeq H_t$. Then, we get
\begin{eqnarray*}
\lefteqn{\|\theta_{t+1} - \theta_*\|_{H_{t+1}}^2 - \|\theta_{t+1} - \theta_*\|_{H_t}^2} \\
& = & \langle \theta_{t+1} - \theta_*, \text{Diag}(s_{t+1} - s_t)(\theta_{t+1} - \theta_*)\rangle \\
& = & \sum_{b=1}^B\|\theta_{t+1, \mG_b} - \theta_{*, \mG_b}\|_2^2(\sqrt{v_{t+1, b}} - \sqrt{v_{t, b}}).
\end{eqnarray*}
Given the above result, we have
\begin{eqnarray*}
\lefteqn{\sum_{t=1}^{T-1}\left[\|\theta_{t+1} - \theta_*\|_{H_{t+1}}^2 - \|\theta_{t+1} - \theta_*\|_{H_t}^2\right]}\\
& = & \sum_{b=1}^B\sum_{t=1}^{T-1}\|\theta_{t+1, \mG_b} - \theta_{*, \mG_b}\|_2^2(\sqrt{v_{t+1, b}} - \sqrt{v_{t, b}}) \\
& = & \sum_{b=1}^B\sum_{t=1}^{T-1}\|\theta_{t+1, \mG_b} - \theta_{*, \mG_b}\|_2^2(\sqrt{v_{t+1, b}} - \sqrt{v_{t, b}}) + \sum_{b=1}^B\|\theta_{1, \mG_b} - \theta_{*, \mG_b}\|_2^2(\sqrt{v_{1, b}} - \sqrt{v_{1, b}}) \\
& \leq & \sum_{b=1}^BD_b^2\sqrt{v_{T, b}}  - \sum_{b=1}^B\|\theta_{1, \mG_b} - \theta_{*, \mG_b}\|_2^2\sqrt{v_{1,b}}.
\end{eqnarray*}
Recall that $v_{T,b} = \|g_{1:T, \mG_b}\|_2^2/d_b$. Let $\epsilon=0$. 
Combining Lemma~\ref{lemma:grad_sum_inequality} with the fact that $\|\theta_{1} - \theta_*\|_{H_1}^2 = \sum_{b=1}^B\|\theta_{1,\mG_b} - \theta_{*, \mG_b}\|^2_2\sqrt{v_{1,b}}$, we have
\begin{eqnarray*}
\sum_{t=1}^Tf_t(\theta_t) - f_t(\theta_*) & \leq & \frac{1}{2\eta}\sum_{b=1}^B\frac{D_b^2}{\sqrt{d_b}}\|g_{1:T, \mG_b}\|_2 + \eta\sum_{b=1}^B\sqrt{d_b}\|g_{1:T, \mG_b}\|_2.
\end{eqnarray*}
\end{proof}

%%%%%%%%%%%%%%%%%%%%%%%%%%%%%%%%%%%%%%

\subsection{Proof of Corollary~\ref{corollary:convex_convergence}}
\begin{lemma} \label{lemma:Hoeffding-Azuma}
(Hoeffding-Azuma) Let $Z_1, Z_2, \dots, Z_T$ be a martingale difference sequence s.t. $|Z_i| \leq C$ (w.p. 1). 
For all $\epsilon \geq 0$, 
\begin{eqnarray*}
P\left(\sum_{t=1}^TZ_t \geq \epsilon\right) \leq \exp\left(-\frac{\epsilon^2}{2C^2T}\right).
\end{eqnarray*}
\end{lemma}
\begin{proof}
Assume that each $f_t$ is generated in an i.i.d. manner, then we have $F(\theta) = \CE_t[f_t(\theta)]$. 
Let $\theta_{\min} = \arg\min_{\theta}F(\theta)$. 
Let us define $Z_t = F(\theta_t) - f_t(\theta_t) - (F(\theta_{\min}) - f_t(\theta_{\min}))$ and $\mathcal{F}_{t-1} = \{f_1, \dots, f_{t-1}\}$. We get
\begin{eqnarray*}
\CE[Z_t|\mathcal{F}_{t-1}] & = & \CE[F(\theta_t) - f_t(\theta_t)|\mathcal{F}_{t-1}] - \CE[F(\theta_{\min}) - f_t(\theta_{\min})|\mathcal{F}_{t-1}] \\
& = & F(\theta_t) - F(\theta_t) - (F(\theta_{\min}) - F(\theta_{\min})) \\
& = & 0.
\end{eqnarray*}
Then, the process $\{Z_t\}$ is a martingale difference sequence w.r.t. the history $\mathcal{F}_{t-1}$.
\begin{eqnarray*}
\sum_{t=1}^T[F(\theta_t) - F(\theta_{\min})] & = & \sum_{t=1}^T[f_t(\theta_t) - f_t(\theta_{\min}) + Z_t] \\
& \leq & \sum_{t=1}^Tf_t(\theta_t) - \inf_{\theta}\sum_{t=1}^Tf_t(\theta) + \sum_{t=1}^TZ_t \\
& = & R(T) + \sum_{t=1}^TZ_t. 
\end{eqnarray*}
It is clearly that $|Z_i| \leq 2$. Applying Lemma~\ref{lemma:Hoeffding-Azuma}, with probability greater than $1 - \delta$, we have
\begin{eqnarray*}
\sum_{t=1}^TZ_t \leq \sqrt{8T\log(1/\delta)}.
\end{eqnarray*}
Then, with the convexity of $F$ and probability greater than $1 - \delta$, we have
\begin{eqnarray*}
F\left(\frac{1}{T}\sum_{t=1}^T\theta_t\right) - F(\theta_{\min}) \leq \frac{1}{T}\sum_{t=1}^TF(\theta_t) - F(\theta_{\min}) \leq \frac{R(T)}{T} + 2\sqrt{\frac{2\log(1/\delta)}{T}}.
\end{eqnarray*}
\end{proof}

 %%%%%%%%%%%%%%%%%%%%%%%%%%%%%%%%%%%%%%%%%%%%%%%%%%%%%%%%%%%%%%%%%%%%%%%%%%%%%%

%%%%%%%%%%%%%%%%%%%%%%%%%%%%%%%%%%%%%%%%%%%%%%%%%%%%%%%%%%%%%%%%%%%%%%%%%%%%%%%%

\section{Proof of Theorem~\ref{theorem:nonconvex_momentum}}  

In the sequel, we define $H_t$ as
\[H_t = \text{Diag}(s_t),\]
where
\[s_t = [(\sqrt{\hat{v}_{t,1}} + \epsilon)1_{d_1}^T, \dots, (\sqrt{\hat{v}_{t,B}} + \epsilon)1_{d_B}^T]^T.\]
Let $\delta_t = \theta_{t+1} - \theta_t = -\eta_tm_t/(\sqrt{s_t} + \epsilon)$ and $\sigma_{t, b} = \sqrt{\CE_t[|g_{t, \mG_b}\|^2_2}$. We introduce $\tilde{H}_t$ as 
\begin{eqnarray*}
\tilde{H}_t = \text{Diag}(\tilde{s}_t),
\end{eqnarray*}
where 
\begin{eqnarray*}
\tilde{s}_t &=& [(\sqrt{\tilde{v}_{t, 1}} + \epsilon)1_{d_1}^T, \dots, (\sqrt{\tilde{v}_{t, B}} + \epsilon)1_{d_B}^T]^T, \\
\tilde{v}_{t, b} &=& \frac{1}{A_t}\left(\sum_{i=1}^{t-1}a_i\frac{\|g_{i, \mG_b}\|^2_2}{d_b}+a_t\frac{\sigma_{t, b}^2}{d_b}\right) \forall b \in [B].
\end{eqnarray*}
Assume that $\sigma_{t,b}/\sqrt{d_b} \leq \sigma_b$ for all $t$ and let $\Sigma = \text{Diag}([\sigma_1^21_{d_1}^T, \dots, \sigma_B^21_{d_B}^T]^T)$.

\begin{lemma}\label{lemma:log-seq}
Let $S_t = S_0 + \sum_{i=1}^ta_i$, where $\{a_t\}$ is a non-negative sequence and $S_0 > 0$. We have 
\begin{eqnarray*}
\sum_{t=1}^T\frac{a_t}{S_t} \leq \log(S_T) - \log(S_0)
\end{eqnarray*}
\end{lemma}
\begin{proof}
The concavity of $\log$ leads to $\log(b) \leq \log(a) + \frac{1}{a}(b - a)$ for all $a, b > 0$. This suggests that
\begin{eqnarray*}
\frac{a - b}{a} \leq \log(a) - \log(b) = \log\left(\frac{a}{b}\right).
\end{eqnarray*}
Hence, we have
\begin{eqnarray*}
\sum_{t=1}^T\frac{a_t}{S_t} = \sum_{t=1}^T\frac{S_t - S_{t-1}}{S_t} \leq \sum_{t=1}^T\log\left(\frac{S_{t}}{S_{t-1}}\right) = \log(S_T) - \log(S_0).
\end{eqnarray*}
\end{proof}

\begin{lemma}\label{lemma:sum-seq-part}
Let $\{a_t\}$ and $\{s_t\}$ be two real number sequences, and let $S_t = \sum_{i=1}^ts_i$. Then, we have
\begin{eqnarray*}
\sum_{t=1}^Ta_ts_t = \sum_{t=1}^{T-1}(a_t - a_{t+1})S_t + a_TS_T.
\end{eqnarray*}
\end{lemma}
\begin{proof}
Let $S_0$ = 0. Expanding the summation, we obtain
\begin{eqnarray*}
\sum_{t=1}^Ta_ts_t & = & \sum_{t=1}^Ta_t(S_t - S_{t-1}) \\
& = & \sum_{t=1}^{T-1}a_tS_t - \sum_{t=1}^{T-1}a_{t+1}S_{t} + a_TS_T \\
& = & \sum_{t=1}^{T-1}(a_t - a_{t+1})S_t + a_tS_T
\end{eqnarray*}
\end{proof}

\begin{lemma}\label{lemma:nonconvex-block-stepsize-ema-stepsize}
Assume $\{a_t\}$ is non-decreasing such that $\{A_{t-1}/A_t\}$ is non-decreasing. Define $w_t = \eta_t/\sqrt{\frac{a_t}{A_t}}$. Assume $w_t$ is "almost" non-increasing. This means there exists another non-increasing sequence $\{z_t\}$ and positive constants $C_1$ and $C_2$ such that $C_1z_t \leq w_t \leq C_2z_t$. Then, 
\begin{eqnarray*}
w_t \leq C_2/C_1 w_i  \hspace{.05in} \text{ and } \hspace{.05in}\eta_t \leq C_2/C_1\eta_i 
\end{eqnarray*}
for all $i < t$.
\end{lemma}
\begin{proof}
For any $i < t$,
\begin{eqnarray*}
w_t \leq C_2z_t \leq C_2z_i \leq C_2/C_1w_i.
\end{eqnarray*}
Then,
\begin{eqnarray*}
\eta_t \leq \frac{C_2\sqrt{a_t/A_t}}{C_1\sqrt{a_i/A_i}}\eta_i  = \frac{C_2\sqrt{1 - A_{t-1}/A_t}}{C_1\sqrt{1 - A_{i-1}/A_i}}\eta_i  \leq C_2/C_1\eta_i.
\end{eqnarray*}
\end{proof}

 %%%%%%%%%%%%%%%%%%%%%%%%%%%%%%%%%%%%%%%%%%%%%%%%%%%%%%%%%%%%%%%%%%%%%%%%%%%%%%
\begin{lemma}\label{lemma:nonconvex-momentum-grad-weighted-bound-no-eta}
Assume that $\{a_t\}$ is non-decreasing.   
For any block diagonal matrix $C = \text{Diag}([c_11_{d_1}^T, \dots, c_B1_{d_B}^T]^T)$ with $c_b \geq 0$ for all $b$, we have 
    \begin{eqnarray*}
 \sum_{t=1}^T\CE\left[\frac{a_t}{A_t}\|H_{t}^{-1}g_t\|_{C}^2\right] 
  & \leq & \sum_{b=1}^Bc_{b}d_b\left[\log\left(\frac{\sigma_b^2}{\epsilon^2} + 1\right) 
  + \log\left(\frac{1}{a_1}\sum_{i=1}^{T}a_i + 1\right)\right].
 \end{eqnarray*}
\end{lemma}
\begin{proof}
 \begin{eqnarray*}
 \sum_{t=1}^T\frac{a_t}{A_t}\|H_{t}^{-1}g_t\|_{C}^2 & = & 
  \sum_{t=1}^T\sum_{b=1}^B\frac{c_{b}\frac{a_t}{A_t}\|g_{t, \mG_b}\|_2^2}{(\sqrt{\hat{v}_{t,b}} + \epsilon)^2} \\
  & \leq &  \sum_{t=1}^T\sum_{b=1}^B\frac{c_{b}\frac{a_t}{A_t}\|g_{t, \mG_b}\|_2^2}{\hat{v}_{t,b} + \epsilon^2} \\
  & = &  \sum_{b=1}^B\sum_{t=1}^T\frac{c_{b}a_t\|g_{t, \mG_b}\|_2^2}{\sum_{i=1}^{t}a_i\|g_{i,\mathcal{G}_b}\|_2^2/d_b + A_t\epsilon^2} \\
   & = &  \sum_{b=1}^Bc_{b}d_b\sum_{t=1}^T\frac{a_t\|g_{t, \mG_b}\|_2^2}{\sum_{i=1}^{t}a_i\|g_{i,\mathcal{G}_b}\|_2^2  + d_bA_t\epsilon^2} \\
   & \leq &  \sum_{b=1}^Bc_{b}d_b\sum_{t=1}^T\frac{a_t\|g_{t, \mG_b}\|_2^2}{\sum_{i=1}^{t}a_i\|g_{i,\mathcal{G}_b}\|_2^2 + d_ba_1\epsilon^2}.
 \end{eqnarray*}
Hence,
\begin{eqnarray*}
  \sum_{t=1}^T\frac{a_t}{A_t}\|H_{t}^{-1}g_t\|_{C}^2
   & \leq &  \sum_{b=1}^Bc_{b}d_b\left[\log\left(\sum_{i=1}^{T}a_i\|g_{i,\mathcal{G}_b}\|_2^2 + d_ba_1\epsilon^2\right) - \log(d_ba_1\epsilon^2)\right],
 \end{eqnarray*}
 where the inequality follows from Lemma~\ref{lemma:log-seq}. Using Jensen's inequality, we get
   \begin{eqnarray*}
 \sum_{t=1}^T\CE\left[\frac{a_t}{A_t}\|H_{t}^{-1}g_t\|_{C}^2\right] 
 & \leq & \sum_{b=1}^Bc_{b}d_b\CE\left[\log\left(\sum_{i=1}^{T}a_i\|g_{i,\mathcal{G}_b}\|_2^2 + d_ba_1\epsilon^2\right) - \log(d_ba_1\epsilon^2)\right] \\
  & \leq & \sum_{b=1}^Bc_{b}d_b\left[\log\left(\sum_{i=1}^{T}a_i\CE[\|g_{i,\mathcal{G}_b}\|_2^2] + d_ba_1\epsilon^2\right) - \log(d_ba_1\epsilon^2)\right] \\
  & \leq & \sum_{b=1}^Bc_{b}d_b\left[\log\left(d_b\sigma_b^2\sum_{i=1}^{T}a_i + d_ba_1\epsilon^2\right) - \log(d_ba_1\epsilon^2)\right] \\
  & = & \sum_{b=1}^Bc_{b}d_b\log\left(\frac{\sigma_b^2}{a_1\epsilon^2}\sum_{i=1}^{T}a_i + 1\right).
 \end{eqnarray*}
 Using the inequality $\log(1 + ab) \leq \log(1 + a + b + ab) = \log(1 + a) + \log(1 + b)$ for $a, b \geq 0$, we have
    \begin{eqnarray*}
 \sum_{t=1}^T\CE\left[\frac{a_t}{A_t}\|H_{t}^{-1}g_t\|_{C}^2\right] 
  & \leq & \sum_{b=1}^Bc_{b}d_b\left[\log\left(\frac{\sigma_b^2}{\epsilon^2} + 1\right) 
  + \log\left(\frac{1}{a_1}\sum_{i=1}^{T}a_i + 1\right)\right].
 \end{eqnarray*}
\end{proof}

\begin{lemma}\label{lemma:nonconvex-momentum-grad-weighted-bound}
Assume that $\{a_t\}$ is non-decreasing. Define $w_t = \eta_t/\sqrt{\frac{a_t}{A_t}}$. 
Assume $w_t$ is "almost" non-increasing. This means there exists another non-increasing sequence $\{z_t\}$ and positive constants $C_1$ and $C_2$ such that $C_1z_t \leq w_t \leq C_2z_t$ for all $t$.  
For any block diagonal matrix $C = \text{Diag}([c_11_{d_1}^T, \dots, c_B1_{d_B}^T]^T)$ with $c_b \geq 0$ for all $b$, we have 
  \begin{eqnarray*}
 \sum_{t=1}^T\eta_t\CE\left[\sqrt{\frac{a_t}{A_t}}\|H_{t}^{-1}g_t\|_{C}^2\right]
    & \leq & \frac{C_2}{C_1}\left[w_1 \sum_{b=1}^Bc_{b}d_b\log\left(\frac{\sigma_b^2}{\epsilon^2} + 1\right) 
  + \sum_{b=1}^Bc_{b}d_b\sum_{t=1}^{T}\eta_t\sqrt{\frac{a_t}{A_t}}\frac{A_t}{A_{t-1} + a_1}\right].
 \end{eqnarray*}
 \end{lemma}
 \begin{proof}
 Let $\xi_t = \frac{a_t}{A_t}\|H_{t}^{-1}g_t\|_{C}^2$, then $\zeta_t = \sum_{i=1}^t\xi_i$. Lemma~\ref{lemma:sum-seq-part} indicates that we have
  \begin{eqnarray*}
 \sum_{t=1}^T\eta_t\sqrt{\frac{a_t}{A_t}}\|H_{t}^{-1}g_t\|_{C}^2 & = & 
  \sum_{t=1}^Tw_t\xi_t \\
  & \leq & C_2\sum_{t=1}^Tz_t\xi_t \\
  & = & \sum_{t=1}^{T-1}(z_t - z_{t+1})\zeta_t + z_T\zeta_T.
 \end{eqnarray*}
 Define $M_t = \sum_{b=1}^Bc_{b}d_b\left[\log\left(\frac{\sigma_b^2}{\epsilon^2} + 1\right) 
  + \log\left(\frac{1}{a_1}\sum_{i=1}^{t}a_i + 1\right)\right]$. 
 By Lemma~\ref{lemma:nonconvex-momentum-grad-weighted-bound-no-eta}, we have $\CE[\zeta_t] \leq M_t$.  
 Then, 
   \begin{eqnarray*}
 \sum_{t=1}^T\eta_t\CE\left[\sqrt{\frac{a_t}{A_t}}\|H_{t}^{-1}g_t\|_{C}^2\right]
  & \leq & C_2\left[\sum_{t=1}^{T-1}(z_t - z_{t+1})\CE[\zeta_t] + z_T\CE[\zeta_T]\right] \\
  & \leq & C_2\left[\sum_{t=1}^{T-1}(z_t - z_{t+1})M_t + z_TM_T\right],
 \end{eqnarray*}
 where the last inequality follows from the assumption that $z_t \geq z_{t+1}$. Then,
 \begin{eqnarray*}
 \sum_{t=1}^T\eta_t\CE\left[\sqrt{\frac{a_t}{A_t}}\|H_{t}^{-1}g_t\|_{C}^2\right]
  & \leq & C_2\left[\sum_{t=1}^{T-1}(z_t - z_{t+1})M_t + z_TM_T\right] \\
  & = & C_2\left[\sum_{t=1}^{T}z_t(M_t - M_{t-1}) + z_1M_0\right] \\
  & = & C_2\left[z_1\sum_{b=1}^Bc_{b}d_b\log\left(\frac{\sigma_b^2}{\epsilon^2} + 1\right) 
  + \sum_{t=1}^{T}z_t\sum_{b=1}^Bc_{b}d_b\log\left(\frac{A_t + a_1}{A_{t-1} + a_1}\right)\right] \\
  & \leq & \frac{C_2}{C_1}\left[w_1\sum_{b=1}^Bc_{b}d_b\log\left(\frac{\sigma_b^2}{\epsilon^2} + 1\right) 
  + \sum_{t=1}^{T}w_t\sum_{b=1}^Bc_{b}d_b\log\left(\frac{A_t + a_1}{A_{t-1} + a_1}\right)\right].
 \end{eqnarray*}
 As $\log(1 + x) \leq x$ for $x > -1$ and the fact that $A_t \geq A_{t-1}$, we get
  \begin{eqnarray*}
 \log\left(\frac{A_t + a_1}{A_{t-1} + a_1}\right) = \log\left(1 + \frac{A_t + a_1}{A_{t-1} + a_1} - 1\right) \leq \frac{A_t + a_1}{A_{t-1} + a_1} - 1 = \frac{a_t}{A_{t-1} + a_1}.  
 \end{eqnarray*}
 Hence, 
  \begin{eqnarray*}
 \sum_{t=1}^T\eta_t\CE\left[\sqrt{\frac{a_t}{A_t}}\|H_{t}^{-1}g_t\|_{C}^2\right]
  & \leq &  \frac{C_2}{C_1}\left[w_1\sum_{b=1}^Bc_{b}d_b\log\left(\frac{\sigma_b^2}{\epsilon^2} + 1\right) 
  + \sum_{t=1}^{T}w_t\sum_{b=1}^Bc_{b}d_b\log\left(\frac{A_t + a_1}{A_{t-1} + a_1}\right)\right] \\
    & \leq & \frac{C_2}{C_1}\left[w_1 \sum_{b=1}^Bc_{b}d_b\log\left(\frac{\sigma_b^2}{\epsilon^2} + 1\right) 
  + \sum_{b=1}^Bc_{b}d_b\sum_{t=1}^{T}\eta_t\sqrt{\frac{a_t}{A_t}}\frac{A_t}{A_{t-1} + a_1}\right].
 \end{eqnarray*}
 \end{proof}

\begin{lemma}\label{lemma:nonconvex-block-stepsize-ema-iterate-diff}
Let $\tilde{\eta}_{t, b} = \frac{\eta_t}{\sqrt{\tilde{v}_{t,b}} + \epsilon}$. For each block $b$ and $t \geq 2$, we have
\begin{eqnarray*}
\left(\delta_t - \frac{\beta_t\eta_t}{\sqrt{1 - a_t/A_t}\eta_{t-1}}\delta_{t-1}\right)_{\mathcal{G}_b} = 
- (1 - \beta_t)\tilde{\eta}_{t, b}g_{t,\mathcal{G}_b} + \tilde{\eta}_{t, b}\frac{\frac{a_t}{A_td_b}\|g_{t,\mathcal{G}_b}\|_2^2}{\sqrt{\hat{v}_{t,b}} + \epsilon}X_{t,b}
+\tilde{\eta}_{t, b}\frac{\sigma_{t,b}}{\sqrt{d_b}}Y_{t,b} + Z_{t,b},
\end{eqnarray*}
where
\begin{eqnarray*}
X_{t,b} &=& \frac{\beta_tm_{t-1, \mathcal{G}_b}}{\sqrt{\hat{v}_{t, b}} + \sqrt{A_{t-1}\hat{v}_{t-1, b}/A_t}} + \frac{(1 - \beta_t)g_{t,\mathcal{G}_b}}{\sqrt{\tilde{v}_{t,b}} + \sqrt{\hat{v}_{t,b}}}, \\
Y_{t,b} &=& 
\frac{\frac{a_t}{A_t\sqrt{d_b}}\|g_{t,\mathcal{G}_b}\|_2}{\sqrt{\hat{v}_{t,b}} + \epsilon}\frac{\beta_tm_{t-1, \mathcal{G}_b}}{\sqrt{A_{t-1}\hat{v}_{t-1,b}/A_t} + \epsilon}\frac{\sqrt{\frac{a_t}{A_td_b}}\|g_{t,\mathcal{G}_b}\|_2}{\sqrt{\hat{v}_{t, b}} + \sqrt{A_{t-1}\hat{v}_{t-1, b}/A_t}}\frac{\sqrt{\frac{a_t}{A_td_b}}\sigma_{t,b}}{\sqrt{\tilde{v}_{t, b}} + \sqrt{A_{t-1}\hat{v}_{t-1, b}/A_t}} \\
&& - \frac{\frac{a_t}{A_t}g_{t,\mathcal{G}_b}}{\sqrt{\hat{v}_{t,b}} + \epsilon}\frac{(1 - \beta_t)\frac{\sigma_{t,b}}{\sqrt{d_b}}}{\sqrt{\tilde{v}_{t,b}} + \sqrt{\hat{v}_{t,b}}}, \\
Z_{t,b} & = & \beta_t\eta_tm_{t-1, \mathcal{G}_b}\frac{\left(1 - \sqrt{A_{t-1}/A_t}\right)\epsilon}{(\sqrt{A_{t-1}\hat{v}_{t-1, b}/A_t} + \epsilon)(\sqrt{A_{t-1}\hat{v}_{t-1, b}/A_t} + \sqrt{A_{t-1}/A_t}\epsilon)} .
\end{eqnarray*}
\end{lemma}
\begin{proof}
Let $\tilde{s}_t = [(\sqrt{\tilde{v}_{t,1}} + \epsilon)1_{d_1}^T, \dots, (\sqrt{\tilde{v}_{t,B}} + \epsilon)1_{d_B}^T]^T$. For any $t \geq 2$, 
\begin{eqnarray}
\lefteqn{\delta_t - \frac{\beta_t\eta_t}{\sqrt{A_{t-1}/A_t}\eta_{t-1}}\delta_{t-1}} \nonumber\\
 & = & -\frac{\eta_tm_t}{\sqrt{s_t} + \epsilon} + \frac{\beta_t\eta_tm_{t-1}}{\sqrt{A_{t-1}s_{t-1}/A_t} + \sqrt{A_{t-1}/A_t}\epsilon} \nonumber\\
& =& -\eta_t\left[\frac{m_t}{\sqrt{s_t} + \epsilon} - \frac{\beta_tm_{t-1}}{\sqrt{A_{t-1}s_{t-1}/A_t} + \sqrt{A_{t-1}/A_t}\epsilon}\right] \nonumber\\
& = & -\frac{(1 - \beta_t)\eta_tg_t}{\sqrt{s_t} + \epsilon} - \beta_t\eta_tm_{t-1}\left[\frac{1}{\sqrt{s_t} + \epsilon} - \frac{1}{\sqrt{A_{t-1}s_{t-1}/A_t} + \epsilon}\right] \nonumber \\
&& - \beta_t\eta_tm_{t-1}\left[\frac{1}{\sqrt{A_{t-1}s_{t-1}/A_t} + \epsilon} - \frac{1}{\sqrt{A_{t-1}s_{t-1}/A_t} + \sqrt{A_{t-1}/A_t}\epsilon}\right] \nonumber\\
& = & -\frac{(1 - \beta_t)\eta_tg_t}{\sqrt{s_t} + \epsilon} - \beta_t\eta_tm_{t-1}\left[\frac{1}{\sqrt{s_t} + \epsilon} - \frac{1}{\sqrt{A_{t-1}s_{t-1}/A_t} + \epsilon}\right]   \label{eq:nonconvex-bema-1} \\
&& + \beta_t\eta_tm_{t-1}\frac{\left(1 - \sqrt{A_{t-1}/A_t}\right)\epsilon}{(\sqrt{A_{t-1}s_{t-1}/A_t} + \epsilon)(\sqrt{A_{t-1}s_{t-1}/A_t} + \sqrt{A_{t-1}/A_t}\epsilon)} \nonumber
\end{eqnarray}
Let expand the first term of (\ref{eq:nonconvex-bema-1}) as
\begin{eqnarray*}
\frac{(1 - \beta_t)\eta_tg_t}{\sqrt{s_t} + \epsilon} & =& \frac{(1 - \beta_t)\eta_tg_t}{\sqrt{\tilde{s}_t} + \epsilon} 
+ (1 - \beta_t)\eta_tg_t\left[\frac{1}{\sqrt{s_t} + \epsilon} - \frac{1}{\sqrt{\tilde{s}_t} + \epsilon}\right] \\
& =& \frac{(1 - \beta_t)\eta_tg_t}{\sqrt{\tilde{s}_t} + \epsilon} 
+ (1 - \beta_t)\eta_tg_t\frac{\tilde{s}_t - s_t}{(\sqrt{s_t} + \epsilon)(\sqrt{\tilde{s}_t} + \epsilon)(\sqrt{\tilde{s}_t} + \sqrt{s_t} )}.
\end{eqnarray*}
For each block $b$, we have
\begin{eqnarray}
\lefteqn{\frac{(1 - \beta_t)\eta_tg_{t, \mathcal{G}_b}}{\sqrt{\hat{v}_{t,b}} + \epsilon}} \nonumber\\
& =& \frac{(1 - \beta_t)\eta_tg_{t, \mathcal{G}_b}}{\sqrt{\tilde{v}_{t,b}} + \epsilon} 
+ (1 - \beta_t)\eta_tg_{t,\mathcal{G}_b}\frac{\frac{a_t}{A_td_b}(\sigma_{t,b}^2 - \|g_{t,\mathcal{G}_b}\|^2_2)}{(\sqrt{\hat{v}_{t,b}} + \epsilon)(\sqrt{\tilde{v}_{t,b}} + \epsilon)(\sqrt{\tilde{v}_{t,b}} + \sqrt{\hat{v}_{t,b}} )} \nonumber\\
& =& (1 - \beta_t)\tilde{\eta}_{t, b}g_{t,\mathcal{G}_b}
+ \tilde{\eta}_{t, b}\frac{\sigma_{t,b}}{\sqrt{d_b}}\frac{\frac{a_t}{A_t}g_{t,\mathcal{G}_b}}{\sqrt{\hat{v}_{t,b}} + \epsilon}\frac{(1 - \beta_t)\frac{\sigma_{t,b}}{\sqrt{d_b}}}{\sqrt{\tilde{v}_{t,b}} + \sqrt{\hat{v}_{t,b}}} - \tilde{\eta}_{t, b}\frac{\frac{a_t}{A_td_b}\|g_{t,\mathcal{G}_b}\|_2^2}{\sqrt{\hat{v}_{t,b}} + \epsilon}\frac{(1 - \beta_t)g_{t,\mathcal{G}_b}}{\sqrt{\tilde{v}_{t,b}} + \sqrt{\hat{v}_{t,b}}}. \label{eq:nonconvex-bema-2}
\end{eqnarray}
Then, we expand the second term of (\ref{eq:nonconvex-bema-1}):
\begin{eqnarray*}
 \lefteqn{\beta_t\eta_tm_{t-1}\left[\frac{1}{\sqrt{s_t} + \epsilon} - \frac{1}{\sqrt{A_{t-1}s_{t-1}/A_t} + \epsilon}\right]}\\
 & = & \beta_t\eta_tm_{t-1}\frac{\sqrt{A_{t-1}s_{t-1}/A_t} - \sqrt{s_t}}{(\sqrt{s_t} + \epsilon)(\sqrt{A_{t-1}s_{t-1}/A_t} + \epsilon)} \\
  & = & \beta_t\eta_tm_{t-1}\frac{A_{t-1}s_{t-1}/A_t - s_t}{(\sqrt{s_t} + \epsilon)(\sqrt{A_{t-1}s_t/A_t} + \epsilon)(\sqrt{s_{t}} + \sqrt{A_{t-1}s_{t-1}/A_t})}.
\end{eqnarray*}
Similarly, for each block $b$, we have
\begin{eqnarray}
 \lefteqn{\beta_t\eta_tm_{t-1, \mathcal{G}_b}\left[\frac{1}{\sqrt{\hat{v}_{t,b}} + \epsilon} - \frac{1}{\sqrt{A_{t-1}\hat{v}_{t-1, b}/A_t} + \epsilon}\right]} \nonumber\\
  & = & -\beta_t\eta_tm_{t-1, \mathcal{G}_b}\frac{a_t\|g_{t,\mathcal{G}_b}\|_2^2/(A_td_b)}{(\sqrt{\hat{v}_{t,b}} + \epsilon)(\sqrt{A_{t-1}\hat{v}_{t-1,b}/A_t} + \epsilon)(\sqrt{\hat{v}_{t, b}} + \sqrt{A_{t-1}\hat{v}_{t-1, b}/A_t})}  \nonumber\\
    & = & -\beta_t\eta_tm_{t-1, \mathcal{G}_b}\left[\frac{a_t\|g_{t,\mathcal{G}_b}\|_2^2/(A_td_b)}{(\sqrt{\hat{v}_{t,b}} + \epsilon)(\sqrt{\tilde{v}_{t,b}} + \epsilon)(\sqrt{\hat{v}_{t, b}} + \sqrt{A_{t-1}\hat{v}_{t-1, b}/A_t})} \right. \nonumber\\
  && \left. + \frac{a_t\|g_{t,\mathcal{G}_b}\|_2^2/(A_td_b)}{(\sqrt{\hat{v}_{t,b}} + \epsilon)(\sqrt{\hat{v}_{t, b}} + \sqrt{A_{t-1}\hat{v}_{t-1, b}/A_t})}\left[\frac{1}{\sqrt{A_{t-1}\hat{v}_{t-1,b}/A_t} + \epsilon} - \frac{1}{\sqrt{\tilde{v}_{t,b}} + \epsilon}\right]\right] \nonumber\\
      & = & -\tilde{\eta}_{t,b}\frac{\frac{a_t}{A_td_b}\|g_{t,\mathcal{G}_b}\|_2^2}{\sqrt{\hat{v}_{t,b}} + \epsilon}\frac{\beta_tm_{t-1, \mathcal{G}_b}}{\sqrt{\hat{v}_{t, b}} + \sqrt{A_{t-1}\hat{v}_{t-1, b}/A_t}} \nonumber\\
       && -\tilde{\eta}_{t,b}\frac{\sigma_{t,b}}{\sqrt{d_b}}\left[\frac{\frac{a_t}{A_t\sqrt{d_b}}\|g_{t,\mathcal{G}_b}\|_2}{\sqrt{\hat{v}_{t,b}} + \epsilon}\frac{\beta_tm_{t-1, \mathcal{G}_b}}{\sqrt{A_{t-1}\hat{v}_{t-1,b}/A_t} + \epsilon}\frac{\sqrt{\frac{a_t}{A_td_b}}\|g_{t,\mathcal{G}_b}\|_2}{\sqrt{\hat{v}_{t, b}} + \sqrt{A_{t-1}\hat{v}_{t-1, b}/A_t}}\frac{\sqrt{\frac{a_t}{A_td_b}}\sigma_{t,b}}{\sqrt{\tilde{v}_{t, b}} + \sqrt{A_{t-1}\hat{v}_{t-1, b}/A_t}}\right]. \label{eq:nonconvex-bema-3}
\end{eqnarray}
Combining (\ref{eq:nonconvex-bema-2}) and (\ref{eq:nonconvex-bema-3}) into (\ref{eq:nonconvex-bema-1}), we obtain the result.
\end{proof}

\begin{lemma}\label{lemma:nonconvex-block-stepsize-ema-exp-bound}
Suppose that $\{a_t\}$ is a non-decreasing sequence and $A_t = \sum_{i=1}^ta_t$ such that $\{A_t/A_{t+1}\}$ is non-decreasing and $\lim_{t \rightarrow \infty}\frac{A_{t}}{A_{t+1}} = p > 0$. Let $\hat{A}_{t, i} = \prod_{j = i + 1}^t\frac{A_{j - 1}}{A_j}$ for $1 \leq i < t$ and $\hat{A}_{t, t} = 1$. For a fixed constant $\tilde{p}$ such that $\beta^2 <  \tilde{p} < p$, we have
\begin{eqnarray*}
\hat{A}_{t, i} \geq C_a\tilde{p}^{t - i} \hspace{.05in},
\end{eqnarray*}
where $C_a = \left(\prod_{j = 2}^N\frac{A_{j - 1}}{A_j\tilde{p}}\right)$ and 
N is the maximum of the indices for which $A_{j - 1}/A_j < \tilde{p}$. 
When there are no such indices, i.e., $A_1/A_2 \geq \tilde{p}$, we use $C_a = 1$ by convention. 
\end{lemma}
\begin{proof}
\begin{eqnarray*}
\hat{A}_{t, i} = \prod_{j = i + 1}^t\frac{A_{j - 1}}{A_j} \geq \left(\prod_{j = i + 1}^N\frac{A_{j - 1}}{A_j}\right)\tilde{p}^{t - N}
 = \left(\prod_{j = i + 1}^N\frac{A_{j - 1}}{A_j\tilde{p}}\right)\tilde{p}^{t - i} \geq \left(\prod_{j = 2}^N\frac{A_{j - 1}}{A_j\tilde{p}}\right)\tilde{p}^{t - i}.
\end{eqnarray*}
\end{proof}

\begin{lemma}\label{lemma:nonconvex-block-stepsize-ema-ratio-bound}
Suppose that $0 \leq \beta_t \leq \beta < 1$ for all $t$. Let $\rho := \frac{\beta^2}{\tilde{p}}$, where $\tilde{p}$ is defined in Lemma~\ref{lemma:nonconvex-block-stepsize-ema-exp-bound}. Then, for all $t$, we have
\begin{eqnarray*}
\|m_{t, \mathcal{G}_b}\|_2^2 \leq \frac{1}{C_aa_t/(A_td_b)(1 - \rho)}\hat{v}_{t, b},
\end{eqnarray*}
where $C_a$ is defined in Lemma~\ref{lemma:nonconvex-block-stepsize-ema-exp-bound}.
\end{lemma}
\begin{proof}
Let $\hat{\beta}_{t, i} = \prod_{j=i + 1}^t\beta_j$ for $i < t$ and $\hat{\beta}_{t, t} = 1$
\begin{eqnarray}
\|m_{t, \mathcal{G}_b}\|_2^2 &=& \left\|\sum_{i=1}^t(1 - \beta_i)\hat{\beta}_{t, i}g_{i, \mathcal{G}_b}\right\|^2_2 \nonumber\\
& = & \left\|\sum_{i=1}^t\frac{(1 - \beta_i)\hat{\beta}_{t, i}}{\sqrt{\frac{a_i}{A_td_b}}}\sqrt{\frac{a_i}{A_td_b}}g_{i, \mathcal{G}_b}\right\|^2_2 \nonumber\\
& \leq & \left(\sum_{i=1}^t\frac{(1 - \beta_i)^2\hat{\beta}_{t, i}^2}{\frac{a_i}{A_td_b}}\right)\left(\sum_{i=1}^t\frac{a_i}{A_td_b}\left\|g_{i, \mathcal{G}_b}\right\|^2_2\right) \nonumber\\
& = & \left(\sum_{i=1}^t\frac{(1 - \beta_i)^2\hat{\beta}_{t, i}^2}{\frac{a_i}{A_td_b}}\right)\hat{v}_{t,b}. \label{eq:nonconvex-bema-exp-1}
\end{eqnarray}
Then, with Lemma~\ref{lemma:nonconvex-block-stepsize-ema-exp-bound},  we get
\begin{eqnarray}
\sum_{i=1}^t\frac{(1 - \beta_i)^2\hat{\beta}_{t, i}^2}{\frac{a_i}{A_td_b}} = \sum_{i=1}^t\frac{(1 - \beta_i)^2\hat{\beta}_{t, i}^2}{\frac{a_i}{A_id_b}\hat{A}_{t, i} }  \leq \frac{1}{C_aa_t/(A_td_b)}\sum_{i=1}^t\left(\frac{\beta^{2}}{\tilde{p}}\right)^{t-i} \leq \frac{1}{C_aa_t/(A_td_b)(1 - \rho)}. \label{eq:nonconvex-bema-exp-2}
\end{eqnarray}
Then, combining (\ref{eq:nonconvex-bema-exp-1}) and (\ref{eq:nonconvex-bema-exp-2}), we obtain the result.
\end{proof}

\begin{lemma}\label{lemma:nonconvex-block-stepsize-ema-iterate-recursive}
Assume $F$ is $L$-smooth, $\{a_t\}$ is non-decreasing such that $\{A_{t-1}/A_t\}$ is non-decreasing and $\lim_{t \rightarrow \infty}\frac{A_{t}}{A_{t+1}} = p > 0$. Let $\tilde{p}$ be a constant such that $\beta^2 <  \tilde{p} < p$.  Assume $\CE_t[\|g_{t, \mG_b}\|^2_2] = \sigma_{t,b}^2 \leq d_b\sigma_b^2$.  
Define $w_t = \eta_t/\sqrt{\frac{a_t}{A_t}}$. Assume $w_t$ is "almost" non-increasing. This means there exists another non-increasing sequence $\{z_t\}$ and positive constants $C_1$ and $C_2$ such that $C_1z_t \leq w_t \leq C_2z_t$ for all $t$. Assume $0 \leq \beta_t \leq \beta < 1$ for all $t$. Define following Lyapunov function:
\begin{eqnarray*}
M_t = \CE[\langle \nabla F(\theta_t), \delta_t\rangle + L\|\delta_t\|_2^2].
\end{eqnarray*}
Let $C_3 \equiv \left[\frac{\beta/(1 - \beta)}{\sqrt{C_aA_{1}/A_2(1 - \rho)}}  + 1\right]$, where $\rho := \frac{\beta^2}{\tilde{p}}$. Then, for any $t\geq 2$, we have
\begin{eqnarray}
M_t
& \leq & \frac{\beta_t\eta_t}{\sqrt{A_{t-1}/A_t}\eta_{t-1}}M_{t-1} -\frac{1 - \beta_t}{2}\eta_t\CE\left[\left\|\nabla F(\theta_t)\right\|_{\tilde{H}_t^{-1}}^2\right] + 2w_tC_3^2\CE\left[\frac{a_t}{A_t}\|H^{-1}_tg_{t}\|_{\Sigma^{1/2}}^2\right] \nonumber\\
&& + L\CE[\|\delta_t\|_2^2] + \frac{\beta w_t}{\sqrt{C_a(1 - \rho)}}\left(\sqrt{\frac{A_t}{A_{t-1}}} - 1\right)\sum_{b=1}^B\sigma_bd_b, \label{eq:nonconvex-block-stepsize-ema-iterate-recursive-1}
\end{eqnarray} 
and for $t = 1$, we have
\begin{eqnarray}
M_1
& \leq &  -\frac{1 - \beta_1}{2}\eta_1\CE\left[\left\|\nabla F(\theta_1)\right\|_{\tilde{H}_1^{-1}}^2\right] + 2w_1C_3^2\CE\left[\frac{a_1}{A_1}\|H_1^{-1}g_{1}\|_{\Sigma^{1/2}}^2\right] + L\CE[\|\delta_1\|_2^2]. \label{eq:nonconvex-block-stepsize-ema-iterate-recursive-2}
\end{eqnarray} 
\end{lemma}
\begin{proof}
For any $t \geq 2$,
\begin{eqnarray}
\CE[\langle \nabla F(\theta_t), \delta_t\rangle] = \frac{\beta_t\eta_t}{\sqrt{A_{t-1}/A_t}\eta_{t-1}}\CE[\langle \nabla F(\theta_t), \delta_{t-1}\rangle] + 
\CE\left[\left\langle \nabla F(\theta_t), \delta_t - \frac{\beta_t\eta_t}{\sqrt{A_{t-1}/A_t}\eta_{t-1}}\delta_{t-1}\right\rangle\right]. \label{eq:nonconvex-bema-4}
\end{eqnarray}
Then, for the first term of (\ref{eq:nonconvex-bema-4}), we have
\begin{eqnarray*}
\langle \nabla F(\theta_t), \delta_{t-1}\rangle & = & \langle \nabla F(\theta_{t-1}), \delta_{t-1}\rangle + \langle \nabla F(\theta_t) - \nabla F(\theta_{t-1}), \delta_{t-1}\rangle \\
& \leq & \langle \nabla F(\theta_{t-1}), \delta_{t-1}\rangle + L\|\theta_t - \theta_{t-1}\|_2\|\delta_{t-1}\|_2 \\
& = &  \langle \nabla F(\theta_{t-1}), \delta_{t-1}\rangle + L\|\delta_{t-1}\|_2^2,
\end{eqnarray*}
where the first inequality follows from Schwartz inequality and the smoothness of the function $F$. Hence, we have
\begin{eqnarray*}
\frac{\beta_t\eta_t}{\sqrt{A_{t-1}/A_t}\eta_{t-1}}\CE[\langle \nabla F(\theta_t), \delta_{t-1}\rangle] & \leq & \frac{\beta_t\eta_t}{\sqrt{A_{t-1}/A_t}\eta_{t-1}}\CE\left[\langle \nabla F(\theta_{t-1}), \delta_{t-1}\rangle + L\|\delta_{t-1}\|_2^2\right] \\
& = & \frac{\beta_t\eta_t}{\sqrt{A_{t-1}/A_t}\eta_{t-1}}M_{t-1}.
\end{eqnarray*}
Now, we estimate the second term of (\ref{eq:nonconvex-bema-4}). By Lemma~\ref{lemma:nonconvex-block-stepsize-ema-iterate-diff}, for each block $b$, we get
\begin{eqnarray}
\lefteqn{\CE\left[\left\langle \nabla_{\mathcal{G}_b} F(\theta_t), \delta_{t, \mathcal{G}_b} - \frac{\beta_t\eta_t}{\sqrt{A_{t-1}/A_t}\eta_{t-1}}\delta_{t-1, \mathcal{G}_b}\right\rangle\right]} \nonumber\\
& =& - (1 - \beta_t)\CE[\langle \nabla_{\mathcal{G}_b} F(\theta_t), \tilde{\eta}_{t, b}g_{t,\mathcal{G}_b}\rangle] + \CE\left[\left\langle \nabla_{\mathcal{G}_b} F(\theta_t), \tilde{\eta}_{t, b}\frac{\frac{a_t}{A_td_b}\|g_{t,\mathcal{G}_b}\|_2^2}{\sqrt{\hat{v}_{t,b}} + \epsilon}X_{t,b}\right\rangle\right] \nonumber\\
&& +\CE\left[\left\langle \nabla_{\mathcal{G}_b} F(\theta_t), \tilde{\eta}_{t, b}\frac{\sigma_{t,b}}{\sqrt{d_b}}Y_{t,b}\right\rangle\right] + \CE[\langle \nabla_{\mathcal{G}_b} F(\theta_t), Z_{t,b} \rangle]. \label{eq:nonconvex-bema-5}
\end{eqnarray}
For the first term of (\ref{eq:nonconvex-bema-5}), we have
\begin{eqnarray}
- (1 - \beta_t)\CE[\langle \nabla_{\mathcal{G}_b} F(\theta_t), \tilde{\eta}_{t, b}g_{t,\mathcal{G}_b}\rangle]
& = & - (1 - \beta_t)\CE[\langle \nabla_{\mathcal{G}_b} F(\theta_t), \tilde{\eta}_{t, b}\nabla_{\mathcal{G}_b} F(\theta_t)\rangle] \nonumber\\
& = & - (1 - \beta_t)\tilde{\eta}_{t, b}\CE[\|\nabla_{\mathcal{G}_b} F(\theta_t)\|_2^2]. \label{eq:nonconvex-bema-6}
\end{eqnarray}
For the second term of (\ref{eq:nonconvex-bema-5}), we have
\begin{eqnarray}
\lefteqn{\CE\left[\left\langle \nabla_{\mathcal{G}_b} F(\theta_t), \tilde{\eta}_{t, b}\frac{\frac{a_t}{A_td_b}\|g_{t,\mathcal{G}_b}\|_2^2}{\sqrt{\hat{v}_{t,b}} + \epsilon}X_{t,b}\right\rangle\right]} \nonumber\\
& \leq &  \CE\left[\frac{\sqrt{\tilde{\eta}_{t, b}}\|\nabla_{\mathcal{G}_b} F(\theta_t)\|_2\|g_{t,\mathcal{G}_b}\|_2/\sqrt{d_b}}{\sigma_{t, b}/\sqrt{d_b}}\frac{\sqrt{\tilde{\eta}_{t, b}}\frac{a_t}{A_t}\|g_{t,\mathcal{G}_b}\|_2/\sqrt{d_b}\sigma_{t, b}/\sqrt{d_b}\|X_{t,b}\|_2}{\sqrt{\hat{v}_{t,b}} + \epsilon}\right]. \label{eq:nonconvex-bema-7}
\end{eqnarray}
Note that 
\begin{eqnarray}
\sqrt{\tilde{\eta}_{t, b}}\sigma_{t, b}/\sqrt{d_b} = \sqrt{\frac{\eta_t\sigma_{t,b}^2/d_b}{\sqrt{\tilde{v}_{t,b}} + \epsilon}} \leq \sqrt{\frac{\eta_t\sigma_{t,b}^2/d_b}{\sqrt{a_t/A_t\sigma_{t,b}^2/d_b}}} \leq \sqrt{\frac{\eta_t\sigma_{b}}{\sqrt{a_t/A_t}}} = \sqrt{w_t\sigma_b}. \label{eq:nonconvex-bema-8}
\end{eqnarray}
Besides, we have 
\begin{eqnarray*}
\|X_{t,b}\|_2 &=& \left\|\frac{\beta_tm_{t-1, \mathcal{G}_b}}{\sqrt{\hat{v}_{t, b}} + \sqrt{A_{t-1}\hat{v}_{t-1, b}/A_t}} + \frac{(1 - \beta_t)g_{t,\mathcal{G}_b}}{\sqrt{\tilde{v}_{t,b}} + \sqrt{\hat{v}_{t,b}}}\right\|_2 \\
& \leq & \left\|\frac{\beta_tm_{t-1, \mathcal{G}_b}}{\sqrt{\hat{v}_{t, b}} + \sqrt{A_{t-1}\hat{v}_{t-1, b}/A_t}}\right\|_2 + \left\|\frac{(1 - \beta_t)g_{t,\mathcal{G}_b}}{\sqrt{\tilde{v}_{t,b}} + \sqrt{\hat{v}_{t,b}}}\right\|_2.
\end{eqnarray*}
With Lemma~\ref{lemma:nonconvex-block-stepsize-ema-ratio-bound}, we have
\begin{eqnarray}
\left\|\frac{m_{t-1, \mathcal{G}_b}}{\sqrt{\hat{v}_{t, b}} + \sqrt{A_{t-1}\hat{v}_{t-1, b}/A_t}}\right\|_2 \leq 
\left\|\frac{m_{t-1, \mathcal{G}_b}}{\sqrt{A_{t-1}\hat{v}_{t-1, b}/A_t}}\right\|_2 \leq \frac{1}{\sqrt{C_aA_{t-1}/A_ta_t/(A_td_b)(1 - \rho)}}, \label{eq:nonconvex-bema-10}\\
\left\|\frac{g_{t,\mathcal{G}_b}}{\sqrt{\tilde{v}_{t,b}} + \sqrt{\hat{v}_{t,b}}}\right\|_2 \leq  \left\|\frac{g_{t,\mathcal{G}_b}}{\sqrt{\hat{v}_{t,b}}}\right\|_2
\leq \left\|\frac{g_{t,\mathcal{G}_b}}{\sqrt{a_t/A_t\|g_{t,\mathcal{G}_b}\|_2^2/d_b}}\right\|_2 = \frac{\sqrt{d_b}}{\sqrt{a_t/A_t}}  \label{eq:nonconvex-bema-11}. 
\end{eqnarray}
Then, we get
\begin{eqnarray*}
\|X_{t, b}\|_2 & \leq &\frac{\beta_t}{\sqrt{C_aA_{t-1}/A_ta_t/(A_td_b)(1 - \rho)}} + \frac{(1 - \beta_t)\sqrt{d_b}}{\sqrt{a_t/A_t}} \\
& = & \left[\frac{\beta_t/(1 - \beta_t)}{\sqrt{C_aA_{t-1}/A_t(1 - \rho)}}  + 1\right]\frac{(1 - \beta_t)\sqrt{d_b}}{\sqrt{a_t/A_t}} \\
& \leq & \left[\frac{\beta/(1 - \beta)}{\sqrt{C_aA_{t-1}/A_t(1 - \rho)}}  + 1\right]\frac{(1 - \beta_t)\sqrt{d_b}}{\sqrt{a_t/A_t}} \\
& \leq & \left[\frac{\beta/(1 - \beta)}{\sqrt{C_aA_{1}/A_2(1 - \rho)}}  + 1\right]\frac{(1 - \beta_t)\sqrt{d_b}}{\sqrt{a_t/A_t}} := C_3\frac{(1 - \beta_t)\sqrt{d_b}}{\sqrt{a_t/A_t}},
\end{eqnarray*}
where the last-to-second inequality follows from the assumption that $\beta_t \leq \beta$, and the last inequality holds as we assume $\{a_t\}$ is chosen such that $\{A_{t-1}/A_t\}$ is non-decreasing for all $t$. Hence, combining the above result with (\ref{eq:nonconvex-bema-8}) and (\ref{eq:nonconvex-bema-7}), we have
\begin{eqnarray}
\lefteqn{\CE\left[\left\langle \nabla_{\mathcal{G}_b} F(\theta_t), \tilde{\eta}_{t, b}\frac{\frac{a_t}{A_td_b}\|g_{t,\mathcal{G}_b}\|_2^2}{\sqrt{\hat{v}_{t,b}} + \epsilon}X_{t,b}\right\rangle\right]} \nonumber\\
& \leq &  \CE\left[\frac{\sqrt{\tilde{\eta}_{t, b}}\|\nabla_{\mathcal{G}_b} F(\theta_t)\|_2\|g_{t,\mathcal{G}_b}\|_2/\sqrt{d_b}}{\sigma_{t, b}/\sqrt{d_b}}\sqrt{w_t\sigma_{b}}C_3(1 - \beta_t)\frac{\sqrt{\frac{a_t}{A_t}}\|g_{t,\mathcal{G}_b}\|_2}{\sqrt{\hat{v}_{t,b}} + \epsilon}\right] \nonumber \nonumber\\
& \leq & \CE\left[\frac{1 - \beta_t}{4}\frac{\tilde{\eta}_{t, b}\|\nabla_{\mathcal{G}_b} F(\theta_t)\|_2^2\|g_{t,\mathcal{G}_b}\|_2^2/d_b}{\sigma_{t, b}^2/d_b} 
+ w_t\sigma_{b}C_3^2(1 - \beta_t)\frac{\frac{a_t}{A_t}\|g_{t,\mathcal{G}_b}\|_2^2}{(\sqrt{\hat{v}_{t,b}} + \epsilon)^2}\right] \nonumber\\
& \leq & \CE\left[\frac{1 - \beta_t}{4}\frac{\tilde{\eta}_{t, b}\|\nabla_{\mathcal{G}_b} F(\theta_t)\|_2^2\CE_t[\|g_{t,\mathcal{G}_b}\|_2]^2/d_b}{\sigma_{t, b}^2/d_b} 
+ w_t\sigma_{b}C_3^2(1 - \beta_t)\frac{\frac{a_t}{A_t}\|g_{t,\mathcal{G}_b}\|_2^2}{(\sqrt{\hat{v}_{t,b}} + \epsilon)^2}\right] \nonumber\\
& \leq & \CE\left[\frac{1 - \beta_t}{4}\tilde{\eta}_{t, b}\|\nabla_{\mathcal{G}_b} F(\theta_t)\|_2^2
+ w_t\sigma_{b}C_3^2\frac{\frac{a_t}{A_t}\|g_{t,\mathcal{G}_b}\|_2^2}{(\sqrt{\hat{v}_{t,b}} + \epsilon)^2}\right], \label{eq:nonconvex-bema-9}
\end{eqnarray}
where the second inequality follows from $ab \leq \frac{a^2}{2c} + \frac{cb^2}{2}$ for any $c > 0$. Now, we estimate the third term of (\ref{eq:nonconvex-bema-5}):
\begin{eqnarray*}
\CE\left[\left\langle \nabla_{\mathcal{G}_b} F(\theta_t), \tilde{\eta}_{t, b}\frac{\sigma_{t,b}}{\sqrt{d_b}}Y_{t,b}\right\rangle\right]  & \leq &
\CE\left[\sqrt{\tilde{\eta}_{t, b}}\left\|\nabla_{\mathcal{G}_b} F(\theta_t)\right\|_2\sqrt{\tilde{\eta}_{t, b}}\frac{\sigma_{t,b}}{\sqrt{d_b}}\left\|Y_{t,b}\right\|_2\right].
\end{eqnarray*}
Similarly, with (\ref{eq:nonconvex-bema-10}) and (\ref{eq:nonconvex-bema-11}), by expanding $\|Y_{t,b}\|_2$, we have
\begin{eqnarray*}
\|Y_{t,b}\|_2 &\leq& \frac{\frac{a_t}{A_t\sqrt{d_b}}\|g_{t,\mathcal{G}_b}\|_2}{\sqrt{\hat{v}_{t,b}} + \epsilon}\frac{\beta_t\|m_{t-1, \mathcal{G}_b}\|_2}{\sqrt{A_{t-1}\hat{v}_{t-1,b}/A_t} + \epsilon}\frac{\sqrt{\frac{a_t}{A_td_b}}\|g_{t,\mathcal{G}_b}\|_2}{\sqrt{\hat{v}_{t, b}} + \sqrt{A_{t-1}\hat{v}_{t-1, b}/A_t}}\frac{\sqrt{\frac{a_t}{A_td_b}}\sigma_{t,b}}{\sqrt{\tilde{v}_{t, b}} + \sqrt{A_{t-1}\hat{v}_{t-1, b}/A_t}}  \\
&& + \frac{\frac{a_t}{A_t}\|g_{t,\mathcal{G}_b}\|_2}{\sqrt{\hat{v}_{t,b}} + \epsilon}\frac{(1 - \beta_t)\frac{\sigma_{t,b}}{\sqrt{d_b}}}{\sqrt{\tilde{v}_{t,b}} + \sqrt{\hat{v}_{t,b}}} \\
&\leq& \frac{\frac{a_t}{A_t\sqrt{d_b}}\|g_{t,\mathcal{G}_b}\|_2}{\sqrt{\hat{v}_{t,b}} + \epsilon}\frac{\beta_t}{\sqrt{C_aA_{t-1}/A_ta_t/(A_td_b)(1 - \rho)}}  + \frac{\frac{a_t}{A_t}\|g_{t,\mathcal{G}_b}\|_2}{\sqrt{\hat{v}_{t,b}} + \epsilon}\frac{1 - \beta_t}{\sqrt{a_t/A_t}} \\
&=& \frac{\sqrt{\frac{a_t}{A_t}}\|g_{t,\mathcal{G}_b}\|_2}{\sqrt{\hat{v}_{t,b}} + \epsilon}\frac{\beta_t}{\sqrt{C_aA_{t-1}/A_t(1 - \rho)}}  + \frac{\sqrt{\frac{a_t}{A_t}}\|g_{t,\mathcal{G}_b}\|_2}{\sqrt{\hat{v}_{t,b}} + \epsilon}(1 - \beta_t) \\
&=& \frac{\sqrt{\frac{a_t}{A_t}}\|g_{t,\mathcal{G}_b}\|_2}{\sqrt{\hat{v}_{t,b}} + \epsilon}\left[\frac{\beta_t/(1 - \beta_t)}{\sqrt{C_aA_{t-1}/A_t(1 - \rho)}}  + 1\right](1 - \beta_t) \\
&\leq& \frac{\sqrt{\frac{a_t}{A_t}}\|g_{t,\mathcal{G}_b}\|_2}{\sqrt{\hat{v}_{t,b}} + \epsilon}C_3(1 - \beta_t),
\end{eqnarray*}
where $C_3$ is the constant defined above. Hence, together with (\ref{eq:nonconvex-bema-8}), we obtain
\begin{eqnarray}
\CE\left[\left\langle \nabla_{\mathcal{G}_b} F(\theta_t), \tilde{\eta}_{t, b}\frac{\sigma_{t,b}}{\sqrt{d_b}}Y_{t,b}\right\rangle\right]  & \leq &
\CE\left[\sqrt{\tilde{\eta}_{t, b}}\left\|\nabla_{\mathcal{G}_b} F(\theta_t)\right\|_2\sqrt{\tilde{\eta}_{t, b}}\frac{\sigma_{t,b}}{\sqrt{d_b}}\frac{\sqrt{\frac{a_t}{A_t}}\|g_{t,\mathcal{G}_b}\|_2}{\sqrt{\hat{v}_{t,b}} + \epsilon}C_3(1 - \beta_t)\right] \nonumber\\
& \leq & \frac{1 - \beta_t}{4}\tilde{\eta}_{t, b}\CE\left[\left\|\nabla_{\mathcal{G}_b} F(\theta_t)\right\|_2^2\right] + w_t\sigma_bC_3^2(1 - \beta_t)\CE\left[\frac{\frac{a_t}{A_t}\|g_{t,\mathcal{G}_b}\|_2^2}{(\sqrt{\hat{v}_{t,b}} + \epsilon)^2}\right] \nonumber\\
& \leq & \frac{1 - \beta_t}{4}\tilde{\eta}_{t, b}\CE\left[\left\|\nabla_{\mathcal{G}_b} F(\theta_t)\right\|_2^2\right] + w_t\sigma_bC_3^2\CE\left[\frac{\frac{a_t}{A_t}\|g_{t,\mathcal{G}_b}\|_2^2}{(\sqrt{\hat{v}_{t,b}} + \epsilon)^2}\right]. \label{eq:nonconvex-bema-13}
\end{eqnarray}
The last term of (\ref{eq:nonconvex-bema-5}) can be bounded as follows
\begin{eqnarray*}
\CE[\langle \nabla_{\mathcal{G}_b} F(\theta_t), Z_{t,b} \rangle] \leq \CE[\| \nabla_{\mathcal{G}_b} F(\theta_t)\|_2\|Z_{t,b}\|_2] \leq \CE[\sigma_b\sqrt{d_b}\|Z_{t,b}\|_2],
\end{eqnarray*}
and with (\ref{eq:nonconvex-bema-10}), we get
\begin{eqnarray*}
\|Z_{t,b}\|_2 & \leq & \frac{\beta_t\eta_t\left\|m_{t-1, \mathcal{G}_b}\right\|_2}{\sqrt{A_{t-1}\hat{v}_{t-1, b}/A_t} + \sqrt{A_{t-1}/A_t}\epsilon}\frac{\left(1 - \sqrt{A_{t-1}/A_t}\right)\epsilon}{(\sqrt{A_{t-1}\hat{v}_{t-1, b}/A_t} + \epsilon)} \\
& \leq & \frac{\beta_t\eta_t}{\sqrt{C_aA_{t-1}/A_ta_t/(A_td_b)(1 - \rho)}}\left(1 - \sqrt{A_{t-1}/A_t}\right) \\
& \leq & \frac{\beta\eta_t\sqrt{A_td_b/a_t}}{\sqrt{C_a(1 - \rho)}}\left(\sqrt{\frac{A_t}{A_{t-1}}} - 1\right) \\
& = & \frac{\beta w_t\sqrt{d_b}}{\sqrt{C_a(1 - \rho)}}\left(\sqrt{\frac{A_t}{A_{t-1}}} - 1\right).
\end{eqnarray*}
Hence, 
\begin{eqnarray} \label{eq:nonconvex-bema-12}
\CE[\langle \nabla_{\mathcal{G}_b} F(\theta_t), Z_{t,b} \rangle] \leq \frac{\beta\sigma_bd_bw_t}{\sqrt{C_a(1 - \rho)}}\left(\sqrt{\frac{A_t}{A_{t-1}}} - 1\right).
\end{eqnarray}
Combining (\ref{eq:nonconvex-bema-5}), (\ref{eq:nonconvex-bema-6}), (\ref{eq:nonconvex-bema-9}), (\ref{eq:nonconvex-bema-13}), and (\ref{eq:nonconvex-bema-12}), we get
\begin{eqnarray*}
\lefteqn{\CE\left[\left\langle \nabla_{\mathcal{G}_b} F(\theta_t), \delta_{t, \mathcal{G}_b} - \frac{\beta_t\eta_t}{\sqrt{A_{t-1}/A_t}\eta_{t-1}}\delta_{t-1, \mathcal{G}_b}\right\rangle\right]} \\
& \leq & -\frac{1 - \beta_t}{2}\tilde{\eta}_{t, b}\CE\left[\left\|\nabla_{\mathcal{G}_b} F(\theta_t)\right\|_2^2\right] + 2w_t\sigma_bC_3^2\CE\left[\frac{\frac{a_t}{A_t}\|g_{t,\mathcal{G}_b}\|_2^2}{(\sqrt{\hat{v}_{t,b}} + \epsilon)^2}\right] + \frac{\beta\sigma_bd_bw_t}{\sqrt{C_a(1 - \rho)}}\left(\sqrt{\frac{A_t}{A_{t-1}}} - 1\right).
\end{eqnarray*}
Summing from $b=1$ to $B$, we obtain
\begin{eqnarray*}
\lefteqn{\CE\left[\left\langle \nabla F(\theta_t), \delta_{t} - \frac{\beta_t\eta_t}{\sqrt{A_{t-1}/A_t}\eta_{t-1}}\delta_{t-1}\right\rangle\right]} \\
& \leq & -\frac{1 - \beta_t}{2}\eta_t\CE\left[\left\|\nabla F(\theta_t)\right\|_{\tilde{H}_t^{-1}}^2\right] + 2w_tC_3^2\CE\left[\frac{a_t}{A_t}\|H^{-1}_tg_{t}\|_{\Sigma^{1/2}}^2\right] \\
&& + \frac{\beta w_t}{\sqrt{C_a(1 - \rho)}}\left(\sqrt{\frac{A_t}{A_{t-1}}} - 1\right)\sum_{b=1}^B\sigma_bd_b. 
\end{eqnarray*}
Then, with (\ref{eq:nonconvex-bema-4}), we have
\begin{eqnarray*}
\lefteqn{\CE[\langle \nabla F(\theta_t), \delta_t\rangle]}\\ 
& \leq & \frac{\beta_t\eta_t}{\sqrt{A_{t-1}/A_t}\eta_{t-1}}M_{t-1} -\frac{1 - \beta_t}{2}\eta_t\CE\left[\left\|\nabla F(\theta_t)\right\|_{\tilde{H}_t^{-1}}^2\right] + 2w_tC_3^2\CE\left[\frac{a_t}{A_t}\|H^{-1}_tg_{t}\|_{\Sigma^{1/2}}^2\right] \\
&& + \frac{\beta w_t}{\sqrt{C_a(1 - \rho)}}\left(\sqrt{\frac{A_t}{A_{t-1}}} - 1\right)\sum_{b=1}^B\sigma_bd_b.
\end{eqnarray*}
We obtain (\ref{eq:nonconvex-block-stepsize-ema-iterate-recursive-1}) by adding the term $L\CE[\|\delta_t\|_2^2]$ to both sides of the above equation. 
When $t = 1$, we have
\begin{eqnarray}
M_1  & = & \CE[-\langle \nabla F(\theta_1), \eta_1m_1/(\sqrt{\hat{v}_1} + \epsilon)\rangle + L\|\delta_1\|_2^2] \nonumber\\
& = & \CE[-\langle \nabla F(\theta_1), \eta_1(1 - \beta_1)g_1/(\sqrt{\hat{v}_1} + \epsilon)\rangle + L\|\delta_1\|_2^2]. \label{eq:nonconvex-bema-14}
\end{eqnarray}
Then, following the derivation of (\ref{eq:nonconvex-bema-2}), for each block $b$, we have
\begin{eqnarray*}
\frac{(1 - \beta_1)\eta_1g_{1, \mathcal{G}_b}}{\sqrt{\hat{v}_{1,b}} + \epsilon}
& =& (1 - \beta_1)\tilde{\eta}_{1, b}g_{1,\mathcal{G}_b}
+ \tilde{\eta}_{1, b}\frac{\sigma_{1,b}}{\sqrt{d_b}}\frac{\frac{a_1}{A_1}g_{1,\mathcal{G}_b}}{\sqrt{\hat{v}_{1,b}} + \epsilon}\frac{(1 - \beta_1)\frac{\sigma_{1,b}}{\sqrt{d_b}}}{\sqrt{\tilde{v}_{1,b}} + \sqrt{\hat{v}_{1,b}}} \\
&& - \tilde{\eta}_{1, b}\frac{\frac{a_1}{A_1d_b}\|g_{1,\mathcal{G}_b}\|_2^2}{\sqrt{\hat{v}_{1,b}} + \epsilon}\frac{(1 - \beta_1)g_{1,\mathcal{G}_b}}{\sqrt{\tilde{v}_{1,b}} + \sqrt{\hat{v}_{1,b}}}. 
\end{eqnarray*}
Hence, with similar argument, we get
\begin{eqnarray*}
\CE\left[-\left\langle \nabla F(\theta_1), \eta_1\frac{(1 - \beta_1)g_1}{\sqrt{\hat{v}_1} + \epsilon}\right\rangle\right] 
& \leq & -\frac{1 - \beta_1}{2}\eta_1\CE\left[\left\|\nabla F(\theta_1)\right\|_{\tilde{H}_1^{-1}}^2\right] + 2w_1C_3^2\CE\left[\frac{a_1}{A_1}\|H^{-1}_1g_{1}\|_{\Sigma^{1/2}}^2\right].
\end{eqnarray*}
Combining above with (\ref{eq:nonconvex-bema-14}), and adding $L\CE[\|\delta\|_2^2]$, we obtain (\ref{eq:nonconvex-block-stepsize-ema-iterate-recursive-2}).
\end{proof}

\begin{lemma}\label{lemma:nonconvex-block-stepsize-ema-iterate-delta-sum}
With the same assumptions in Lemma~\ref{lemma:nonconvex-block-stepsize-ema-iterate-recursive}, we have
\begin{eqnarray*}
\sum_{t=1}^T\|\delta_t\|_2^2 \leq \frac{C_2^2/C_1^2w_1}{C_a(1 - \sqrt{\rho})^2}\sum_{t=1}^Tw_t\frac{a_t}{A_t}\|H_t^{-1}g_t\|_2^2.
\end{eqnarray*} 
\end{lemma}
\begin{proof}
For each block $b$,
\begin{eqnarray*}
\|m_{t, \mathcal{G}_b}\|_2 & = & \left\|\sum_{i=1}^t\left(\prod_{j=i+1}^t\beta_j\right)(1-\beta_i)g_{i, \mathcal{G}_b}\right\|_2 \\
&\leq& \sum_{i=1}^t\left(\prod_{j=i+1}^t\beta_j(1-\beta_i)\right)\left\|g_{i, \mathcal{G}_b}\right\|_2 \\
&\leq& \sum_{i=1}^t\beta^{t-i}\left\|g_{i, \mathcal{G}_b}\right\|_2.
\end{eqnarray*}
Then, 
\begin{eqnarray*}
\frac{\|m_{t, \mathcal{G}_b}\|_2}{\sqrt{\hat{v}}_{t,b} + \epsilon} 
 \leq  \sum_{i=1}^t\frac{\beta^{t-i}\left\|g_{i, \mathcal{G}_b}\right\|_2}{\sqrt{\hat{v}}_{t,b} + \epsilon}.
\end{eqnarray*}
Since $\hat{v}_{t,b} \geq A_{t-1}\hat{v}_{t-1,b}/A_t$, we have $\hat{v}_{t,b} \geq \left(\prod_{j=i+1}^tA_{j-1}/A_j\right)\hat{v}_{i,b} = \hat{A}_{t, i}\hat{v}_{i,b} \geq C_a\tilde{p}^{t-i}\hat{v}_{i,b}$ by Lemma~\ref{lemma:nonconvex-block-stepsize-ema-exp-bound}. It follows that
\begin{eqnarray*}
\frac{\|m_{t, \mathcal{G}_b}\|_2}{\sqrt{\hat{v}}_{t,b} + \epsilon}  \leq \sum_{i=1}^t\frac{\beta^{t-i}\left\|g_{i, \mathcal{G}_b}\right\|_2}{\sqrt{\hat{v}}_{t,b} + \epsilon} \leq \frac{1}{\sqrt{C_a}}\sum_{i=1}^t\left(\frac{\beta}{\sqrt{\tilde{p}}}\right)^{t-i}\frac{\left\|g_{i, \mathcal{G}_b}\right\|_2}{\sqrt{\hat{v}}_{i,b} + \epsilon} =  \frac{1}{\sqrt{C_a}}\sum_{i=1}^t\sqrt{\rho}^{t-i}\frac{\left\|g_{i, \mathcal{G}_b}\right\|_2}{\sqrt{\hat{v}}_{i,b} + \epsilon}.
\end{eqnarray*}
Then, as $a_t/A_t = 1 - A_{t-1}/A_t$ is non-decreasing, we have
\begin{eqnarray*}
\|\delta_t\|_2^2 & = & \sum_{b=1}^B\left\|\frac{\eta_tm_{t, \mathcal{G}_b}}{\sqrt{\hat{v}_{t,b}} + \epsilon}\right\|_2^2
= \sum_{b=1}^B\left\|\frac{w_t\sqrt{a_t/A_t}m_{t, \mathcal{G}_b}}{\sqrt{\hat{v}_{t,b}} + \epsilon}\right\|_2^2
  \leq  \frac{w_t^2}{C_a}\sum_{b=1}^B\left(\sum_{i=1}^t\sqrt{\rho}^{t-i}\frac{\sqrt{a_t/A_t}\left\|g_{i, \mathcal{G}_b}\right\|_2}{\sqrt{\hat{v}_{i,b}} + \epsilon}\right)^2\\
  & \leq & \frac{w_t^2}{C_a}\sum_{b=1}^B\left(\sum_{i=1}^t\sqrt{\rho}^{t-i}\frac{\sqrt{a_i/A_i}\left\|g_{i, \mathcal{G}_b}\right\|_2}{\sqrt{\hat{v}_{i,b}} + \epsilon}\right)^2 \\
   &\leq& \frac{w_t^2}{C_a}\sum_{b=1}^B\left(\sum_{j=1}^t\sqrt{\rho}^{t-j}\right)^2\left(\sum_{i=1}^t\frac{\sqrt{\rho}^{t-i}}{\left(\sum_{j=1}^t\sqrt{\rho}^{t-j}\right)}\frac{\sqrt{a_i/A_i}\left\|g_{i, \mathcal{G}_b}\right\|_2}{\sqrt{\hat{v}_{i,b}} + \epsilon}\right)^2\\
  & \leq & \frac{w_t^2}{C_a}\sum_{b=1}^B\left(\sum_{j=1}^t\sqrt{\rho}^{t-j}\right)\sum_{i=1}^t\sqrt{\rho}^{t-i}\frac{a_i/A_i\left\|g_{i, \mathcal{G}_b}\right\|_2^2}{(\sqrt{\hat{v}_{i,b}} + \epsilon)^2}  \leq  \frac{w_t^2}{C_a(1 - \sqrt{\rho})}\sum_{b=1}^B\sum_{i=1}^t\sqrt{\rho}^{t-i}\frac{a_i/A_i\left\|g_{i, \mathcal{G}_b}\right\|_2^2}{(\sqrt{\hat{v}_{i,b}} + \epsilon)^2}\\
  &=& \frac{w_t^2}{C_a(1 - \sqrt{\rho})}\sum_{i=1}^t\sqrt{\rho}^{t-i}\frac{a_i}{A_i}\|H_i^{-1}g_i\|_2^2.
\end{eqnarray*} 
As $w_t \leq C_2/C_1w_i$ for any $i \leq t$ by Lemma~\ref{lemma:nonconvex-block-stepsize-ema-stepsize}, then we have
\begin{eqnarray*}
\|\delta_t\|_2^2 
 \leq \frac{C_2^2/C_1^2w_1}{C_a(1 - \sqrt{\rho})}\sum_{i=1}^t\sqrt{\rho}^{t-i}w_i\frac{a_i}{A_i}\|H_i^{-1}g_i\|_2^2.
\end{eqnarray*} 
Hence,
\begin{eqnarray*}
\sum_{t=1}^T\|\delta_t\|_2^2 
 & \leq & \frac{C_2^2/C_1^2w_1}{C_a(1 - \sqrt{\rho})}\sum_{t=1}^T\sum_{i=1}^t\sqrt{\rho}^{t-i}w_i\frac{a_i}{A_i}\|H_i^{-1}g_i\|_2^2\\
  & = & \frac{C_2^2/C_1^2w_1}{C_a(1 - \sqrt{\rho})}\sum_{i=1}^T\sum_{t=i}^T\sqrt{\rho}^{t-i}w_i\frac{a_i}{A_i}\|H_i^{-1}g_i\|_2^2 \\
    & \leq & \frac{C_2^2/C_1^2w_1}{C_a(1 - \sqrt{\rho})^2}\sum_{i=1}^Tw_i\frac{a_i}{A_i}\|H_i^{-1}g_i\|_2^2.
\end{eqnarray*} 
\end{proof}

\begin{lemma}\label{lemma:nonconvex-block-stepsize-ema-iterate-recursive-sum}
With the same assumptions in Lemma~\ref{lemma:nonconvex-block-stepsize-ema-iterate-recursive}, let $M_t = \CE[\langle \nabla F(\theta_t), \delta_t\rangle + L\|\delta_t\|_2^2]$, we have
\begin{eqnarray*}
\sum_{t=1}^TM_t 
& \leq & \frac{C_2}{C_1\sqrt{C_a}(1 - \sqrt{\rho})}\left[2C_3^2\sum_{t=1}^Tw_t\CE\left[\frac{a_t}{A_t}\|H^{-1}_tg_{t}\|_{\Sigma^{1/2}}^2\right] + \frac{LC_2^2/C_1^2w_1}{C_a(1 - \sqrt{\rho})^2}\sum_{t=1}^Tw_t\CE\left[\frac{a_t}{A_t}\|H_t^{-1}g_t\|_2^2\right] \right. \\
&& \left. + \frac{\beta}{\sqrt{C_a(1 - \rho)}}\sum_{b=1}^B\sigma_bd_b\sum_{t=2}^Tw_t\left(\sqrt{\frac{A_t}{A_{t-1}}} - 1\right)\right] 
- \frac{1 - \beta}{2}\sum_{t=1}^T\eta_t\CE\left[\left\|\nabla F(\theta_t)\right\|_{\tilde{H}_t^{-1}}^2\right].
\end{eqnarray*} 
\end{lemma}
\begin{proof}
Let define following quantity
\begin{eqnarray*}
N_t &=& 2w_tC_3^2\CE\left[\frac{a_t}{A_t}\|H^{-1}_tg_{t}\|_{\Sigma^{1/2}}^2\right] + L\CE[\|\delta_t\|_2^2] + \frac{\beta w_t}{\sqrt{C_a(1 - \rho)}}\left(\sqrt{\frac{A_t}{A_{t-1}}} - 1\right)\sum_{b=1}^B\sigma_bd_b, \hspace{.05in} \forall t \geq 2, \\
N_1 & = & 2w_1C_3^2\CE\left[\frac{a_1}{A_1}\|H^{-1}_1g_{1}\|_{\Sigma^{1/2}}^2\right] + L\CE[\|\delta_1\|_2^2].
\end{eqnarray*}  
Then, by Lemma~\ref{lemma:nonconvex-block-stepsize-ema-iterate-recursive}, for any $t \geq 2$, we have
\begin{eqnarray*}
M_t & \leq & \frac{\beta_t\eta_t}{\sqrt{A_{t-1}/A_t}\eta_{t-1}}M_{t-1} -\frac{1 - \beta_t}{2}\eta_t\CE\left[\left\|\nabla F(\theta_t)\right\|_{\tilde{H}_t^{-1}}^2\right] + N_t \\
& \leq & \frac{\beta_t\eta_t}{\sqrt{A_{t-1}/A_t}\eta_{t-1}}M_{t-1}  + N_t 
\end{eqnarray*} 
and $M_1 \leq N_1$. Then, by recursively applying above relation, we get
\begin{eqnarray*}
M_t & \leq & \frac{\hat{\beta}_{t, 1}\eta_t}{\sqrt{\hat{A}_{t, 1}}\eta_1}M_1 + \sum_{i=2}^t \frac{\hat{\beta}_{t, i}\eta_t}{\sqrt{\hat{A}_{t, i}}\eta_i}N_i -\frac{1 - \beta_t}{2}\eta_t\CE\left[\left\|\nabla F(\theta_t)\right\|_{\tilde{H}_t^{-1}}^2\right] \\
& \leq & \sum_{i=1}^t \frac{\hat{\beta}_{t, i}\eta_t}{\sqrt{\hat{A}_{t, i}}\eta_i}N_i -\frac{1 - \beta_t}{2}\eta_t\CE\left[\left\|\nabla F(\theta_t)\right\|_{\tilde{H}_t^{-1}}^2\right],
\end{eqnarray*} 
where $\hat{\beta}_{t, i} = \prod_{j=i + 1}^t\beta_j$ for $i < t$ and $\hat{\beta}_{t, t} = 1$ and $\hat{A}_{t, i} = \prod_{j = i + 1}^t\frac{A_{j - 1}}{A_j}$ for $i < t$ and $\hat{A}_{t, t} = 1$. 
Note that $\hat{\beta}_{t, i}  \leq \beta^{t- i}$, and $\eta_t \leq C_2/C_1\eta_i$. By Lemma~\ref{lemma:nonconvex-block-stepsize-ema-exp-bound}, we have  $\hat{A}_{t, i} \geq C_a\tilde{p}^{t - i} \hspace{.05in}$. Then,
\begin{eqnarray*}
M_t & \leq & \frac{C_2}{C_1\sqrt{C_a}}\sum_{i=1}^t \left(\frac{\beta}{\sqrt{\tilde{p}}}\right)^{t- i}N_i -\frac{1 - \beta_t}{2}\eta_t\CE\left[\left\|\nabla F(\theta_t)\right\|_{\tilde{H}_t^{-1}}^2\right] \\
& = & \frac{C_2}{C_1\sqrt{C_a}}\sum_{i=1}^t\sqrt{\rho}^{t- i}N_i -\frac{1 - \beta_t}{2}\eta_t\CE\left[\left\|\nabla F(\theta_t)\right\|_{\tilde{H}_t^{-1}}^2\right].
\end{eqnarray*} 
It can be verified that the above inequality holds for $t = 1$ as $C_2/(C_1\sqrt{C_a}) \geq 1$. Then, summing from $t=1$ to $t = T$, we obtain
\begin{eqnarray}
\sum_{t=1}^TM_t & \leq &  \frac{C_2}{C_1\sqrt{C_a}}\sum_{t=1}^T\sum_{i=1}^t\sqrt{\rho}^{t- i}N_i - \sum_{t=1}^T\frac{1 - \beta_t}{2}\eta_t\CE\left[\left\|\nabla F(\theta_t)\right\|_{\tilde{H}_t^{-1}}^2\right] \nonumber\\
& = & \frac{C_2}{C_1\sqrt{C_a}}\sum_{i=1}^T\sum_{t=i}^T\sqrt{\rho}^{t- i}N_i - \sum_{t=1}^T\frac{1 - \beta_t}{2}\eta_t\CE\left[\left\|\nabla F(\theta_t)\right\|_{\tilde{H}_t^{-1}}^2\right] \nonumber\\
& \leq &  \frac{C_2}{C_1\sqrt{C_a}(1 - \sqrt{\rho})}\sum_{t=1}^TN_t - \frac{1 - \beta}{2}\sum_{t=1}^T\eta_t\CE\left[\left\|\nabla F(\theta_t)\right\|_{\tilde{H}_t^{-1}}^2\right]. \label{eq:nonconvex-block-stepsize-ema-iterate-recursive-sum-1}
\end{eqnarray} 
With Lemma~\ref{lemma:nonconvex-block-stepsize-ema-iterate-delta-sum}, we get
\begin{eqnarray*}
\sum_{t=1}^TN_t & = & \sum_{t=1}^T\left[2w_tC_3^2\CE\left[\frac{a_t}{A_t}\|H^{-1}_tg_{t}\|_{\Sigma^{1/2}}^2\right] + L\CE[\|\delta_t\|_2^2]\right] + \sum_{t=2}^T\frac{\beta w_t}{\sqrt{C_a(1 - \rho)}}\left(\sqrt{\frac{A_t}{A_{t-1}}} - 1\right)\sum_{b=1}^B\sigma_bd_b \\
& \leq & 2C_3^2\sum_{t=1}^Tw_t\CE\left[\frac{a_t}{A_t}\|H^{-1}_tg_{t}\|_{\Sigma^{1/2}}^2\right] + \frac{LC_2^2/C_1^2w_1}{C_a(1 - \sqrt{\rho})^2}\sum_{t=1}^Tw_t\CE\left[\frac{a_t}{A_t}\|H_t^{-1}g_t\|_2^2\right]\\
&& + \frac{\beta}{\sqrt{C_a(1 - \rho)}}\sum_{b=1}^B\sigma_bd_b\sum_{t=2}^Tw_t\left(\sqrt{\frac{A_t}{A_{t-1}}} - 1\right).
\end{eqnarray*} 
Combining the above with (\ref{eq:nonconvex-block-stepsize-ema-iterate-recursive-sum-1}), we obtain the result.
\end{proof}

\begin{lemma}\label{lemma:nonconvex-block-stepsize-grad-norm}
Assume $\{a_t\}$ is non-decreasing such that $\{A_{t-1}/A_t\}$ is non-decreasing. Define $w_t = \eta_t/\sqrt{\frac{a_t}{A_t}}$. Assume $w_t$ is "almost" non-increasing. This means there exists another non-increasing sequence $\{z_t\}$ and positive constants $C_1$ and $C_2$ such that $C_1z_t \leq w_t \leq C_2z_t$. 
Assume $\CE_t[\|g_{t, \mG_b}\|^2_2] = \sigma_{t,b}^2 \leq d_b\sigma_b^2$. We have
\begin{eqnarray*}
\frac{1}{T}\sum_{t=1}^T\left(\CE[\|\nabla F(\theta_t)\|^{4/3}_2]\right)^{3/2} 
& \leq & \frac{\sqrt{2\left(\max_{1 \leq t \leq T}\CE\left[\max_b \tilde{v}_{t,b}\right] + \epsilon^2\right)}}{C_1/C_2\eta_T T}\sum_{t=1}^T\eta_t\CE\left[\|\nabla F(\theta_t)\|_{\tilde{H}^{-1}_t}^2\right].
\end{eqnarray*}
\end{lemma}
\begin{proof}
By H$\ddot{\text{o}}$lder's inequality, 
we have $\CE[|XY|] \leq (\CE[|X|^p])^{1/p}(\CE[|Y|^q])^{1/q}$ for any $0 < p, q < 1$ with $1/p + 1/q = 1$. Taking $p=3/2$, $q=3$, and
\begin{eqnarray*}
X = \left(\frac{\|\nabla F(\theta_t)\|^2_2}{\sqrt{\max_b \tilde{v}_{t,b}} + \epsilon}\right)^{2/3}, \hspace{.1in} Y = \left(\sqrt{\max_b \tilde{v}_{t,b}} + \epsilon\right)^{2/3},
\end{eqnarray*}
we obtain
\begin{eqnarray*}
\CE[\|\nabla F(\theta_t)\|^{4/3}_2] \leq \left(\CE\left[\frac{\|\nabla F(\theta_t)\|^2_2}{\sqrt{\max_b\tilde{v}_{t,b}} + \epsilon}\right]\right)^{2/3}\left(\CE\left[\left(\sqrt{\max_b\tilde{v}_{t,b}} + \epsilon\right)^{2}\right]\right)^{1/3}.
\end{eqnarray*}
Hence,
\begin{eqnarray*}
\left(\CE[\|\nabla F(\theta_t)\|^{4/3}_2]\right)^{3/2} \leq \left(\CE\left[\frac{\|\nabla F(\theta_t)\|^2_2}{\sqrt{\max_b\tilde{v}_{t,b}} + \epsilon}\right]\right)\left(\CE\left[\left(\sqrt{\max_b\tilde{v}_{t,b}} + \epsilon\right)^{2}\right]\right)^{1/2}.
\end{eqnarray*}
Note that 
\begin{eqnarray*}
\frac{\|\nabla F(\theta_t)\|^2_2}{\sqrt{\max_b\tilde{v}_{t,b}} + \epsilon} 
& = & \sum_{i=1}^B\frac{\|\nabla_{\mathcal{G}_i} F(\theta_t)\|^2_2}{\sqrt{\max_b\tilde{v}_{t,b}} + \epsilon} \\
& \leq & \sum_{b=1}^B\frac{\|\nabla_{\mathcal{G}_b} F(\theta_t)\|^2_2}{\sqrt{\tilde{v}_{t,b}} + \epsilon} \\
& = & \|\nabla F(\theta_t)\|_{\tilde{H}^{-1}_t}^2.
\end{eqnarray*}
We also have
\begin{eqnarray*}
\CE\left[\left(\sqrt{\max_b\tilde{v}_{t,b}} + \epsilon\right)^{2}\right]
& \leq & 2\CE\left[\left(\max_b\tilde{v}_{t,b} + \epsilon^2\right)\right] \\
& =& 2\left(\CE\left[\max_b \tilde{v}_{t,b}\right] + \epsilon^2\right).
\end{eqnarray*}
Then, for any $t \leq T$, we get
\begin{eqnarray*}
\left(\CE[\|\nabla F(\theta_t)\|^{4/3}_2]\right)^{3/2} &\leq& \sqrt{2\left(\CE\left[\max_b \tilde{v}_{t,b}\right] + \epsilon^2\right)}\CE\left[\|\nabla F(\theta_t)\|_{\tilde{H}^{-1}_t}^2\right]\\
& = & \frac{\sqrt{2\left(\CE\left[\max_b \tilde{v}_{t,b}\right] + \epsilon^2\right)}}{\eta_t}\eta_t\CE\left[\|\nabla F(\theta_t)\|_{\tilde{H}^{-1}_t}^2\right] \\
& \leq & \frac{\sqrt{2\left(\max_{1 \leq t \leq T}\CE\left[\max_b \tilde{v}_{t,b}\right] + \epsilon^2\right)}}{C_1/C_2\eta_T}\eta_t\CE\left[\|\nabla F(\theta_t)\|_{\tilde{H}^{-1}_t}^2\right],
\end{eqnarray*}
where the last inequality follows from Lemma~\ref{lemma:nonconvex-block-stepsize-ema-stepsize}. Taking average from $t=1$ to $T$, we get
\begin{eqnarray*}
\frac{1}{T}\sum_{t=1}^T\left(\CE[\|\nabla F(\theta_t)\|^{4/3}_2]\right)^{3/2} 
& \leq & \frac{\sqrt{2\left(\max_{1 \leq t \leq T}\CE\left[\max_b \tilde{v}_{t,b}\right] + \epsilon^2\right)}}{C_1/C_2\eta_T T}\sum_{t=1}^T\eta_t\CE\left[\|\nabla F(\theta_t)\|_{\tilde{H}^{-1}_t}^2\right].
\end{eqnarray*}
\end{proof}

%%%%%%%%%%%%%%%%%%%%%%%%%%%%%%%%%%%%%%

\subsection{Proof of Theorem~\ref{theorem:nonconvex_momentum}}
\begin{proof}
As $F$ is $L$-smooth, then we have
\begin{eqnarray*}
F(\theta_{t+1}) \leq  F(\theta_t) + \langle \nabla F(\theta_t), \theta_{t+1} - \theta_t \rangle + \frac{L}{2}\|\theta_{t+1} - \theta_t\|^2.
\end{eqnarray*}
Recursively applying the above relation, we get
\begin{eqnarray*}
F(\theta_*) \leq \CE[F(\theta_{T+1})] \leq F(\theta_1) + \sum_{t=1}^TM_t,
\end{eqnarray*}
where $M_t = \CE[\langle \nabla F(\theta_t), \delta_t\rangle + L\|\delta_t\|_2^2]$. By Lemma~\ref{lemma:nonconvex-block-stepsize-ema-iterate-recursive-sum}, we have
\begin{eqnarray*}
\lefteqn{\frac{1 - \beta}{2}\sum_{t=1}^T\eta_t\CE\left[\left\|\nabla F(\theta_t)\right\|_{\tilde{H}_t^{-1}}^2\right]} \\
& \leq & F(\theta_1) - F(\theta_*) + \frac{C_2}{C_1\sqrt{C_a}(1 - \sqrt{\rho})}\left[2C_3^2\sum_{t=1}^Tw_t\CE\left[\frac{a_t}{A_t}\|H^{-1}_tg_{t}\|_{\Sigma^{1/2}}^2\right]  \right.\\
&& \left. + \frac{LC_2^2/C_1^2w_1}{C_a(1 - \sqrt{\rho})^2}\sum_{t=1}^Tw_t\CE\left[\frac{a_t}{A_t}\|H_t^{-1}g_t\|_2^2\right] 
+ \frac{\beta}{\sqrt{C_a(1 - \rho)}}\sum_{b=1}^B\sigma_bd_b\sum_{t=2}^Tw_t\left(\sqrt{\frac{A_t}{A_{t-1}}} - 1\right)\right] \\
& = & F(\theta_1) - F(\theta_*) + \frac{C_2}{C_1\sqrt{C_a}(1 - \sqrt{\rho})}\left[2C_3^2\sum_{t=1}^T\eta_t\CE\left[\sqrt{\frac{a_t}{A_t}}\|H^{-1}_tg_{t}\|_{\Sigma^{1/2}}^2\right] \right. \\ 
&& \left. + \frac{LC_2^2/C_1^2w_1}{C_a(1 - \sqrt{\rho})^2}\sum_{t=1}^T\eta_t\CE\left[\sqrt{\frac{a_t}{A_t}}\|H_t^{-1}g_t\|_2^2\right]  
 + \frac{\beta}{\sqrt{C_a(1 - \rho)}}\sum_{b=1}^B\sigma_bd_b\sum_{t=2}^Tw_t\left(\sqrt{\frac{A_t}{A_{t-1}}} - 1\right)\right].
\end{eqnarray*} 
Applying Lemma~\ref{lemma:nonconvex-momentum-grad-weighted-bound}, we have
\begin{eqnarray*}
\lefteqn{\frac{1 - \beta}{2}\sum_{t=1}^T\eta_t\CE\left[\left\|\nabla F(\theta_t)\right\|_{\tilde{H}_t^{-1}}^2\right]} \\
& \leq & F(\theta_1) - F(\theta_*) \\
&& + \frac{C_2}{C_1\sqrt{C_a}(1 - \sqrt{\rho})}\left[2C_3^2\frac{C_2}{C_1}\left[w_1 \sum_{b=1}^B\sigma_{b}d_b\log\left(\frac{\sigma_b^2}{\epsilon^2} + 1\right)  + \sum_{b=1}^B\sigma_{b}d_b\sum_{t=1}^{T}\eta_t\sqrt{\frac{a_t}{A_t}}\frac{A_t}{A_{t-1} + a_1}\right] \right. \\
 && \left.
+ \frac{LC_2^3/C_1^3w_1}{C_a(1 - \sqrt{\rho})^2}\left[w_1 \sum_{b=1}^Bd_b\log\left(\frac{\sigma_b^2}{\epsilon^2} + 1\right) 
  + \sum_{b=1}^Bd_b\sum_{t=1}^{T}\eta_t\sqrt{\frac{a_t}{A_t}}\frac{A_t}{A_{t-1} + a_1}\right] \right. \\
&& \left. + \frac{\beta}{\sqrt{C_a(1 - \rho)}}\sum_{b=1}^B\sigma_bd_b\sum_{t=2}^Tw_t\left(\sqrt{\frac{A_t}{A_{t-1}}} - 1\right)\right] \\
& \leq & F(\theta_1) - F(\theta_*) \\
&& + \frac{C_2}{C_1\sqrt{C_a}(1 - \sqrt{\rho})}\left[2C_3^2\frac{C_2}{C_1}\left[w_1 \sum_{b=1}^B\sigma_{b}d_b\log\left(\frac{\sigma_b^2}{\epsilon^2} + 1\right) 
  + \omega\sum_{b=1}^B\sigma_{b}d_b\sum_{t=1}^{T}\eta_t\sqrt{\frac{a_t}{A_t}}\right] \right. \\
 && \left.
+ \frac{LC_2^3/C_1^3w_1}{C_a(1 - \sqrt{\rho})^2}\left[w_1 \sum_{b=1}^Bd_b\log\left(\frac{\sigma_b^2}{\epsilon^2} + 1\right) 
  + \omega\sum_{b=1}^Bd_b\sum_{t=1}^{T}\eta_t\sqrt{\frac{a_t}{A_t}}\right] \right. \\
&& \left. + \frac{\beta}{\sqrt{C_a(1 - \rho)}}\sum_{b=1}^B\sigma_bd_b\sum_{t=2}^Tw_t\left(\sqrt{\frac{A_t}{A_{t-1}}} - 1\right)\right] \\
& = & F(\theta_1) - F(\theta_*) + \frac{C_2}{C_1\sqrt{C_a}(1 - \sqrt{\rho})}\left[\frac{\beta}{\sqrt{C_a(1 - \rho)}}\sum_{b=1}^B\sigma_bd_b\sum_{t=2}^Tw_t\left(\sqrt{\frac{A_t}{A_{t-1}}} - 1\right) \right. \\
 && \left. + \sum_{b=1}^B\left[\frac{LC_3^3/C_1^3w_1d_b}{C_a(1 - \sqrt{\rho})^2} + \frac{2C_3^2C_2\sigma_bd_b}{C_1}\right]\left[w_1\log\left(\frac{\sigma_b^2}{\epsilon^2} + 1\right) 
  + \omega\sum_{t=1}^{T}\eta_t\sqrt{\frac{a_t}{A_t}}\right]\right].
\end{eqnarray*} 
Combining above with Lemma~\ref{lemma:nonconvex-block-stepsize-grad-norm}, we have
\begin{eqnarray}
\lefteqn{\min_{1 \leq t \leq T}\left(\CE[\|\nabla F(\theta_t)\|^{4/3}_2]\right)^{3/2} \leq \frac{1}{T}\sum_{t=1}^T\left(\CE[\|\nabla F(\theta_t)\|^{4/3}_2]\right)^{3/2}} \nonumber\\
& \leq & \frac{2\sqrt{2\left(\max_{1 \leq t \leq T}\CE\left[\max_b \tilde{v}_{t,b}\right] + \epsilon^2\right)}}{C_1/C_2(1-\beta)\eta_TT}\left[F(\theta_1) - F(\theta_*) \right. \nonumber\\
&& \left. + \frac{C_2}{C_1\sqrt{C_a}(1 - \sqrt{\rho})}\left[\frac{\beta}{\sqrt{C_a(1 - \rho)}}\sum_{b=1}^B\sigma_bd_b\sum_{t=2}^Tw_t\left(\sqrt{\frac{A_t}{A_{t-1}}} - 1\right) \right.\right. \nonumber\\
 && \left.\left. + \sum_{b=1}^B\left[\frac{LC_2^3/C_1^3w_1d_b}{C_a(1 - \sqrt{\rho})^2} + \frac{2C_3^2C_2\sigma_bd_b}{C_1}\right]\left[w_1\log\left(\frac{\sigma_b^2}{\epsilon^2} + 1\right) 
  + \omega\sum_{t=1}^{T}\eta_t\sqrt{\frac{a_t}{A_t}}\right]\right]\right] \nonumber\\
  & = & \frac{\sqrt{2\left(\max_{1 \leq t \leq T}\CE\left[\max_b \tilde{v}_{t,b}\right] + \epsilon^2\right)}}{\eta_TT}\left[\frac{2C_2}{(1 - \beta)C_1}[F(\theta_1) - F(\theta_*)] \right. \nonumber\\
  && \left. + \frac{2C_2^2}{C_1^2\sqrt{C_a}(1 - \sqrt{\rho})(1 - \beta)}\left[\frac{\beta}{\sqrt{C_a(1 - \rho)}}\sum_{b=1}^B\sigma_bd_b\sum_{t=2}^Tw_t\left(\sqrt{\frac{A_t}{A_{t-1}}} - 1\right) \right.\right. \label{eq:mom-final-1}\\
 && \left.\left. + \sum_{b=1}^B\left[\frac{LC_2^3w_1d_b}{C_1^3C_a(1 - \sqrt{\rho})^2} + \frac{2C_3^2C_2\sigma_bd_b}{C_1}\right]\left[w_1\log\left(\frac{\sigma_b^2}{\epsilon^2} + 1\right) 
  + \omega\sum_{t=1}^{T}\eta_t\sqrt{\frac{a_t}{A_t}}\right]\right]\right]. \label{eq:mom-final-2}
\end{eqnarray} 
\end{proof}

%%%%%%%%%%%%%%%%%%%%%%%%%%%%%%%%%%%%%%

\subsection{Proof of Corollary~\ref{corollary:nonconvex_momentum_highpro}}
By the concavity of the minimum, we have
\begin{eqnarray*}
\CE\left[\min_{1 \leq t \leq T}\|\nabla F(\theta_t)\|^{4/3}_2\right]^{3/2} \leq 
\min_{1 \leq t \leq T}\left(\CE[\|\nabla F(\theta_t)\|^{4/3}_2]\right)^{3/2}.
\end{eqnarray*}
Let $X = \min_{1 \leq t \leq T}\|\nabla F(\theta_t)\|_2^2$. 
The Theorem~\ref{theorem:nonconvex_momentum} suggests that we have $\CE[X^{2/3}] \leq C(T)^{2/3}$. 
By Markov's inequality, we get
\begin{eqnarray*}
P\left(X^{2/3} > \frac{C(T)^{2/3}}{\delta^{2/3}}\right) \leq \frac{\CE[X^{2/3}]}{C(T)^{2/3}}\delta^{2/3} \leq \delta^{2/3}.
\end{eqnarray*}
Hence, $P\left(X > \frac{C(T)}{\delta}\right) \leq \delta^{2/3}$, and we have $P(X \leq \frac{C(T)}{\delta}) \geq 1 - \delta^{2/3}$.

%%%%%%%%%%%%%%%%%%%%%%%%%%%%%%%%%%%%%%

\subsection{Proof of Corollary~\ref{corollary:nonconvex_momentum_seq}}
\begin{proof}
When $a_t = at^{\tau}$, we have $A_t = \mathcal{O}(t^{1 + \tau})$. This suggests that
\begin{eqnarray*}
\eta_t\sqrt{\frac{a_t}{A_t}} =  \eta\sqrt{\frac{a_t}{tA_t}} = \mathcal{O}\left(\frac{1}{t}\right),
\end{eqnarray*}
\begin{eqnarray*}
w_t = \frac{\eta}{1 - \tilde{\beta}_t}\sqrt{\frac{A_t}{ta_t}} \leq  \frac{\eta}{1 - \beta}\sqrt{\frac{A_t}{ta_t}}  = \mathcal{O}\left(1\right),
\end{eqnarray*}
and
\begin{eqnarray*}
\sqrt{\frac{A_t}{A_{t-1}}} - 1 = \frac{\sqrt{A_t} - \sqrt{A_{t-1}}}{\sqrt{A_{t-1}}} = \mathcal{O}\left(\frac{1}{t}\right).
\end{eqnarray*}
Hence, 
\begin{eqnarray*}
\sum_{t=1}^T\eta_t\sqrt{\frac{a_t}{A_t}} = \mathcal{O}\left(\log(T)\right), \hspace{.05in} \sum_{t=1}^Tw_t\left(\sqrt{\frac{A_t}{A_{t-1}}} - 1 \right) = \mathcal{O}\left(\log(T)\right), \hspace{.05in} \text{ and } \hspace{.05in} C(T) = \mathcal{O}\left(\frac{\log(T)}{\sqrt{T}}\right).
\end{eqnarray*}
On the other hand, when $a_t = \alpha^{-t}$, we have
\begin{eqnarray*}
\eta_t\sqrt{\frac{a_t}{A_t}} = \eta\sqrt{\frac{1 - \alpha}{(1 - \alpha^t)t}} \leq \frac{\eta}{\sqrt{t}},
\end{eqnarray*}
\begin{eqnarray*}
w_t = \frac{\eta}{1 - \tilde{\beta}_t}\sqrt{\frac{A_t}{a_tt}} = \frac{\eta}{1 - \tilde{\beta}_t}\sqrt{\frac{1 - \alpha^t}{(1 - \alpha)t}} \leq \frac{\eta}{(1 - \beta)\sqrt{(1 - \alpha)t}},
\end{eqnarray*}
and
\begin{eqnarray*}
\sqrt{\frac{A_t}{A_{t-1}}} - 1 = \sqrt{\frac{1 - \alpha^t}{(1 - \alpha^{t-1})\alpha}} \leq \sqrt{\frac{1 + \alpha}{\alpha}}.
\end{eqnarray*}
Then, we get
\begin{eqnarray*}
\sum_{t=1}^T\eta_t\sqrt{\frac{a_t}{A_t}} \leq 2\eta\sqrt{T}, \hspace{.05in} \sum_{t=1}^Tw_t\left(\sqrt{\frac{A_t}{A_{t-1}}} - 1 \right) = \mathcal{O}\left(\sqrt{T}\right), \hspace{.05in} \text{ and } \hspace{.05in} C(T) = \mathcal{O}\left(1\right).
\end{eqnarray*}
\end{proof}

%%%%%%%%%%%%%%%%%%%%%%%%%%%%%%%%%%%%%%

\subsection{Proof of Corollary~\ref{corollary:nonconvex_momentum_comparison}}
\begin{proof}
As $\|g_{t,\mG_b}\|_2^2/d_b \leq G_b^2$, then we have 
\begin{eqnarray*}
\tilde{v}_{t,b} = (v_{t-1, b} + a_t\CE_t[\|g_{t,\tilde{\mG}_b}\|_2^2]/d_b])/A_t \leq G_b^2, 
\end{eqnarray*}
and therefore $\bar{v}_{T, B} \equiv \max_{1 \leq t \leq T}\CE\left[\max_b \tilde{v}_{t,b}\right] \leq \max_b G_b^2$. 
Arranging the terms in $\tilde{C}(T)$, we obtain
\begin{eqnarray*}
  \lefteqn{\tilde{C}(T)}\\
   & = & \frac{\sqrt{2\left(\max_b G_b^2+ \epsilon^2\right)}}{\eta_TT}\left[\frac{2C_2}{(1 - \beta)C_1}[F(\theta_1) - F(\theta_*)] \right. \\
  && \left. + \frac{2C_2^2}{C_1^2\sqrt{C_a}(1 - \sqrt{\rho})(1 - \beta)}\left[\left[\frac{\beta}{\sqrt{C_a(1 - \rho)}}\sum_{t=2}^Tw_t\left(\sqrt{\frac{A_t}{A_{t-1}}} - 1\right) + \frac{2C_3^2C_2}{C_1}\omega\sum_{t=1}^{T}\eta_t\sqrt{\frac{a_t}{A_t}}\right]\sum_{b=1}^B\sigma_bd_b \right.\right. \\
 && \left.\left.  + \frac{LC_2^3w_1^2}{C_1^3C_a(1 - \sqrt{\rho})^2}\sum_{b=1}^Bd_b\log\left(\frac{\sigma_b^2}{\epsilon^2} + 1\right)
  + \frac{2C_3^2C_2w_1}{C_1}\sum_{b=1}^B\sigma_bd_b\log\left(\frac{\sigma_b^2}{\epsilon^2} + 1\right) \right.\right. \\
  && \left.\left. + \frac{LC_2^3w_1d\omega}{C_1^3C_a(1 - \sqrt{\rho})^2}\sum_{t=1}^{T}\eta_t\sqrt{\frac{a_t}{A_t}} \right]\right].
\end{eqnarray*} 
When $B=d$, we have
\begin{eqnarray*}
  \lefteqn{\tilde{C}_d(T)}\\
   & = & \frac{\sqrt{2\left(\max_b\max_{i\in\tilde{\mG}_b} G_{i}^2+ \epsilon^2\right)}}{\eta_TT}\left[\frac{2C_2}{(1 - \beta)C_1}[F(\theta_1) - F(\theta_*)] \right. \\
  && \left. + \frac{2C_2^2}{C_1^2\sqrt{C_a}(1 - \sqrt{\rho})(1 - \beta)}\left[\left[\frac{\beta}{\sqrt{C_a(1 - \rho)}}\sum_{t=2}^Tw_t\left(\sqrt{\frac{A_t}{A_{t-1}}} - 1\right) + \frac{2C_3^2C_2}{C_1}\omega\sum_{t=1}^{T}\eta_t\sqrt{\frac{a_t}{A_t}}\right]\sum_{i=1}^d\sigma_i \right.\right. \\
 && \left.\left.  + \frac{LC_2^3w_1^2}{C_1^3C_a(1 - \sqrt{\rho})^2}\sum_{i=1}^d\log\left(\frac{\sigma_i^2}{\epsilon^2} + 1\right)
  + \frac{2C_3^2C_2w_1}{C_1}\sum_{i=1}^d\sigma_i\log\left(\frac{\sigma_i^2}{\epsilon^2} + 1\right) \right.\right. \\
  && \left.\left. + \frac{LC_2^3w_1d\omega}{C_1^3C_a(1 - \sqrt{\rho})^2}\sum_{t=1}^{T}\eta_t\sqrt{\frac{a_t}{A_t}} \right]\right] \\
  & = & \frac{\sqrt{2\left(\max_b\max_{i\in\tilde{\mG}_b} G_{i}^2+ \epsilon^2\right)}}{\eta_TT}\left[\frac{2C_2}{(1 - \beta)C_1}[F(\theta_1) - F(\theta_*)] \right. \\
  && \left. + \frac{2C_2^2}{C_1^2\sqrt{C_a}(1 - \sqrt{\rho})(1 - \beta)}\left[\left[\frac{\beta}{\sqrt{C_a(1 - \rho)}}\sum_{t=2}^Tw_t\left(\sqrt{\frac{A_t}{A_{t-1}}} - 1\right) + \frac{2C_3^2C_2}{C_1}\omega\sum_{t=1}^{T}\eta_t\sqrt{\frac{a_t}{A_t}}\right]\sum_{b=1}^{\tilde{B}}\sum_{i\in \tilde{\mG}_b}\sigma_i \right.\right. \\
 && \left.\left.  + \frac{LC_2^3w_1^2}{C_1^3C_a(1 - \sqrt{\rho})^2}\sum_{b=1}^{\tilde{B}}\sum_{i\in \tilde{\mG}_b}\log\left(\frac{\sigma_i^2}{\epsilon^2} + 1\right)
  + \frac{2C_3^2C_2w_1}{C_1}\sum_{b=1}^{\tilde{B}}\sum_{i\in \tilde{\mG}_b}\sigma_i\log\left(\frac{\sigma_i^2}{\epsilon^2} + 1\right) \right.\right. \\
  && \left.\left. + \frac{LC_2^3w_1d\omega}{C_1^3C_a(1 - \sqrt{\rho})^2}\sum_{t=1}^{T}\eta_t\sqrt{\frac{a_t}{A_t}} \right]\right].
\end{eqnarray*}
Substituting $r_1 := \frac{\sum_{b=1}^{\tilde{B}}\sum_{i \in \tilde{\mG}_b}\log\left(\sigma_i^2/\epsilon^2 + 1\right)}{\sum_{b=1}^{\tilde{B}}d_b\log\left(\sigma_b^2/\epsilon^2 + 1\right)}$, 
$r_2 := \frac{\sum_{b=1}^{\tilde{B}}\sum_{i \in \tilde{\mG}_b}\sigma_i}{\sum_{b=1}^{\tilde{B}}\sigma_bd_b}$ and $r_3 := \frac{\sum_{b=1}^{\tilde{B}}\sum_{i \in \tilde{\mG}_b}\sigma_i\log\left(\sigma_i^2/\epsilon^2 + 1\right)}{\sum_{b=1}^{\tilde{B}}\sigma_bd_b\log\left(\sigma_b^2/\epsilon^2 + 1\right)}$, we get
\begin{eqnarray*}
  \lefteqn{\tilde{C}_d(T)}\\ 
  & = & \frac{\sqrt{2\left(\max_b\max_{i\in\tilde{\mG}_b} G_{i}^2+ \epsilon^2\right)}}{\eta_TT}\left[\frac{2C_2}{(1 - \beta)C_1}[F(\theta_1) - F(\theta_*)] \right. \\
  && \left. + \frac{2C_2^2}{C_1^2\sqrt{C_a}(1 - \sqrt{\rho})(1 - \beta)}\left[\left[\frac{\beta}{\sqrt{C_a(1 - \rho)}}\sum_{t=2}^Tw_t\left(\sqrt{\frac{A_t}{A_{t-1}}} - 1\right) + \frac{2C_3^2C_2}{C_1}\omega\sum_{t=1}^{T}\eta_t\sqrt{\frac{a_t}{A_t}}\right]r_2\sum_{b=1}^{\tilde{B}}\sigma_bd_b \right.\right. \\
 && \left.\left.  + \frac{LC_2^3w_1^2}{C_1^3C_a(1 - \sqrt{\rho})^2}r_1\sum_{b=1}^{\tilde{B}}d_b\log\left(\frac{\sigma_b^2}{\epsilon^2} + 1\right)
  + \frac{2C_3^2C_2w_1}{C_1}r_3\sum_{b=1}^{\tilde{B}}\sigma_bd_b\log\left(\frac{\sigma_b^2}{\epsilon^2} + 1\right) \right.\right. \\
  && \left.\left. + \frac{LC_2^3w_1d\omega}{C_1^3C_a(1 - \sqrt{\rho})^2}\sum_{t=1}^{T}\eta_t\sqrt{\frac{a_t}{A_t}} \right]\right] \\
  & \geq & \min(1, r_{\min})\sqrt{\frac{\max_b\max_{i\in\tilde{\mG}_b} G_{i}^2 + \epsilon^2}{\max_b G_b^2 +
\epsilon^2}}C_{T, \tilde{B}}.
\end{eqnarray*}
The proof is completed.
\end{proof}

%%%%%%%%%%%%%%%%%%%%%%%%%%%%%%%%%%%%%%

\subsection{Proof of Proposition~\ref{prop:gen_error}}

 %%%%%%%%%%%%%%%%%%%%%%%%%%%%%%%%%%%%%%%%%%%%%%%%%%%%%%%%%%%%%%%%%%%%%%%%%%%%%%

\begin{proof}
As function is $\tilde{\gamma}$-Lipschitz, we have the following result:
\begin{eqnarray*}
\sup_z\CE_M[f(M(S); z) - f(M(S'); z)] \leq \tilde{\gamma}\CE_M[\|M(S) - M(S')\|_2].
\end{eqnarray*}
Therefore, we can consider bounding $\CE_M[\|M(S) - M(S')\|_2]$. 
Let $\beta_t = 0$ for all $t$. 
\begin{eqnarray*}
\theta_{t+1} & = & \theta_1 - \sum_{k=1}^t\eta_kH_k^{-1}m_k \\
& = & \theta_1 - \sum_{k=1}^t\eta_kH_k^{-1}g_k \\
& = & \theta_1 - \sum_{k=1}^t\eta_kH_k^{-1}\nabla f(\theta_k; z_{i_k}),
\end{eqnarray*}
where $i_k \in [n]$ is the example index selected at iteration $k$. Then, we can bound $\Delta_{t+1} = \|\theta_{t+1} - \theta_{t+1}'\|_2$ as follows
\begin{eqnarray}
\CE[\Delta_{t+1}] & = & \CE[\|\theta_{t+1} - \theta_{t+1}'\|_2] \nonumber\\
& = & \CE[\|\theta_{1} - \theta_{1}' - \sum_{k=1}^t\eta_kH_k^{-1}\nabla f(\theta_k; z_{i_k}) + \sum_{k=1}^t\eta_kH_k^{'-1}\nabla f(\theta_k'; z_{i_k}')\|_2] \nonumber\\
& \leq & \CE[\|\theta_{1} - \theta_{1}'\|_2] + \sum_{k=1}^t\eta_k\CE[\|H_k^{-1}\nabla f(\theta_k; z_{i_k}) -H_k^{'-1}\nabla f(\theta_k'; z_{i_k}')\|_2] \nonumber\\
& = & \sum_{k=1}^t\eta_k\CE[\|H_k^{-1}\nabla f(\theta_k; z_{i_k}) -H_k^{'-1}\nabla f(\theta_k'; z_{i_k}')\|_2].  \label{eq:gen_err_0}
\end{eqnarray}
Note that $z_{i_k} = z{i_{k}}'$ with probability $1 - 1/n$. Then, we can bound each term $\CE[\|H_k^{-1}\nabla f(\theta_k; z_{i_k}) -H_k^{'-1}\nabla f(\theta_k'; z_{i_k}')\|_2]$ as follows
\begin{eqnarray}
\lefteqn{\CE[\|H_k^{-1}\nabla f(\theta_k; z_{i_k}) -H_k^{'-1}\nabla f(\theta_k'; z_{i_k}')\|_2]} \nonumber\\
& \leq & \frac{2}{n}\CE[\|H_k^{-1}\nabla f(\theta_k; z_{i_k})\|_2] 
+ \left(1 - \frac{1}{n}\right)\CE[\|H_k^{-1}\nabla f(\theta_k; z_{i_k}) -H_k^{'-1}\nabla f(\theta_k'; z_{i_k})\|_2]\nonumber\\
& \leq & \frac{2}{n}\CE[\|H_k^{-1}\nabla f(\theta_k; z_{i_k})\|_2]  + \left(1 - \frac{1}{n}\right)\CE[\|H_k^{-1}\nabla f(\theta_k; z_{i_k}) -H_k^{'-1}\nabla f(\theta_k; z_{i_k})\|_2] \nonumber\\
&&+ \left(1 - \frac{1}{n}\right)\CE[\|H_k^{'-1}\nabla f(\theta_k; z_{i_k}) -H_k^{'-1}\nabla f(\theta_k'; z_{i_k})\|_2]. \label{eq:gen_err_1}
\end{eqnarray}
 The second term is bounded as
 \begin{eqnarray*}
 \lefteqn{\CE[\|H_k^{-1}\nabla f(\theta_k; z_{i_k}) -H_k^{'-1}\nabla f(\theta_k; z_{i_k})\|_2]} \\
 & \leq & \CE[\|H_k^{-1} -H_k^{'-1}\|_2\|\nabla f(\theta_k; z_{i_k})\|_2] \\
  & \leq & \tilde{\gamma}\CE[\|H_k^{-1} -H_k^{'-1}\|_2] \\
  & = & \tilde{\gamma}\CE\left[\max_b\left|\frac{1}{\sqrt{\hat{v}_{k, b}} + \epsilon} - \frac{1}{\sqrt{\hat{v}_{k, b}'} + \epsilon}\right|\right].
 \end{eqnarray*}
We expand the third term of (\ref{eq:gen_err_1}) as
\begin{eqnarray*}
\lefteqn{\CE[\|H_k^{'-1}\nabla f(\theta_k; z_{i_k}) -H_k^{'-1}\nabla f(\theta_k'; z_{i_k})\|_2]}\\
& \leq & \CE[\|H_k^{'-1}\|_2\|\nabla f(\theta_k; z_{i_k}) -\nabla f(\theta_k'; z_{i_k})\|_2] \\
& \leq & L\CE[\|H_k^{'-1}\|_2\|\theta_k - \theta_k'\|_2] \\
& \leq & L\CE\left[\frac{1}{\sqrt{\min_b\hat{v}_{k, b}} + \epsilon}\|\theta_k - \theta_k'\|_2\right]\\
& = & L\CE\left[\frac{1}{\sqrt{\min_b\hat{v}_{k, b}} + \epsilon}\Delta_k\right].
\end{eqnarray*}
Substituting the above results into (\ref{eq:gen_err_1}) and combining with (\ref{eq:gen_err_0}), we obtain
\begin{eqnarray*}
\CE[\Delta_{t+1}] 
& \leq & \frac{2}{n}\sum_{k=1}^t\eta_k\CE[\|H_k^{-1}\nabla f(\theta_k; z_{i_k})\|_2] \\
&&  + \left(1 - \frac{1}{n}\right)\tilde{\gamma}\sum_{k=1}^t\eta_k\CE\left[\max_b\left|\frac{1}{\sqrt{\hat{v}_{k, b}} + \epsilon} - \frac{1}{\sqrt{\hat{v}_{k, b}'} + \epsilon}\right|\right] \nonumber\\
&&+ \left(1 - \frac{1}{n}\right)L\sum_{k=1}^t\eta_k\CE\left[\frac{1}{\sqrt{\min_b\hat{v}_{k, b}} + \epsilon}\Delta_k\right].
\end{eqnarray*}
Note that if $w_t = \eta_t/\sqrt{a_t/A_t}$ is ''almost" non-increasing w.r.t. another non-increasing sequence $\{z_t\}$ and positive constant $C_1$ and $C_2$, then $w_t^2$ is also ''almost" non-increasing w.r.t. another non-increasing sequence $\{z_t^2\}$ and positive constant $C_1^2$ and $C_2^2$. 
Using Lemma~\ref{lemma:nonconvex-momentum-grad-weighted-bound} with $C = I$, we have 
\begin{eqnarray*}
\lefteqn{\sum_{k=1}^t\eta_k\CE[\|H_k^{-1}\nabla f(\theta_k; z_{i_k})\|_2]} \\
& \leq & \sqrt{t}\sqrt{\sum_{k=1}^t\eta_k^2\CE[\|H_k^{-1}\nabla f(\theta_k; z_{i_k})\|_2^2]} \\
& = & \sqrt{t}\sqrt{\sum_{k=1}^t\eta_k^2\sqrt{\frac{A_k}{a_k}}\CE\left[\sqrt{\frac{a_k}{A_k}}\|H_k^{-1}\nabla f(\theta_k; z_{i_k})\|_2^2\right]} \\
& \leq & \sqrt{t}\sqrt{\frac{C_2^2}{C_1^2}\left[w_1^2\sum_{b=1}^Bd_b\log\left(\frac{\sigma_b^2}{\epsilon^2} + 1\right) 
  + d\sum_{k=1}^{t}\eta_k^2\frac{A_k}{A_{k-1} + a_1}\right]} \\
  & \leq & \frac{C_2}{C_1}\sqrt{\left[w_1^2\sum_{b=1}^Bd_b\log\left(\frac{\sigma_b^2}{\epsilon^2} + 1\right) 
  + d\omega\sum_{k=1}^{t}\eta_k^2\right]t}.
\end{eqnarray*}
Then, we get
\begin{eqnarray*}
\CE[\Delta_{t+1}] 
& \leq & \frac{2C_2}{nC_1}\sqrt{\left[w_1^2\sum_{b=1}^Bd_b\log\left(\frac{\sigma_b^2}{\epsilon^2} + 1\right) 
  + d\omega\sum_{k=1}^{t}\eta_k^2\right]t} \\
  && + \left(1 - \frac{1}{n}\right)\tilde{\gamma}\sum_{k=1}^t\eta_k\CE\left[\max_b\left|\frac{1}{\sqrt{\hat{v}_{k, b}} + \epsilon} - \frac{1}{\sqrt{\hat{v}_{k, b}'} + \epsilon}\right|\right] \nonumber\\
&&+ \left(1 - \frac{1}{n}\right)L\sum_{k=1}^t\eta_k\CE\left[\frac{1}{\sqrt{\min_b\hat{v}_{k, b}} + \epsilon}\Delta_k\right]. 
\end{eqnarray*}
\end{proof}

\end{document}